\documentclass[twoside,11pt]{article}

\usepackage{jmlr2e}

\usepackage{subcaption}
\usepackage{amsmath}
\usepackage{hyperref}
\usepackage{comment}
\usepackage{cleveref}
\usepackage{cancel}
\usepackage{mathtools}
\usepackage{bbm}
\usepackage{enumitem}
\usepackage{xcolor}
\usepackage{booktabs}
\usepackage[normalem]{ulem}
\usepackage{multirow}
\usepackage{bm}

\newcommand{\mcl}{\mathcal}

\newcommand{\mbb}{\mathbb}
\newcommand{\lp}{\left(}
\newcommand{\rp}{\right)}

\newcommand{\e}{\varepsilon}

\newcommand{\R}{\mathbb{R}}
\newcommand{\Ker}{\mathcal{K}}
\newcommand{\Pmat}{\mathrm{P}}

\newcommand{\km}{\textcolor{blue}}
\newcommand{\todo}{\textcolor{red}}

\definecolor{darkgreen}{rgb}{0.1,0.6,0.1}

\DeclareMathOperator*{\argmin}{arg\!min}
\DeclareMathOperator*{\argmax}{arg\!max}

\allowdisplaybreaks

\usepackage{lastpage}
\jmlrheading{23}{2023}{1-\pageref{LastPage}}{10/23}{}{21-0000}{Kevin Miller and Ryan Murray}

\ShortHeadings{Dirichlet Active Learning}{Miller and Murray}
\firstpageno{1}

\begin{document}

\title{Dirichlet Active Learning}

\author{\name Kevin Miller \email ksmiller@utexas.edu \\
       \addr Oden Institute of Computational Engineering \& Sciences\\
       University of Texas\\
       Austin, TX 78712, USA
       \AND
       \name Ryan Murray \email rwmurray@ncsu.edu \\
       \addr Department of Mathematics\\
       North Carolina State University\\
       Raleigh, NC 27607, USA}

\editor{My editor}

\maketitle

\begin{abstract}
    This work introduces Dirichlet Active Learning (DiAL), a Bayesian-inspired approach to the design of active learning algorithms. Our framework models feature-conditional class probabilities as a Dirichlet random field and lends observational strength between similar features in order to calibrate the random field. This random field can then be utilized in learning tasks: in particular, we can use current estimates of mean and variance to conduct classification and active learning in the context where labeled data is scarce. We demonstrate the applicability of this model to low-label rate graph learning by constructing ``propagation operators'' based upon the graph Laplacian, and offer computational studies demonstrating the method's competitiveness with the state of the art. 
    Finally, we provide rigorous guarantees regarding the ability of this approach to ensure both exploration and exploitation, expressed respectively in terms of cluster exploration and increased attention to decision boundaries.
\end{abstract}


\section{Introduction} \label{sec:intro}

The advent of big data applications in machine learning has necessitated the design of efficient methods to label data for downstream learning tasks such as classification.  
Massive computing capabilities can produce ever-increasing amounts of data, yet obtaining meaningful labels for training accurate machine learning classifiers is often time-intensive and expensive. Hence, while \textit{labeled data} (i.e., data for which the practitioner has access to observed labels) can be difficult to obtain, \textit{unlabeled data} (i.e., data for which the practitioner does \textit{not} currently have access to such labels) is ubiquitous in many practical applications. While supervised machine learning algorithms rely on the ability to acquire an abundance of labeled data, semi-supervised learning methods leverage both unlabeled and labeled data to achieve accurate classification with significantly fewer labeled data. Simultaneously, the choice of training points can significantly affect classifier performance, especially due to the limited size of the training set of labeled data in the case of semi-supervised learning. 

Active learning seeks to judiciously select a limited number of currently unlabeled data points that will inform the classification task \citep{settles_active_2012}. These points are then labeled by an expert, or human in the loop, with the aim of improving the performance of an underlying classifier.
While there are various paradigms for active learning \citep{settles_active_2012}, we focus on {\it pool-based} active learning wherein an unlabeled pool of data is available at each iteration of the active learning process from which query points may be selected. This paradigm is a natural fit for applying active learning in conjunction with semi-supervised learning since the underlying semi-supervised learner also uses the unlabeled pool. These query points are selected by optimizing an {\it acquisition function} over the discrete set of points available in the unlabeled data pool; an acquisition function is a user-defined quantity that aims to measure the ``utility'' of expending the effort to label a currently unlabeled input. Figure \ref{fig:al-flowchart} illustrates the active learning process that iterates between (1) using currently-labeled data to infer a classifier on the unlabeled data, (2) selecting the next query point from the unlabeled data via the acquisition function, and (3) labeling the query point and updating the labeled data. In practical applications, the labeling ``oracle'' is a domain expert that plays the role of a human in the loop. The number of points that are labeled at each iteration depends upon the application, but in this work, we focus on the setting where points are labeled one at a time (i.e., sequential active learning).

% Depending on the application, one iteration of this active learning process may be allowed to acquire labels for different numbers of points. In \textit{sequential} active learning only a single query point is selected at each iteration, whereas in \textit{batch} active learning multiple query points are selected at each iteration. 
% Coreset selection methods resemble an extreme case of batch active learning, wherein a single batch (coreset) of the dataset is selected but the interactive labeling of the batch is not necessarily involved. Submodular set functions are an important type of coreset selection methods, and various batch active learning methods fall under this categorization as submodular coreset methods adapted to the interactive labeling framework of Figure \ref{fig:al-flowchart}. 
% We will only consider the sequential setting in this work and leave the adaptation to batch methods for future work. 

\begin{figure}
    \centering
    \includegraphics[width=0.8\textwidth]{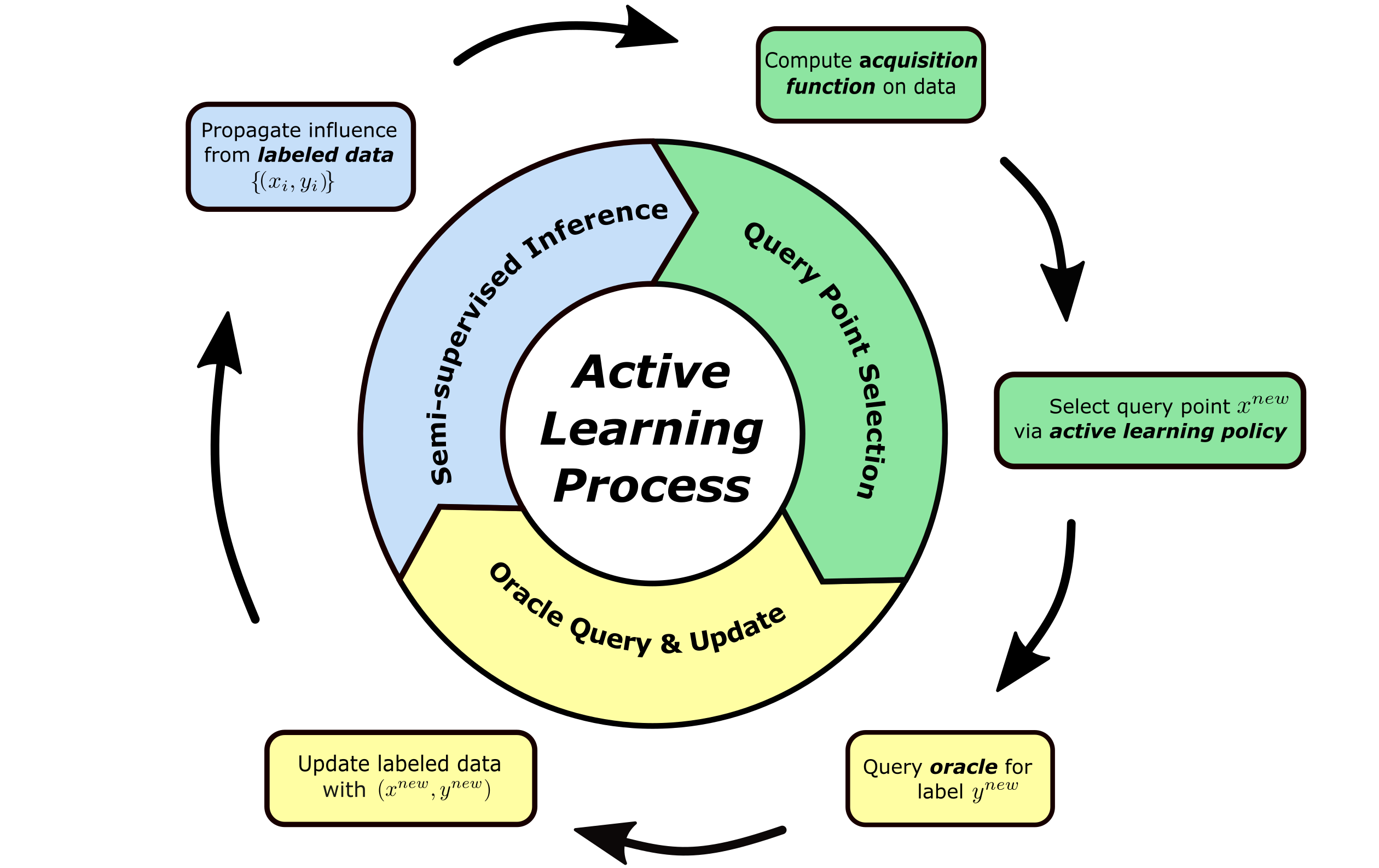}
    \caption{Active learning process flowchart.}
    \label{fig:al-flowchart}
\end{figure}

Active learning may have a number of different goals, such as ensuring high classifier accuracy relative to the number of query points or ensuring sampling from all parts of the feature distribution. Part of the challenge in doing so is that we need an appropriate vocabulary, as well as computational experiments, to understand tradeoffs between different goals in active learning. This paper seeks to develop this vocabulary and the associated computational experiments and does so while proposing a novel active learning method, which we call \emph{Dirichlet Active Learning (DiAL)}. 
% \todo{We introduce DiAL here, but we haven't incorporated that name into other parts of the paper. Would it be useful to have a schematic that distinguishes DiAL the process from Dir. Learning and Dir. Variance?}\rwm{One challenge here is that we have a number of different parts of the process, and we're not just trying to sell a named algorithm ala Neurips. I think a different schematic incorporating the name, and different parts of the process could make sense. I have some ideas about this that we can discuss.} 
% We introduce a novel semi-supervised learning model (Dirichlet Learning) that possesses a ``natural'' choice of acquisition function (Dirichlet Variance) that reflects the underlying classifier's uncertainty in the classification of unlabeled points. 

\subsection{Setup and main contributions}

We assume that we have access to a set of inputs $\mcl X \subset \mbb R^d$ which are sampled from an underlying data-generating distribution with density $\rho : \mbb R^d \rightarrow [0, \infty)$. While in application this set $\mcl X$ represents a discrete set of points, we will also consider the continuum case where $\mcl X$ represents the support of $\rho$. We seek to classify these inputs $x \in \mcl X$ into $K \ge 2$ classes, and we assume access to an initially labeled set $\mcl L = \{x^1, \ldots, x^\ell\} \subset \mcl X$ for which we have observed $y^j \in [K] = \{1, 2, \ldots, K\}$ for $x^j \in \mcl L$. Let $\mcl Y_{\mcl L} = \{y^1, \ldots, y^\ell\}$ denote the corresponding set of labels (classes) for the labeled set $\mcl L$.
% We will assume access to a set of data points $\mcl X \subset \mbb R^d$ that \km{we seek to classify into} 
 % each belongs to one of 
% $K \ge 2$ classes. Let $\mcl L = \{x^1, \ldots, x^j\} \subset \mcl X$ denote the \textit{initially labeled set} of inputs where the class assignment $y^j \in [K] = \{1, 2, \ldots, K\}$ for $x^j \in \mcl L$ have been observed. We assume that there is a relationship between the inputs and their associated class assignments that we will model via a joint probability distribution with density $\nu : \mbb R^d \times [K] \rightarrow [0, \infty)$. For this initial setup, we will not discuss the particulars of the distribution $\nu$, but it suffices to say that we will encode an assumed clustering structure in $\nu$. 
The task in active learning will be to iteratively augment the labeled data set $\mcl L \subset \mcl X$ by cycling between (1) learning a semi-supervised learning classifier given current labeled data $\mcl L$ with associated labels $y^j$ for $x_j \in \mcl L$, (2) selecting a query point $x^\ast$ from the unlabeled data, $\mcl U = \mcl X \setminus \mcl L$,  and (3) labeling said query point and adding it to the labeled data accordingly. As such, this process generates a sequence of labeled sets $\mcl L^{(0)} \subset \mcl L^{(1)} \subset \mcl L^{(2)} \subset \ldots$ as query points are labeled and added to the labeled set; we refrain from this iteration-explicit notation in favor of readability and denote the (changing) labeled set simply as $\mcl L$.

Active learning serves as a naturally complementary problem to semi-supervised learning. At a high level, semi-supervised learning seeks an answer to the question  %\rwm{I don't know if these questions look better boxed? Ok either way, just trying it out.}
% \begin{center}
%     \textit{How do we accurately infer the classification of the unlabeled \\ data $\mcl U$ given the currently labeled data $\mcl L$?}
% \end{center} 
\begin{center}
\fbox{ 
    \begin{tabular}{@{}c@{}}
     \textit{How do we accurately infer the classification of the unlabeled} \hspace{0.75em} \\
    \textit{data $\mcl U$ given the currently labeled data $\mcl L$?}
    \end{tabular}%
}
\end{center}

In a complementary fashion, active learning seeks to answer the question
% \begin{center}
%     \textit{How do we judiciously select currently unlabeled points $x \in \mcl U$ that, once labeled, \\ can maximally improve performance in the underlying semi-supervised problem?}
% \end{center} 
\begin{center}
\fbox{ 
    \begin{tabular}{@{}c@{}}
     \textit{How do we judiciously select currently unlabeled points $x \in \mcl U$ that, once labeled,} \hspace{0.75em} \\
    \textit{can maximally improve performance in the underlying semi-supervised problem?}
    \end{tabular}%
}
\end{center}
An implicit constraint is to perform as few queries as possible; in practical applications, labeling of query points may be costly or time-consuming and so an overall budget of queries may be explicitly imposed. Hence, a goal of active learning is to design methods which are provably efficient at selecting the query points that are the most useful for the underlying semi-supervised learning task. 

The terminology for the mechanism by which one selects query points is not uniform across the literature and we accordingly suggest the following framework. The active learning query points are selected at each iteration via the combination of what we will term an  \textit{acquisition function}\footnote{The term ``acquisition function'' appears in various references such as \citep{gal_deep_2017}.} and a \textit{policy}. We define an acquisition function $\mcl A (\cdot, \mcl L, \mcl Y_{\mcl L}): \mcl U \rightarrow \mbb R$ to be a user-defined criterion to quantify the utility of expending the effort to label a point $x \in \mcl U$. This acquisition function is dependent on the currently labeled data $\mcl L, \mcl Y_{\mcl L}$, though for ease of notation we drop the explicit dependence and write $\mcl A(\cdot) = \mcl A(\cdot; \mcl L, \mcl Y_{\mcl L})$. Indeed it may be said that the ``art'' of active learning primarily depends on the choice of acquisition function to reflect the desired properties of the selected query points. With a chosen acquisition function it remains to decide how to select the next query point $x^\ast \in \mcl U$ based upon the acquisition function values on the unlabeled data, $\{\mcl A(x)\}_{x \in \mcl U}$. We refer to the method for selecting the query point from said values as the \emph{active learning policy}. For example, a common policy is to select the maximizer, $x^\ast = \argmax_{x \in \mcl U} \mcl A(x)$, or minimizer $x^\ast = \argmin_{x \in \mcl U} \mcl A(x)$ depending on the specific properties of the chosen acquisition function. We also consider a proportional sampling policy that fits naturally into our theoretical results (Section \ref{sec:theory}) and is beneficial in our experiments (Section \ref{sec:results}). 

% \km{Gaussian process/random fields, graphs is a case, briefly explain clear message that isn't fit to the classification problem.} 
A novel idea of the present work is to model the influence of the labeled data on the unlabeled data via what we will term a ``Dirichlet random field''. Random fields (i.e., collections of random variables $\{F_x\}_{x \in \mcl X}$ indexed by elements in (a subset of) a topological space) are convenient structures for statistical processes, such as active learning. For example, Gaussian random fields are collections of Gaussian random variables where the correlation structure between the random variables is also Gaussian \citep[see][]{rasmussen_gaussian_2006}. Gaussian random fields and Markov random fields are commonly used in the active learning literature to define acquisition functions that reflect measures of uncertainty or variance in the underlying statistical model \citep{zhu_combining_2003, jun_chapter_2018, ji_variance_2012, schreiter2015safe, riis2022bayesian, kapoor2007active, krause2007nonmyopic}. Graph-based random fields are an especially relevant case \citep{zhu_combining_2003, jun_chapter_2018, ji_variance_2012, ma_sigma_2013, qiao_uncertainty_2019, miller_efficient_2020}, wherein the index set $\mcl X$ is finite and the correlation structure is determined by the connectivity structure of an associated graph $G(\mcl X, W)$.
% ; that is, the correlation between $A_x$ and $A_z$ is directly related to the edge weight $W_{xz} \ge 0$ which is usually chosen to model the similarity between the inputs $x, z \in \mcl X$ (e.g., $W_{xz} = \exp(-\|x - z\|_2^2/2)$). 
Various previous graph-based methods for semi-supervised and active learning have been framed in the context of Gaussian random fields, where the correlation structure between the outputs at nodes (i.e., inputs $x \in \mcl X$) follows a Gaussian distribution influenced by the connectivity structure of a graph Laplacian matrix \citep{zhu_semi-supervised_2003, zhu_combining_2003, bertozzi2018uncertainty, qiao_uncertainty_2019} constructed from the set of inputs. 
While this approach yields a straightforward Bayesian interpretation of the associated semi-supervised and active learning problems, \textit{the inherent modeling assumption is that of regression, not classification}. See Section \ref{sec:prev-graph-based} for further discussion.

In contrast, we introduce a Dirichlet random field in the context of semi-supervised and active learning for classification tasks. The details of this model as given in Sections 2 and 3, but we summarize briefly here: we model the classification of each $x \in \mcl X$ via a categorical random variable $Y(x) \sim Cat(P(x))$. Inspired by the classical Bayesian approach, we express our current information about the probability vector $P$ as a Dirichlet random variable% regarding the likelihood that $x$ belongs to one of the classes $1, 2, \ldots, K$ is modeled as a Dirichlet random variable
\[
 P(x)= (p_1, \ldots, p_K)\sim Dir(\alpha_1(x), \alpha_2(x), \ldots, \alpha_K(x)).
\]
% that models our belief about the probabilities of a categorical random variable $Y(x) \sim Cat(P(x))$ that models the likelihood that $x$ belongs to one of the classes $1, \ldots, K$. 
%With a uniform prior belief over the classification of outputs, 
The Dirichlet posterior belief at each point $x \in \mcl X$ is updated according to the amount of information that is propagated from labeled set $\mcl L$ according to the respective classes. We call this information a ``pseudolabel'', as the $\alpha_i$ parameters of the Dirichlet distribution are interpreted as observed labels in a classical Bayesian context, whereas in the low-label learning context we must resort to treating observed labels at $x' \sim x$ as proxies for label observations at $x$. Using the vector $P$ we can construct a semi-supervised learning classifier which we call \textit{Dirichlet Learning}. 
% We remark that in 
In the graph-based setting, the propagation operator is defined via the graph Laplacian matrix. 

% The proposed Dirichlet Learning semi-supervised classifier frames the prior and observation model in an arguably more fitting Bayesian-inspired framework.  While we primarily focus on this model in the context of graph-based propagation operators in our experiments and motivation, the methodology is not restricted to graph-based methods; accordingly, we introduce the Dirichlet Learning model in terms of general propagation operators. 
As a result of the Bayesian-inspired approach of Dirichlet Learning, the variance of the current estimate of the categorical probabilities at each unlabeled point is readily computable. This variance quantifies classifier uncertainty which we use as an active learning acquisition function. We term this acquisition function \textit{Dirichlet Variance}.

% setup implicitly relies upon the assumption that the observed outputs $y^j \in \{-1, 1\}$ for $x^j \in \mcl L$ follow
% \[
%     y^j = u^\dagger(x^j) + \eta, \quad \eta \sim \mcl N(0, \gamma^2),
% \]
% for an underlying, ``ground-truth'' function $u^\dagger : \mcl X \rightarrow \mbb R$. In classification, this observation model should reflect the discrete nature of the output space, namely the different classes. While some methods \citep{qiao_uncertainty_2019, miller_efficient_2020} propose different observation models to potentially better fit the classification problem, the resulting Bayesian setting does require approximate methods to estimate the statistics of the posterior distribution; in these cases, while the prior belief is Gaussian, the corresponding likelihood distribution of the observation model is necessarily non-Gaussian and results in a non-Gaussian posterior. Our Dirichlet Learning model provides a 

Having described the context, we now briefly summarize the main contributions of the present work as follows.
\begin{enumerate}[label=\textbf{\arabic*})]
    \item \textbf{Novel semi-supervised learning model (Dirichlet Learning) for efficient active learning.}
    
    We introduce a novel semi-supervised learning classifier (Dirichlet Learning) that models the inferred classifications on unlabeled data via a Dirichlet random field. An intuitive measure of ``uncertainty'' (Dirichlet Variance) is readily available for this classifier, and we we find this quantity to be an informative acquisition function for both exploration and exploitation in active learning. The proposed Dirichlet Learning (classifier) and DiAL (active learning) frameworks are flexible and easily adapted to any desired propagation operator from labeled data to the rest of the dataset. We specialize the framework to graph-based propagations that are very useful for capturing the clustering structure of the underlying dataset. These graph-based propagation operators are naturally defined on the discrete graph structure, but also conveniently have intimate connections to continuum-limit second-order elliptic operators that depend on the underlying data-generating distribution. 
    
    \item \textbf{Theoretical guarantees for exploration of clustering structure by DiAL.}
    
    We establish theoretical guarantees regarding explorative behavior of 
    % the proposed graph-based version of 
    DiAL in the low-label rate regime (Section \ref{subsec:exploration-guarantees}). Under some simple assumptions about the choice of propagation operator as it relates to the underlying data distribution, we demonstrate that $K$ active learning queries are sufficient to ensure the labeling of points in $K$ clusters of potentially varying sizes. 
    % This is in contrast to the straightforward estimate of $\mcl O(K \log K)$ samples to achieve the same guarantee by uniformly sampling at random from a dataset comprised of $K$ roughly equal-sized clusters.
    \item \textbf{Novel asymptotic analysis of the later stages of active learning.} 
    
    We provide an asymptotic analysis of the exploitative behavior of DiAL in the later stages of the active learning process (Section \ref{subsec:exploit-analysis}). We have not seen previous work, particularly in the graph-based active learning literature, that characterizes the late-stage asymptotics of active learning. This novel analysis elucidates that, if scaled properly, Dirichlet Variance as an acquisition function can lead to querying in regions of the dataset where multiple class-conditional probabilities are high--i.e., where there is inherent population-level uncertainty in the underlying data distribution. We suggest this is a natural characterization of beneficial exploitation in active learning.
    \item \textbf{Computational efficiency of the proposed method.} 

    We empirically verify the efficacy of DiAL when utilizing a graph-based Dirichlet Learning classifier to explore dataset clustering structure in real-world datasets, including a comparison to previous methods on hyperspectral imagery (HSI) pixel classification. We also highlight the very favorable computational complexity of the acquisition function, comparing it to previous graph-based active learning acquisition functions. 
    \item \textbf{Development of active learning vocabulary and associated computational tests.} 

    We provide various discussions throughout that seek to clarify important ideas that are not necessarily original to this work, but we believe help to elucidate the mechanisms underlying active learning. For example, in Remark \ref{remark:2types-unc} we discuss two distinct types of ``uncertainty'' in active learning (population-level and data-conditional uncertainty) and provide an illustrative example that demonstrates the importance of designing acquisition functions that ultimately reflect population-level uncertainty. We suggest that the asymptotic analysis of exploitation provides an important avenue for continued research in active learning methods. Finally, we introduce the terminology of an \textit{active learning policy}. While acquisition functions provide a concrete quantity to reflect the utility of labeling a currently unlabeled point, an active learning policy identifies how these acquisition function values are ultimately used to select the next query point. Traditionally, the maximizer of the acquisition function on the unlabeled data is chosen to be labeled; we find that sampling proportional to the acquisition function values is a useful active learning policy for both our numerical and theoretical results. 
\end{enumerate}

\subsection{Previous Work} \label{subsec:prev-work}
Acquisition functions for active learning have been introduced for various machine learning models, including support vector machines \citep{tong_support_2001, balcan_margin_2007, jiang_minimum-margin_2019, hanneke_minimax_2015}, deep neural networks \citep{ren2022deepALsurvey, gal_deep_2017, kushnir_diffusion-based_2020, cai_batch_2017, sener_active_2018, ash_deep_2020}, and graph-based classifiers \citep{miller_efficient_2020, miller_model-change_2021, miller_spie_2022, qiao_uncertainty_2019, ma_sigma_2013, ji_variance_2012, zhu_combining_2003}. Our experiments will primarily focus on graph-based classifiers as the underlying semi-supervised classifier due to their straightforward ability to capture clustering structure in data and their superior performance in the {\it low-label rate regime}--wherein the labeled data constitutes a very small fraction of the total amount of data \citep{calder_poisson_2020, miller2023unc}. While there has been progress in adapting deep neural networks to better handle small amounts of labeled data for semi-supervised classification tasks \citep{mixmatch2019, fixmatch2020, yang_sample_2023}, most active learning methods for deep learning assume a moderate-to-large amount of initially labeled data when evaluating their methods in the active learning process. 

\subsubsection{Exploration versus exploitation} 

An important aspect of active learning is the inherent tradeoff between using the inherently limited resource of queries to either (i) explore the given dataset or (ii) exploit the current classifier's inferred decision boundaries. This tradeoff is reminiscent of the similarly named ``exploration versus exploitation'' tradeoff in reinforcement learning \citep{Sutton1998, agarwal2021reinforcement}. Similar to reinforcement learning, there is a clear motivation for ensuring that exploration is performed prior to exploitation in the active learning process. Broadly speaking, however, most active learning acquisition functions are designed to exhibit one of these two behaviors, though some methods do seem to empirically balance both characteristics \citep{krause2007nonmyopic, huang_active_2010, karzand_maximin_2020, miller_model-change_2021}.  An important contribution of the current work is to present a mathematical analysis that gives qualitative guarantees about the explorative and exploitative characteristics of query points that are chosen by the proposed acquisition function, Dirichlet Variance.

Prior work to establish theoretical foundations for active learning has primarily focused on proving sample-efficiency results for linearly-separable datasets---frequently restricted to the unit sphere \citep{balcan2009agnostic, dasgupta_coarse_2006, hanneke_bound_2007}---for low-complexity function classes using disagreement or margin-based acquisition functions \citep{hanneke_theory_2014, hanneke_minimax_2015, balcan2009agnostic, balcan_margin_2007}. These provide convenient bounds on the number of query points necessary for the associated classifier to achieve (near) perfect classification on these datasets with simple geometry. These results demonstrate that a proposed acquisition function is sufficient to select query points that will (near) optimally refine the associated classifier's decision boundaries to best match the assumed ground-truth decision boundaries; that is, these classical statistical guarantees for active learning focus on the exploitative behavior of the associated methods.  

In contrast, theoretical guarantees for graph-based active learning methods primarily demonstrate that a proposed acquisition function sufficiently explores an assumed clustering structure for the dataset \citep{murphy_unsupervised_2019, dasarathy_s2_2015, miller2023unc, dasgupta_hierarchical_2008, dasgupta_two_2011, cloninger_cautious_2021}. Occasionally this clustering structure is assumed to be hierarchical \citep{dasgupta_hierarchical_2008, dasgupta_two_2011, cloninger_cautious_2021}. Sufficient exploration of clustering structure is characterized by a guarantee that given assumptions about the clustering structure of the observed dataset $\mcl X$, the active learning method in question will query points from \textit{all} clusters. The low-label rate regime of active learning---a significant focus of this current work---is the natural setting for establishing such explorative guarantees.

In summary, while classical statistical analysis of disagreement-based and margin-based active learning methods has focused on \textit{exploitative guarantees}, the analysis of graph-based active learning has focused on \textit{explorative guarantees}. In the current work, we not only provide explorative guarantees for our proposed acquisition function but also give a novel asymptotic analysis of the exploitative capabilities of the acquisition function later on in the active learning process. 

\subsubsection{Computational complexity of active learning}

Computational complexity is an additional consideration that is vital to the practical application of active learning methods. In the assumed setting of sequential active learning, the computational burden at each iteration is encountered in two main ways: (i) the cost to evaluate the acquisition function on a single, currently unlabeled data point and (ii) the size of the set of unlabeled data chosen to evaluate said acquisition function. Additionally, some active learning methods require the computation of ``auxiliary'' variables, such as covariance matrices \citep{ji_variance_2012, ma_sigma_2013, miller_efficient_2020, miller_model-change_2021}, eigenvector matrices \citep{murphy_unsupervised_2019, miller_model-change_2021}, or graph paths and distances \citep{murphy_unsupervised_2019, cloninger_cautious_2021, dasarathy_s2_2015} to be stored and oftentimes updated throughout the active learning process. Finally, some graph-based active learning methods \citep{murphy_unsupervised_2019, cloninger_cautious_2021} do not exactly follow the assumed interactive labeling scheme displayed in Figure \ref{fig:al-flowchart}; instead, these methods more closely resemble coreset methods \citep{bachem2017practical} since the labels of the selected data points do not influence the choice of any subsequently selected points. Such methods incur other computational costs that do not fit into the paradigm we now discuss.

Reducing the number of unlabeled points on which to evaluate the acquisition function is a way of reducing computational cost, and various heuristics have been suggested previously such as uniform random subsampling \citep{gal_deep_2017, miller_model-change_2021} or graph-based local restrictions \citep{chapman2023batchALsar}. While this is an interesting direction for research, we assume that the acquisition function is evaluated on the entire unlabeled data pool, as is standard in pool-based active learning \citep{settles_active_2012}. As such, the primary source of computational complexity follows from the cost of evaluating the acquisition function on a single unlabeled data point.  

Uncertainty sampling \citep{settles_active_2012, miller2023unc} is a category of acquisition functions that use the current classifier's outputs to approximate the ``uncertainty'' of the inferred classification of unlabeled data\footnote{See Example \ref{remark:2types-unc} for a discussion about uncertainty in semi-supervised learning and active learning.}; the most ``uncertain'' points are then selected to be queried in uncertainty sampling. Different measures of uncertainty in the classifier outputs (e.g., smallest margin, entropy, $\ell^2$-norm) determine the different acquisition functions in uncertainty sampling. While uncertainty sampling has not always empirically demonstrated optimal exploration versus exploitation behavior, it is generally among of the most computationally efficient kind of acquisition functions--the current classifier's outputs are the only required quantity at each active learning iteration, which is readily available in our assumed setting (Figure \ref{fig:al-flowchart}). The cost per unlabeled data point simply scales as the number of classes in the classification problem at hand. This is in contrast to other acquisition functions, such as Variance Minimization \citep{ji_variance_2012} and $\Sigma$-Optimality \citep{ma_sigma_2013}, that require computations that scale as the size of the entire dataset in order to evaluate the acquisition function at a single unlabeled data point.  

Our proposed acquisition function, Dirichlet Variance \ref{eq:dirichlet-variance-af}, can indeed be classified as a novel type of uncertainty sampling that naturally fits within the proposed semi-supervised learning model (Dirichlet Learning) which we introduce in Section \ref{sec:beta-learning}. See Section \ref{subsubsec:compare-comp} for further discussion about and comparison of computational complexity among the compared methods.

\subsubsection{Graph-based learning} \label{sec:prev-graph-based}
 
Graph-based methods have shown to be useful models for semi-supervised and active learning, especially in the low-label rate regime (i.e., when the amount of labeled data is \textit{significantly} smaller than the amount of unlabeled data). Generally speaking, these methods construct a similarity graph $G(\mcl X, W)$ from a finite set of inputs $\mcl X = \{x^1, \ldots, x^n\}$, where the edge weight matrix $W \in \mbb R^{n \times n}$ records the pairwise similarities $W_{ij} \ge 0$ between inputs $x^i,x^j \in \mcl X$. For example, edge weights computed by the Gaussian (RBF) kernel are $W_{ij} = \exp(-\|x^i - x^j\|_2^2)$. The task of semi-supervised classification then amounts to identifying how to use both the labeled data (i.e., label $y^j$ at labeled node $x^j \in \mcl L$) and the similarity graph $G(\mcl X, W)$ to infer labels on the set of unlabeled nodes $\mcl U = \mcl X \setminus \mcl L$. Especially relevant to our work is the subset of graph-based methods that are inspired by numerical methods for partial differential equations (PDE), wherein the semi-supervised learning task reduces to identifying a graph function $u : \mcl X \rightarrow \mbb R^K$ where the output $u(x) \in \mbb R^K$ reflects the inferred classification of node $x \in \mcl X$ \citep{bertozzi_diffuse_2016, calder_poisson_2020, calder_properly-weighted_2020, zhu_semi-supervised_2003, zhou2004llgc}. A central component in nearly all graph-based methods is the graph Laplacian matrix, $L \in \mbb R^{n \times n}$; this matrix is a graph-based analog of the Laplace operator \citep{chung1997book, von_luxburg_tutorial_2007}. Common examples of the graph Laplacian matrix are the \textit{combinatorial} $L = D - W$, the \textit{random walk} $L_{rw} = I - D^{-1} W$, and the \textit{symmetric normalized} $L_n = I - D^{-1/2}WD^{-1/2}$ graph Laplacians, where $D = \operatorname{diag}(d_1, d_2, \ldots, d_n)$ is the diagonal degree matrix with $d_i = \sum_{j = 1}^n W_{ij}$. A common rationale that motivates the use of graph-based methods for semi-supervised learning is that the graph Laplacian matrix is effective for identifying clustering structure in the underlying dataset. This is a primary observation and motivation for the use of spectral clustering in unsupervised clustering (see \citep{von_luxburg_tutorial_2007} and the references therein).

We also mention that these graph-based methods for semi-supervised learning (classification and regression) have intimate connections to second-order elliptic PDEs. Namely, significant work has been done to connect the graph Laplacian matrix, its eigenvalues and eigenvectors, and the corresponding semi-supervised classifiers to ``continuum limit'' counterparts which can be thought of the limit of the discrete graphs as the number of nodes $n \rightarrow \infty$, under proper assumptions regarding the graph scaling; see, for example, \citep{calder_properly-weighted_2020, calder_rates_2023}. This analysis is enlightening since the continuum limit PDE formulation acts as a proxy for the large data limit scenarios often faced in application, wherein the amount of unlabeled data is especially large. This allows one to analyze the properties of the second-order elliptic equations as opposed to the discrete graph structure thereby informing the behavior of the methods in the discrete setting. For example, the work in \citep{calder_properly-weighted_2020} uses this continuum limit analysis to propose a ``properly-weighted'' graph-based semi-supervised classifier that resolves degenerate behavior of solutions to the Laplace learning classifier \citep{zhu_semi-supervised_2003} in the presence of extremely low label rates. Portions of our theoretical analysis in Section \ref{sec:theory} for Dirichlet Learning in the graph-based setting rely on a continuum limit formulation and we introduce further notation and setup at that point in the paper.

Similar to graph-based methods for semi-supervised learning, graph-based active learning can then be framed in terms of selecting unlabeled nodes $x \in \mcl U$ based on properties such as the graph-based classifier $u$ given the currently labeled data, $\mcl L$. As mentioned previously, however, various acquisition functions have been proposed from a statistical perspective of the graph-based semi-supervised learning problem \citep{zhu_combining_2003, jun_chapter_2018, ji_variance_2012, ma_sigma_2013, qiao_uncertainty_2019, miller_efficient_2020}. Namely, some graph-based semi-supervised classifiers can be viewed as the \textit{maximum a posteriori} (MAP) estimator of a Gaussian random field whose correlation structure is related to the graph Laplacian matrix of the associated graph $G(\mcl X, W)$. 

For example, the Laplace learning semi-supervised classifier $u : \mcl X \rightarrow \mbb R$ for binary classification introduced by \citet*{zhu_semi-supervised_2003} can be viewed as the MAP estimator of the Gaussian random field with density
\[
    p(u) \propto \exp \lp - \frac{1}{\gamma} \mcl E(u) \rp = \exp \lp - \frac{1}{\gamma} \sum_{i,j=1}^N w_{ij} (u(x^i) - u(x^j))^2 \rp,
\]
where $\mcl E(u) = \sum_{i,j=1}^N w_{ij} (u(x^i) - u(x^j))^2$ is called the graph Dirichlet energy and $\gamma > 0$ is interpreted as a ``temperature'' parameter \citep[see][]{zhu_combining_2003}. Given observations (labels) $y^j \in \{\pm 1\}$ at the labeled nodes $x^j \in \mcl L$, the corresponding posterior distribution reflects fixing the outputs $u(x^j) = y^j$ while interpolating the values at the unlabeled nodes according to the graph topology. Other works have considered variants with Gaussian observation models \citep{bertozzi_posterior_2021, zhou2004llgc}, binary Markov random fields \citep{jun_chapter_2018}, and other non-Gaussian observation models \citep{qiao_uncertainty_2019, bertozzi_posterior_2021}. Each admits computation (or numerical approximation) of the uncertainty in the classifier under the respective Bayesian models. The statistics of the underlying posterior distribution, given the observed labeled data, can be computed for use as acquisition functions to identify which currently unlabeled points would reduce measures of the underlying posterior covariance matrix \citep{ji_variance_2012, ma_sigma_2013} or the overall expected error \citep{zhu_combining_2003, jun_chapter_2018}.

While this convenient Bayesian interpretation underlies each of these aforementioned graph-based semi-supervised learning models, there is an implicit assumption that the prior distribution over node functions $u$ follows a \textit{Gaussian} law--thus, the most ``natural'' Bayesian modeling choice is that of continuous-valued outputs like that of regression problems, \textit{not} classification problems. We directly address this shortcoming with our novel Dirichlet Learning classifier for semi-supervised learning that explicitly models the classification task in a Bayesian-inspired manner. Our proposed Dirichlet Variance acquisition function then is a natural choice for both its computational efficiency and theoretical interpretation for measuring the uncertainty\footnote{Classifier ``uncertainty'' is admittedly an ambiguous term, and we refer the reader to Remark \ref{remark:2types-unc} for a discussion of different types of uncertainty in active learning.} in the underlying Dirichlet Learning classifier.

\subsection{Summary of Notation}

We here briefly summarize various notational conventions used throughout the paper for the reader's convenience. We will generally use caligraphic capital letters (e.g., $\mcl X, \mcl L$) to denote sets, with the lone exception of $\Ker(\cdot, \cdot)$ to denote a kernel used to defined propagations from labeled data points. We will use subscripts to index entries of a vector and superscripts to index elements of a set; for example, $x^i \in \mcl X$ will represent the $i^{th}$ element of a set $\mcl X$ and $x_k$ will represent the $k^{th}$ entry of a vector $x \in \mbb R^d$. We let $B(x,\delta)$ denote an open ball centered at $x$ of radius $\delta>0$.

\section{Modeling uncertainty in semi-supervised learning using Dirichlet priors} \label{sec:beta-learning}

We consider a set of features $\mcl X \subset \mbb R^d$, and a set of $K \ge 2$ classes. We assume that there is a relationship between the features and the $K$ classes (categories), which we model with a joint probability distribution $\nu$ over the space $\mbb R^d \times [K]$. We consider the marginal of $\nu$ on the inputs to be modeled via the mixture model
\[
    \rho(x) = \sum_{k=1}^K w_k\rho_k(x),
\]
where each $\rho_k(x) = \mbb P(x | y = k)$ is the class-conditional distribution's density for the $k^{th}$ class and the weights represent the class marginal distribution weights $\sum_{k=1}^K w_k = 1, w_k \geq 0$. Let $\rho^K(x) = (\rho_1(x), \rho_2(x), \ldots, \rho_K(x))^T \in \mbb R^K$ be the vector of class-conditional densities at the input $x \in \mcl X$.
% \todo{what about $\mbb P(Y = k ) = w_k$ weights in the mixture model? (for imbalanced class sizes), or does it not matter}

% Suppose further that we are given a subset of points with their associated classes; that is, we are
Recall that we are given a set of {\it labeled data}, $\{(x^\ell, y^\ell)\}_{x^\ell \in \mcl L}$, where $\mcl L \subset \mcl X$ and $y^\ell \in \{1 \dots K\} =: [K]$ are the observed labels, which are drawn from the joint distribution $\nu$. We will identify the label $y^\ell$ with its ``one-hot encoding representation'', $y(x^\ell) = e_{y^\ell} \in \mbb R^K$, where $e_k$ is the $k^{th}$ standard basis vector in $\mbb R^K$. 
% \rwm{Would the notation be more consistent if we used a superscript to denote different elements of a set and subscripts to denote coordinates in a vector?}
% \km{Should we make vectors and matrices bold?} \rwm{I usually don't (mostly because it's impossible to reproduce in handwriting), but I can be convinced either way.}
%Consider the set of points $\mcl X = \{x_1, x_2, \ldots, x_n\} \subset \mbb R^d$ indexed by $[n] := \{1, 2, \ldots, n\}$\rwm{Do we really need the $x$'s to be a finite set in this section? Maybe not useful, but would work better with kernels} that are assumed to belong to one of $K \ge 2$ classes. Suppose we are given a subset of points for which we have already observed their classification; that is, we are given the {\it labeled data}, $\{(x^\ell, y^\ell)\}_{x^\ell \in \mcl L}$, where $\mcl L \subset \mcl X$ and $y^\ell \in [K]$ are the observed labels. We will identify the label $y^\ell$ with its ``one-hot encoding representation'', $y(x^\ell) = e_{y^\ell} \in \mbb R^K$, where $e_k$ is the $k^{th}$ standard basis vector in $\mbb R^K$. 
Semi-supervised classification is the task of inferring the classification of the {\it unlabeled data} $\{x\}_{x \not\in \mcl L}$ from the observed labeled data $\{(x^\ell, y^\ell)\}_{x^\ell \in \mcl L}$. As this is an ill-posed problem, one must incorporate a priori assumptions about the ground-truth classification of the points in terms of the underlying geometry of the dataset $\mcl X$. 

In light of the semi-supervised context, we choose to model the classification $y(x)$ of an input $x \in \mcl X$ probabilistically in terms of a categorical random variable $Y$ such that
\[
    \mbb P(Y = k | p) = p_k, \qquad p \in \Delta_K,
\]
where the probability vector $p$ belongs to the $K-1$ dimensional simplex, $\Delta_K = \{q \in \mbb R^K : q_k \ge 0, \sum_{k=1}^K q_k = 1\}$. 

We recall that in the context of classical Bayesian inference, if we have observations of a categorical variable $Y \in [K]$ we model the distribution of the probability of each category using a Dirichlet distribution, which is a probability distribution on the simplex with density function given by
\begin{equation} \label{eq:pdf-dir}
\varphi(\alpha,z) := B(\alpha) \prod_{i=1}^K z^{\alpha_i -1},
\end{equation}
where $\alpha$ is a vector in $\mbb R_+^K$, $z$ is vector in $\Delta_K$, and $B(\alpha)$ is a normalization constant given by $B(\alpha) = \frac{\Gamma(\sum_{i=1}^K \alpha_i)}{\prod_{i=1}^K \Gamma(\alpha_i)}$, where $\Gamma$ is the standard Gamma function generalizing the factorial.
Given a Dirichlet prior distribution with parameters $\alpha$, if we observe the data $Y = j$ then our posterior probabilities take the form $(\alpha_1,\ldots, \alpha_j + 1, \ldots \alpha_K)$. 
In this classical setting one may interpret the vector $\alpha$ as the total number of observations we have of each category, usually across repeated experiments.

Inspired by this classical Bayesian estimation problem, we choose to model semi-supervised learning as a problem of estimating the set of probability vectors $\{p(x)\}_{x \in \mcl X}$.  In light of this aim, for each $x \in \mcl X$, let $P(x) \sim Dir(\alpha_1(x) + \alpha_0, \alpha_2(x) + \alpha_0, \ldots, \alpha_K(x) + \alpha_0)$ 
be a Dirichlet-valued random variable (with uniform prior $Dir(\alpha_0, \alpha_0, \ldots, \alpha_0)$) to model our belief about the multinomial probability vector $p(x)$. In contrast with the classical Bayesian setting, in the semi-supervised setting we do not expect to observe the categorical variable $Y(x)$ associated with every feature $x$: instead, we will observe values of $Y(\tilde x)$ where $\tilde x$ and $x$ are similar. In that light, we will consider a vector $\alpha(x) := (\alpha_1(x), \alpha_2(x), \ldots, \alpha_K(x))^T \in \mbb R^K_+$ which reflects the continuous values of implicit ``observations''--what we will term ``pseudolabel''--from each possible class at the feature $x$. We emphasize that the semi-supervised setting requires us to lend the strength of an observation at one vertex to nearby vertices (using a propagation mechanism that we will describe below). Hence we call the $\alpha(x)$ pseudo-labels or implied observations because the labels were only observed at nearby points.

\begin{figure}[t]
    \centering
    \includegraphics[width=0.7\textwidth]{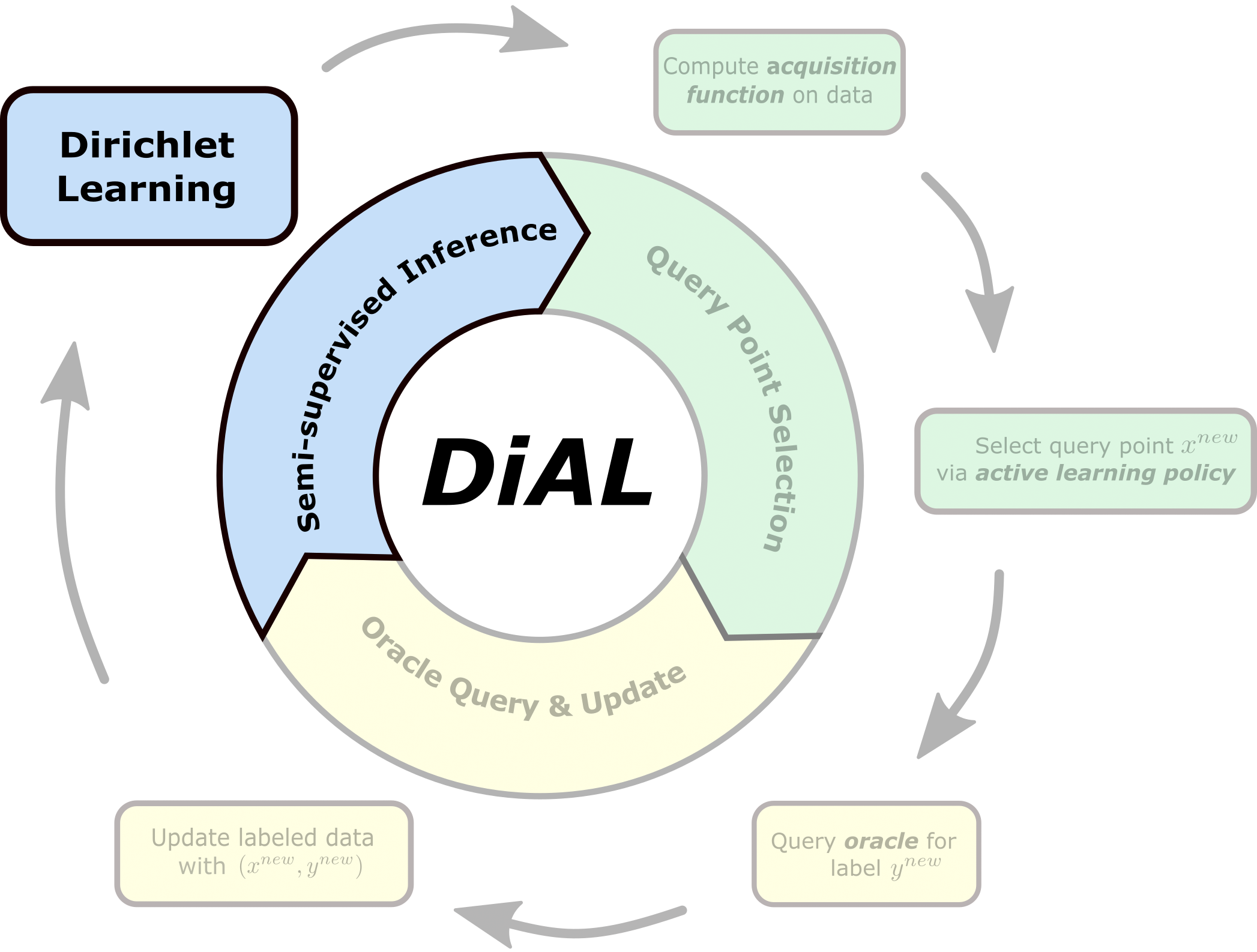}
    \caption{In \textit{Dirichlet Active Learning (DiAL)}, we specialize the semi-supervised inference step from the active learning process (Figure \ref{fig:al-flowchart}) to the \textit{Dirichlet Learning} semi-supervised classifier.}
    \label{fig:dial-ssl}
\end{figure}

The set of $\{P(x)\}_{x \in \mcl X}$ defines a Dirichlet-valued random field. In the interpretation as a semi-supervised classification model, we take the mean estimator $\hat{p}(x) := \mbb E[P(x)]$, where
\begin{equation}\label{eq:pki}
    \hat{p}_k(x) = \frac{\alpha_k(x) + \alpha_0}{\alpha_0 K + \sum_{m=1}^K \alpha_m(x)},
\end{equation}
to be the multinomial probability vector used to determine the inferred classification $Y(x)$ of $x$; that is, $\mbb P(Y(x) = k | \hat{p}(x)) = \hat{p}_k(x)$, so that we infer the classification via
\begin{equation}\label{eq:inference}
    \hat{y}(x) := \argmax_{k=1, 2, \ldots, K} \ \hat{p}_k(x) = \argmax_{k=1, 2, \ldots, K} \ \alpha_k(x).
\end{equation}
Furthermore, the uncertainty in our inference can be modeled by the covariance structure of the random variable $P(x)$. For example, we can model the uncertainty in $P(x)$ by considering the trace of the covariance matrix $C(x) \in \mbb R^{K \times K}$, where $C_{km}(x) = Cov(P_k(x), P_m(x))$. Define $\tilde{\alpha}_k(x) = \alpha_k(x) + \alpha_0$ and $\beta(x) =  \sum_{k=1}^K \tilde{\alpha}_k(x)$ so we can write
\begin{align}\label{eq:var}
    Tr[ C(x)] &= \sum_{k=1}^K Var(P_k(x)) = \sum_{k=1}^K \frac{\tilde{\alpha}_k(x) ( \beta(x) - \tilde{\alpha}_k(x) )}{(\beta(x))^2(\beta(x) + 1)} = \frac{(\beta(x))^2 - \sum_{k=1}^K (\tilde{\alpha}_k(x))^2}{(\beta(x))^2( \beta(x) + 1)}.
\end{align}
Later on, we will use this measure of uncertainty as the basis for our Dirichlet Variance acquisition function for DiAL. 
In summary, the salient features of our proposed model, namely (i) the inferred classification and (ii) the uncertainty in our belief about the inferred classification, are directly determined by the vectors $\{\alpha(x)\}_{x \in \mcl X}$.

\begin{example}[Running example: inference on finitely sampled data]
In many contexts, we have a large fixed quantity of unlabeled data $ \mcl X = \{x^i\}_{i=1}^n$. In that context, we will denote the point of interest by superscripts, namely $p(x^i) = p^i, \alpha(x^i) = \alpha^i, \hat y(x^i) = \hat y^i, \beta(x^i) = \beta^i$ and $C(x^i) = C^i$, %. In this context, 
and we will subsequently (see Example \ref{example:graph-based}) view the $x^i$ as the nodes of a graph. It is worth noting here that viewing the set of nodes as a random field is a common framework for imposing a Bayesian structure with the goal of quantifying uncertainty in semi-supervised learning, especially in graph-based methods \citep{zhu_combining_2003,miller_model-change_2021} which usually interpret learning problems in terms of a Gaussian random field with covariance structure derived from the graph Laplacian (see Section \ref{sec:prev-graph-based}). 
This is a powerful approach, which admits a direct Bayesian interpretation for graph-based regression problems, but the statistical interpretation for categorical problems is not as clear.
% } 
To the best of our knowledge, this is the first work to consider a Dirichlet-valued random field (which is inherently categorical) for graph-based semi-supervised learning and our application to active learning.

We also remark that in this finite data context, it may be useful to view the $y$'s associated with the $x$'s as deterministic (but unobserved). In this context, the ground truth $\nu(x^i,\cdot)$ would be concentrated on a single $y$ value.
\end{example}

\section{Propagation operators}

Given the framework in the previous section, we are now faced with the question of how to determine the vectors $\{\alpha(x)\}_{x \in \mcl X}$ from the currently observed labeled data $\{(x^\ell, y^\ell)\}_{\ell \in \mcl L}$. 
We introduce propagation of ``pseudo-label'' from the labeled data via the use of kernels.
In particular, given a positive definite kernel $\Ker:\mcl X \times \mcl X \to \mbb R_+$, we define the \textit{propagation from input $x \in \mcl X$} to be the function $\Ker_x(\cdot) = \Ker(x, \cdot) : \mcl X \rightarrow  \mbb R_+$. 
We require the following of $\Ker$: % for $x, x' \in \mcl X$:
\begin{itemize}
    \item (Normalization) $\Ker(x, x) = T$. When $T = 1$ this matches the classical Bayesian framework, in the sense that we are counting an observation of the semi-supervised problem at a labeled point $x^\ell$ as a single statistical trial. While the choice $T = 1$ is the most interpretable (and we adhere to this convention in our examples), in data-poor settings it may be advantageous to allow a single labeled pair $(x^\ell, y^\ell)$ to count as $T$ observations of a statistical trial, reflecting the fact that semi-supervised problems we are unlikely to repeat a trial at the same point $x^\ell$. 
    \item (Maximum Principle) \textbf{ $\Ker(x, x') \le \Ker(x, x)$}. This requirement reflects the choice that an observation at a point $x^\ell$ has the greatest effect, in terms of the pseudo-labels, at $x^\ell$.
\end{itemize}
We do not necessarily require that our kernel be symmetric, although for many natural choices symmetry will additionally hold.

Now let $\mcl L_k = \{x^\ell \in \mcl L : y^\ell = k\}$ be the subset of the labeled data $\mcl L$ that belong to class $k$. Then we can define the functions 
\[
    \Ker_{\mcl L_k}(x) = \sum_{x^\ell \in \mcl L_k} \Ker_{x^\ell}(x),
\]
and set the concentration parameter at $x$, $\alpha(x) \in \mbb R^K_+$, to be
\[
    \alpha(x) = \lp \Ker_{\mcl L_1}(x), \Ker_{\mcl L_2}(x), \ldots, \Ker_{\mcl L_K}(x)\rp^T = \lp \sum_{\ell \in \mcl L_1} \Ker(x^\ell, x), \sum_{\ell \in \mcl L_2} \Ker(x^\ell, x), \ldots, \sum_{\ell \in \mcl L_K} \Ker(x^\ell, x)\rp^T.
\]
% to construct the entries of the following vector-valued function
% \[
%     \bm{\Ker}_{\mcl L}(x) = \left[\Ker_{\mcl L_1}(x) \ \Ker_{\mcl L_2}(x) \ \ldots \ \Ker_{\mcl L_K}(x) \right] \in \mbb R^{K}.
% \]
% By setting $\alpha(x) = \bm{\Ker}_{\mcl L}(x)$, we are then utilizing these kernels to specify the quantity of ``pseudo-label'' that has been propagated from the labeled data to an unlabeled point $x$ according to classes:
% \[
%     \alpha(x) = \lp \sum_{\ell \in \mcl L_1} \Ker(x^\ell, x), \sum_{\ell \in \mcl L_2} \Ker(x^\ell, x), \ldots, \sum_{\ell \in \mcl L_K} \Ker(x^\ell, x)\rp.
% \]

\begin{example} Consider a distribution of data on the square $\mcl X = [0,1]^2$, with density described by the classical ``two moons'' distribution. While generally sources do not provide an exact expression for such a density, we construct this density using a kernel density estimator of a finite sample: the associated density is displayed in Figure \ref{fig:propagation-2d-kde}. 

For such continuum data, there are a variety of possible choices for $\Ker$ to describe the strength of the relationship between two points. One of the simplest such approaches is based upon radial basis functions (RBF), namely
\[
    \Ker(x, x') = \exp \lp -\frac{\|x - x'\|_2^2}{2\sigma^2}\rp.
\]
This choice of kernel has been previously utilized in active learning tasks \citep{karzand_maximin_2020}. We notice that such a choice of kernel is isotropic and therefore \textit{independent} of the underlying distribution of the data.

An alternative choice, which has previously been considered in the discrete setting under the name ``Poisson Learning'' \citep{calder_poisson_2020}, utilizes partial differential equations to construct a data-informed propagation operator. In particular, we can define our propagation operator via the expression
\[
   \Ker_P(x, x^\ast) = u_{x^\ast}(x),
\]
where $u_x^\ast$ is the solution to the partial differential equation
\begin{align} \label{eq:poisson-ctnm-2d}
    -\Delta_\rho u_{x^\ast} (x) + \tau u_{x^\ast}(x)&= \delta_{x^\ast}(x) \quad \text{ for } x \in [0,1]^2 \setminus \{x^\ast\}\\
    \nabla u_{x^\ast} (x) \cdot \boldsymbol{\nu} &= 0 \quad \text{ on } \partial [0,1]^2 \nonumber %\\
    % u_{x^\ast}(x^\ast) & = 1
\end{align}
where here we define $\Delta_\rho v := \frac{1}{\rho}\nabla \cdot (\rho \nabla v)$ . This kernel will not be symmetric, but will satisfy our normalization and Maximum Principle assumptions, and will be strongly data-adapted. A finite difference approximation of this solution is displayed in Figure \ref{fig:2d-propagation}, and demonstrates very attractive data-adapted propagation.

While this approach gives elegant propagation functions which respect data density and topology, kernel density estimation and finite difference approximation are computationally challenging in higher dimension. This motivates the graph-based approach as in Example \ref{example:graph-based}.

% The Poisson propagation is a finite difference approximation of a Laplace learning propagation on an underlying density that is estimated via the samples in the Two Moons dataset. That is, we constructed a Gaussian kernel density estimate from the sample of points shown in cyan in Figure \ref{fig:2d-propagation}, and used this to define the density $\rho$. We then solved for the solution $u_{x^\ast}$ to 
% \begin{align*}
%     \Delta_\rho u_{x^\ast} (x) - \tau u_{x^\ast}(x)&= 0 \quad \text{ for } x \in [0,1]^2 \setminus \{x^\ast\}\\
%     \nabla u_{x^\ast} (x) \cdot \mathbf{\nu} &= 0 \quad \text{ on } \partial [0,1]^2 \\
%     u_{x^\ast}(x^\ast) & = 1
% \end{align*}
% via a finite difference approximation, where $x^\ast \in [0,1]^2$ is the source of the propagation. We use $\tau = 10^{-3}$ for this propagation, and the resulting kernel is
% \[
%    \Ker_L(x, x^\ast) = u_{x^\ast}(x).
% \]
% Note that this kernel is not symmetric (i.e., $\Ker_L(x, x') \not=\Ker_L(x', x)$) and depends on the underlying dataset via the density $\rho$.
\end{example}

\begin{figure}
    \begin{subfigure}{0.5\textwidth}
        \centering
        \includegraphics[width=\textwidth]{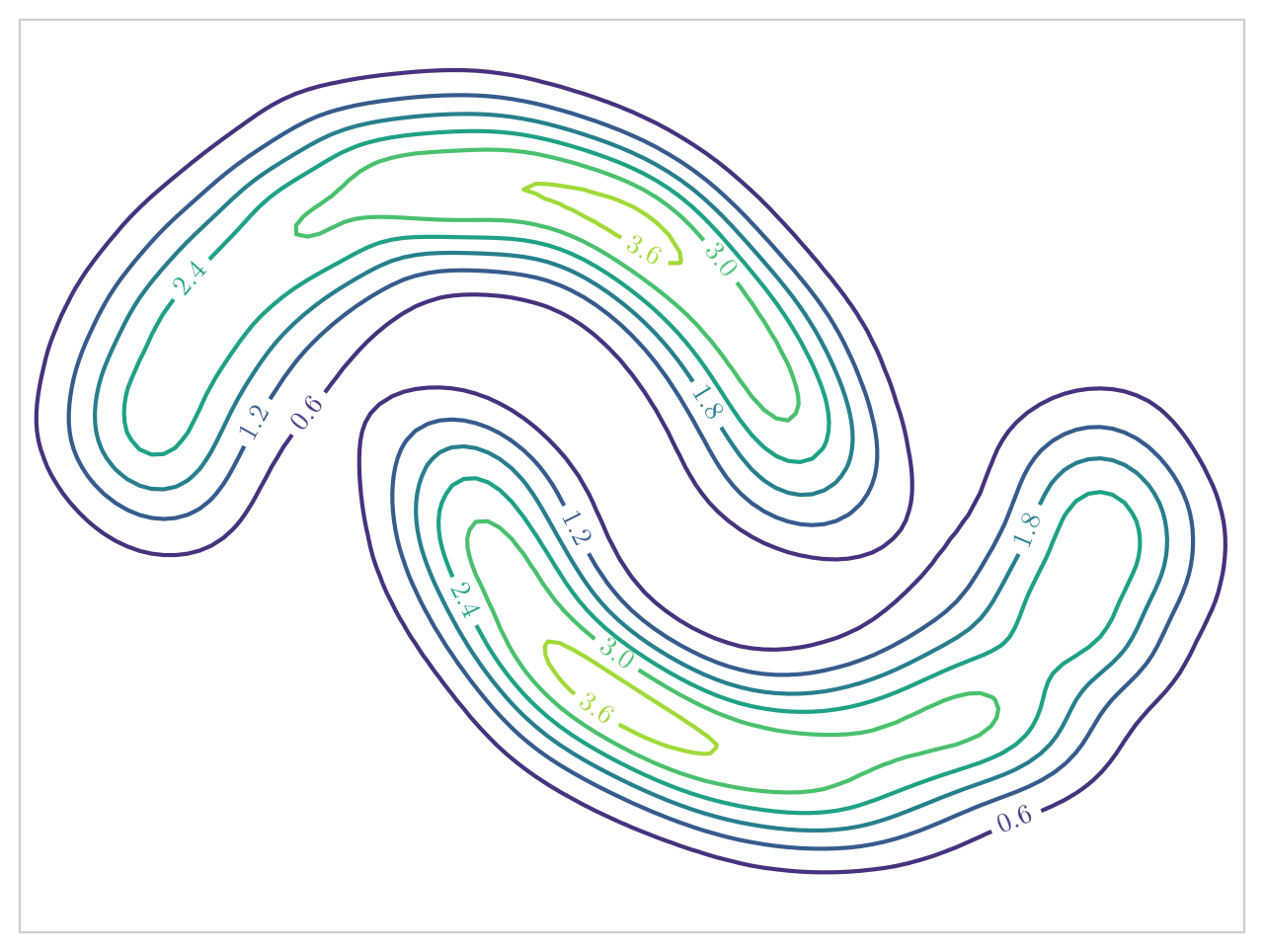}
        \caption{KDE of Finite Sample in panel (b)}
        \label{fig:propagation-2d-kde}
    \end{subfigure}
    \begin{subfigure}{0.5\textwidth}
        \centering
        \includegraphics[width=\textwidth]{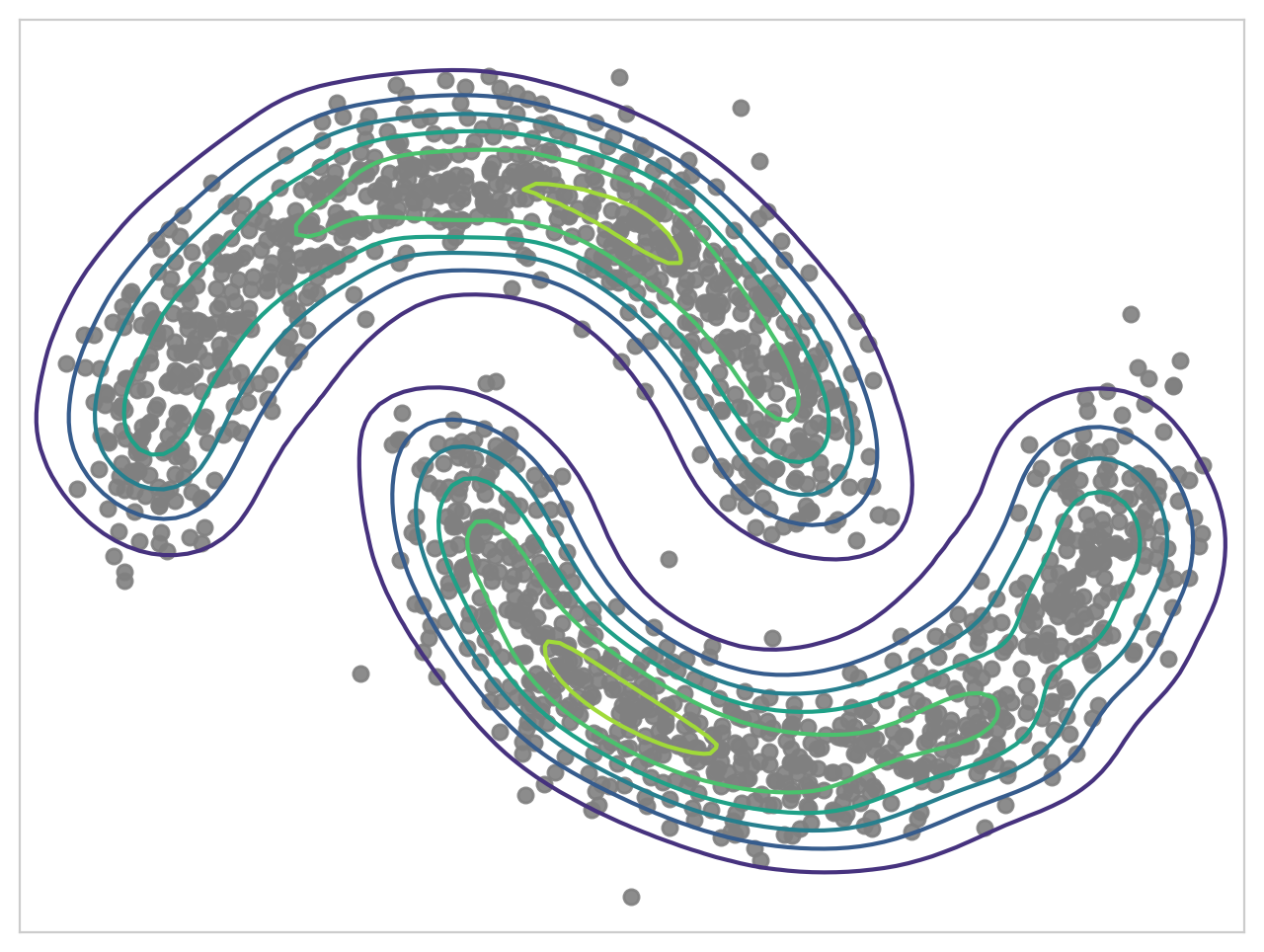}
        \caption{Sample}
    \end{subfigure}
    \newline 
    \begin{subfigure}{0.5\textwidth}
        \centering
        \includegraphics[width=\textwidth]{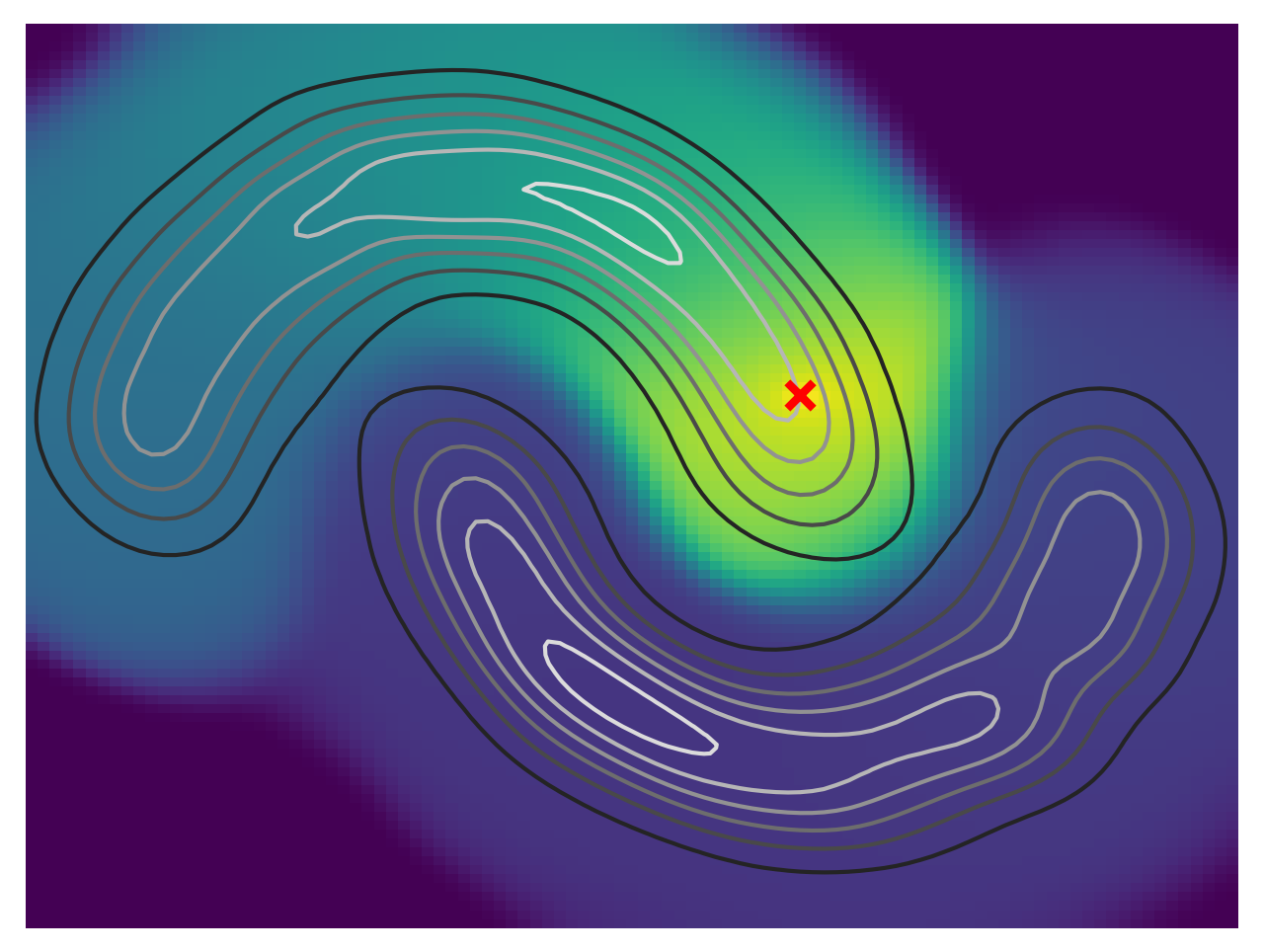}
        \caption{Continuum Poisson}
    \end{subfigure}
    \begin{subfigure}{0.5\textwidth}
        \centering
        \includegraphics[width=\textwidth]{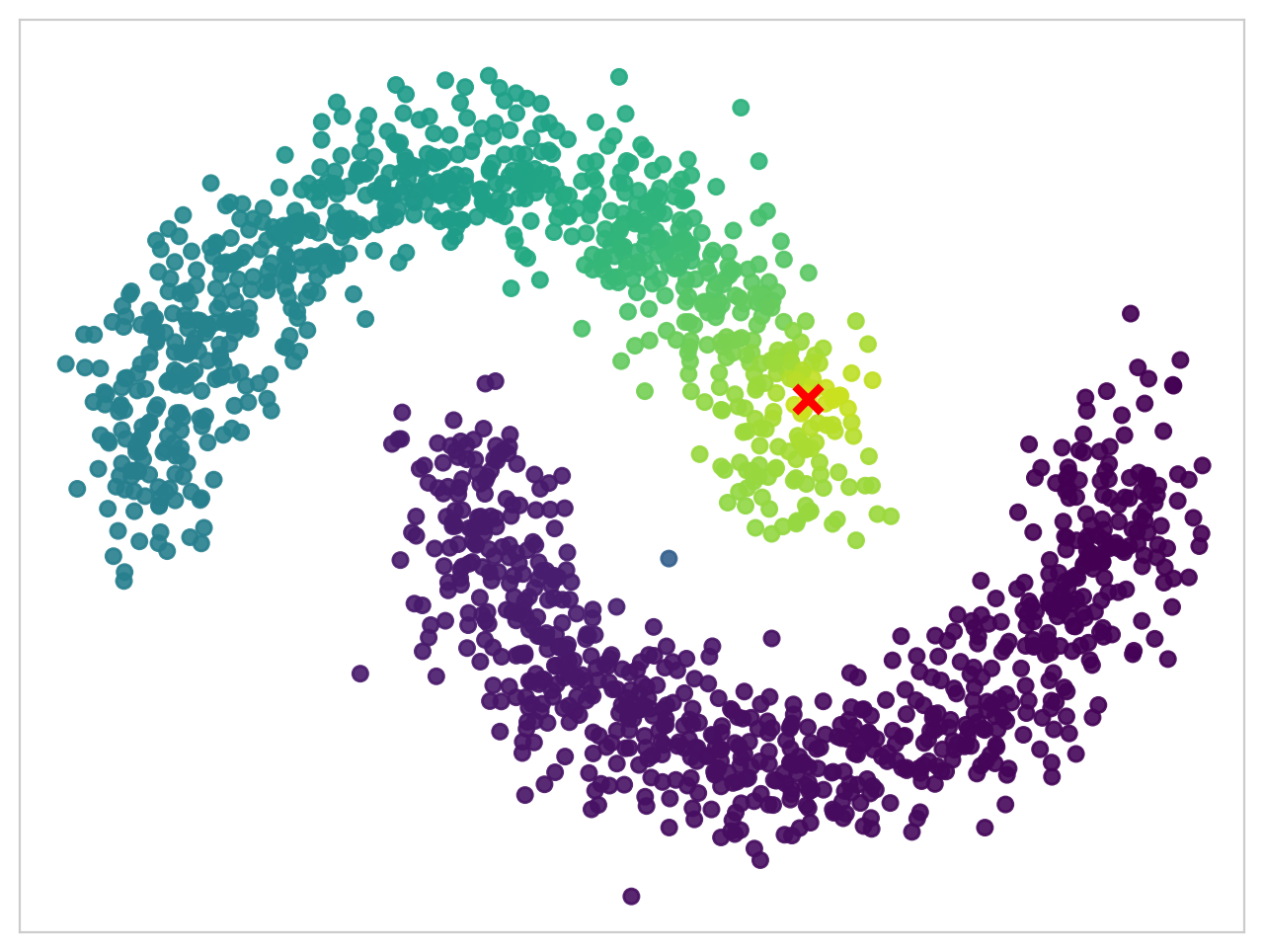}
        \caption{Graph-based Poisson}
    \end{subfigure}
    \newline 
    \begin{center}
        \begin{subfigure}{0.5\textwidth}
            \centering
            \includegraphics[width=\textwidth]{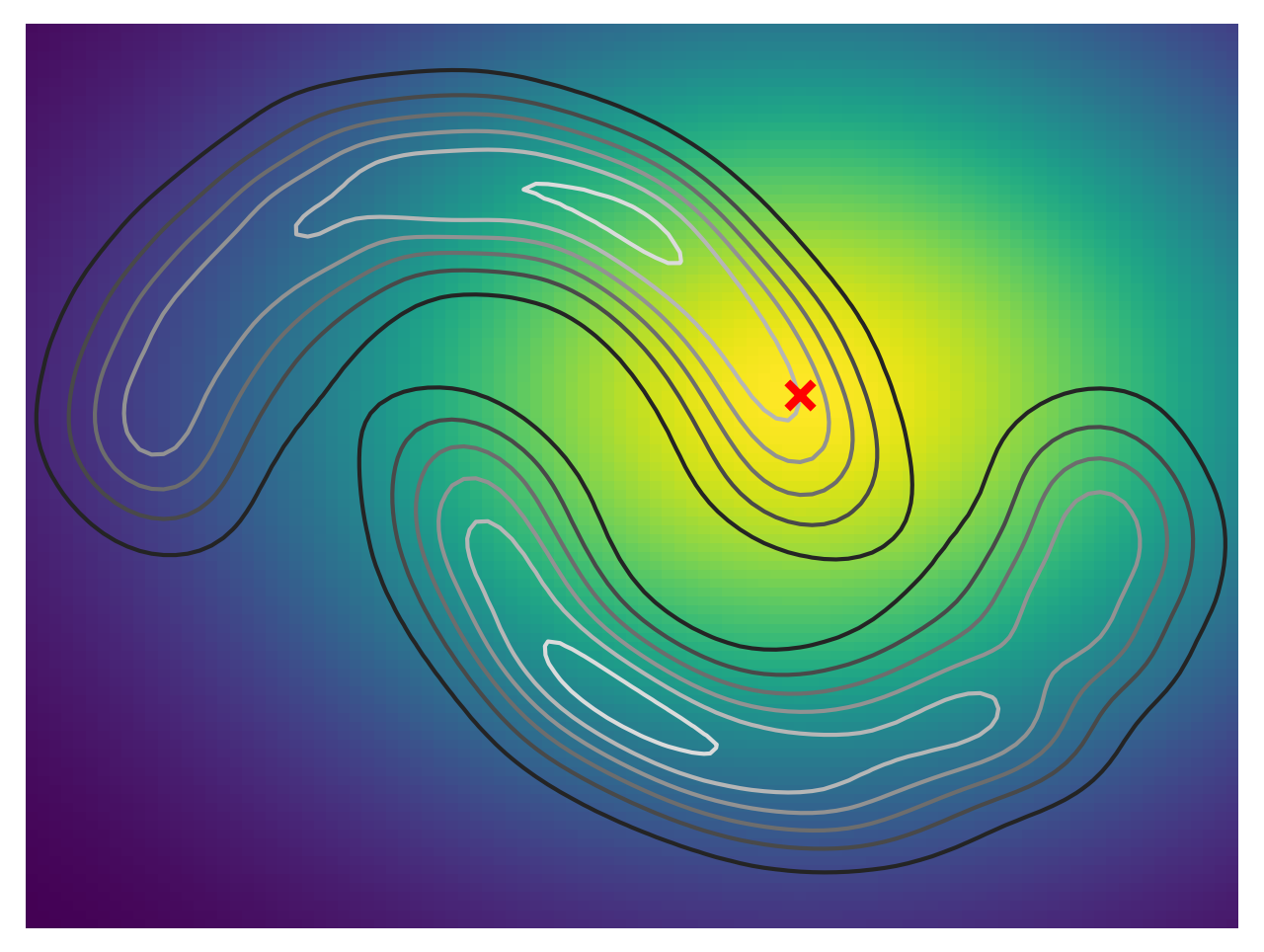}
            \caption{RBF Kernel $\Ker(x, x') = \exp\lp 
     -\frac{\|x - x'\|_2^2}{\sigma^2}\rp$}
        \end{subfigure}
    \end{center}
    \caption{Example propagation operators on a synthetic dataset (Two Moons) in two dimensions. Panel (a) shows the level sets of the underlying data-generating distribution $\rho(x)$ of the dataset, and panel (b) shows an empirical sample of $1500$ points from $\rho$. Panels (c-e) show different propagation operators on the domain, where (c) and (d) are respectively finite-difference and graph-based approximations to the Poisson propagation of Equation \eqref{eq:poisson-ctnm-2d}. These are inherently data \textit{adaptive} propagations that consider the data density $\rho$. In contrast, radial basis function (RBF) propagation, as shown in panel (e), is \textit{independent} of the data distribution, whereas the Poisson propagation depends on the data-generating distribution.}
    \label{fig:2d-propagation}
\end{figure}

\begin{example}[Continuation of running example: Graph-based propagation operators] \label{example:graph-based}
Let $G(\mcl X, W)$ be a similarity graph with finite node set $\mcl X = \{x^i\}_{i=1}^n \subset \mbb R^d$ 
%indexed by $[n] = \{1, 2, \ldots, n\}$ 
with edge weight matrix $W \in \mbb R^{n \times n}$, where the weight $0 \le W_{ij})$ 
% for a symmetric similarity kernel $\sigma : \mbb R^d \times \mbb R^d$ that 
captures the similarity between $x^i, x^j \in \mbb R^d$; that is, $W_{ij}$ is to be larger (smaller) when $x^i$ and $x^j$ are similar (dissimilar). Let $D = \operatorname{diag}(d_1, d_2, \ldots, d_n)$ be the diagonal {\it degree} matrix with $d_i = \sum_{j=1}^n W_{ij}$ denoting the degree of node $x^i$. We consider pseudo-label propagation from labeled $x^\ell \in \mcl L$ to the rest of the nodes in the graph via the use of graph Laplacian matrices, as mentioned in Section \ref{sec:prev-graph-based}.
% Graph Laplacian matrices are ubiquitous in graph-based semi-supervised learning methods (a few prominent examples include \citep{bertozzi_diffuse_2016, calder_poisson_2020, calder_properly-weighted_2020, zhu_semi-supervised_2003}) and important geometric information about the clustering structure of the underlying dataset can be reflected in their eigenvectors. 
% The {\it combinatorial} $L = D - W$, the {\it symmetric normalized} $L_s = I - D^{-1/2}W D^{-1/2}$, and the {\it random walk} $L_{rw} = I - D^{-1}W$ graph Laplacians 
% are a few examples of standard formulations these type of matrices. 
% These specific 
The {\it combinatorial} graph Laplacian, $L = D - W$, is a standard graph Laplacian matrix that
is known to be positive, semi-definite with real eigenvalues and eigenvectors, including a non-trivial null space. The geometric structure of the eigenvectors corresponding to the $k$ smallest eigenvalues of graph Laplacians forms the basis for spectral clustering \citep{von_luxburg_tutorial_2007}. 

% With the above graph formulation in hand, we now turn to defining the form of the pseudo label propagations according to the currently labeled data, $\{x^\ell, y^\ell)\}_{\ell \in \mcl L}$. 
% As each individual labeled data pair $x^\ell, y^\ell$ constitutes an observation of underlying joint probability distribution, we define the propagations on an individual level. 

Define a node function $g^\ell : \mcl X \rightarrow \mbb R$ (equivalently written as a vector $g^\ell \in \mbb R^n$ assuming the ordering on the nodes $\mcl X = \{x^1, x^2, \ldots, x^n\}$ of the graph) to be the solution to the following ``Poisson propagation'' \citep{calder_poisson_2020, miller2023unc}
% We now introduce the specific propagation operator that we focus on, which we will term ``Poisson propagation''. Let $z^\ell \in \mbb R^n$ be the solution of the linear system of equations
\begin{equation}\label{eq:poisson-prop}
    \lp L + \tau \mathrm{I}_n\rp g^\ell = e_\ell, % - \frac{1}{n}\mathbbm{1},
\end{equation}
for a given $\tau > 0$ where $e_\ell \in \{0,1\}^n$ is the $\ell^{th}$ standard basis vector in $\mbb R^n$.
We can then define a corresponding ``Poisson'' graph propagation operator $\Ker_P$ to be
\begin{equation}\label{eq:poisson-prop-shift}
   \Ker_P(x^\ell, x) = \frac{g^\ell(x) - (\min_{\tilde{x} \in \mcl X} g^\ell(\tilde{x}))}{g^\ell(x^\ell) - (\min_{\tilde{x} \in \mcl X} g^\ell(\tilde{x}))}.
\end{equation}
\end{example}

It is straightforward to see that the Normalization property is satisfied for \eqref{eq:poisson-prop-shift} in the previous example. The following Lemma demonstrates that the Maximum Principle property is satisfied by \eqref{eq:poisson-prop-shift}, which leverages the well-known Maximum Principle of the combinatorial graph Laplacian, $L$.

\begin{lemma}[Maximum Principle for Poisson propagation]
    Assume that the graph $G(\mcl X,W)$ is connected and that $x^\ell \in \mcl X$ is fixed. Then, the maximum of the Poisson propagation \eqref{eq:poisson-prop-shift} occurs at the source, $x^\ell$.
\end{lemma}

\begin{proof}
    By virtue of $\Ker_P(x^\ell, x)$ merely being a scaling and shifting of the solution $g^\ell$ to a linear system, we can simply just show that $g^\ell(x^\ell) \ge g^\ell(x^i)$ for all $x^i \in \mcl X$. The solution $g^\ell$ satisfies the system of equations
    \begin{equation*}
        \left[(L + \tau \mathrm{I}) g^\ell\right](x^i) = \begin{cases}
            1 & \text{ for } x^i = x^\ell, \\
            0 & \text{ otherwise},
        \end{cases}
    \end{equation*}
    % \rwm{Our sign convention about $L$ and $\Delta$ doesn't seem to match (one is positive and one is negative?)} \km{Should be fixed now} \rwm{So we're thinking $L \sim -\Delta$, correct?} \km{yes}   
    where $\left[L g^\ell\right](x^i) = \sum_{j=1}^n w_{ij} (g^\ell(x^i) - g^\ell(x^j))$ equivalently denotes the $i^{th}$ entry of the vector $L g^\ell \in \mbb R^n$. Recalling that $d_i = \sum_{j=1}^n w_{ij}$ denotes the degree of node $i$ and defining the indicator $\chi_{ij} = 1$ if $i = j$ and $0$ otherwise, then we have 
    \begin{align*}
        g^\ell(x^i) &= \frac{\chi_{i\ell}}{\tau + d_i} + \frac{1}{\tau + d_i} \sum_{j=1}^n w_{ij} g^\ell(x^j).
    \end{align*}
    
    If the function $g^\ell$ attains its maximum at $x^i \not= x^\ell$, then we see that
    \begin{align*}
        g^\ell(x^i) &= \frac{1}{\tau + d_i} \sum_{j=1}^n w_{ij} g^\ell(x^j) \le \frac{1}{\tau + d_i} \sum_{j=1}^n w_{ij} g^\ell(x^i) < g^\ell(x^i),
    \end{align*}
    which is a contradiction. Whereas, if the maximum occurs at $x^\ell$ then the equation 
    \begin{align*}
        g^\ell(x^\ell) &= \frac{1}{\tau + d_\ell} + \frac{1}{\tau + d_\ell} \sum_{j=1}^n w_{\ell j} g^\ell(x^j) \le \frac{1}{\tau + d_\ell} + \frac{1}{\tau + d_\ell} \sum_{j=1}^n w_{\ell j} g^\ell(x^\ell) = \frac{1}{\tau + d_\ell} + \frac{d_\ell}{\tau + d_\ell}g^\ell(x^\ell),
    \end{align*}
    which simply implies that $g^\ell(x^\ell) \le \tau^{-1}$. We conclude then that 
    \[
       \Ker_P(x, x^\ell) \le\Ker_P(x^\ell, x^\ell),
    \]
    as desired.
\end{proof}

\begin{remark}
    It should be noted that the above proof relies on a Maximum Principle for the \textbf{combinatorial} graph Laplacian matrix, $L$. This property does \textit{not} hold for all graph Laplacian matrices, such as the symmetric-normalized graph Laplacian, $L_s$. In this case, the term $L_s (g^\ell(x^i) - g^\ell(x^j)) = \sqrt{\frac{d_i}{d_j}}g^\ell(x^i)$ may not satisfy the Maximum Principle property due to disparate values of the degrees $d_i, d_j$. For the remainder of this current work, we focus solely on the use of the combinatorial graph Laplacian for the Poisson propagation defined in \eqref{eq:poisson-prop-shift}.
\end{remark}

It is insightful to contrast our setup and choice of propagation operator with previous Bayesian frameworks for graph-based semi-supervised and active learning. For simplicity, consider the binary classification task. In the Gaussian process/random field setting of works such as \citep{zhu_semi-supervised_2003, zhu_combining_2003}, the graph Laplacian is incorporated into the prior distribution as over node functions $u : \mcl X \rightarrow \mbb R$ as $u \sim \mcl N(0, (L + \tau I )^{-1})$ and reflects an \textit{a priori} assumption regarding the smoothness of likely node functions with respect to the graph topology. The covariance matrix of this prior distribution is intimately connected to the graph-based propagation operator that we consider here. Intuitively, this graph-based prior distribution in the Gaussian random field biases the posterior belief given labeled data toward node functions that have similar outputs for nodes that are connected in the graph. In this way, the ample supply of available unlabeled data can straightforwardly be incorporated into the Bayesian framework to admit more sample-efficient learning of the classification task under the assumption that the classification structure aligns with the clustering structure of the unlabeled data.

The motivation for our related graph-based propagation operator does not admit the same direct interpretation in terms of Bayesian inference. The implicit modeling assumption with its use in Dirichlet Learning is that a data-dependent propagation operator should reflect the clustering (geometric) structure of the dataset. This intuition is indeed what underlies the choice of the graph Laplacian-based prior distribution in the Gaussian random field setting but is not directly modeled in the prior belief for each Dirichlet random variable in our setting. Instead, the data-dependent propagation operator is incorporated implicitly into the likelihood model, wherein continuously valued amounts of pseudo-label influence from labeled points are given to other points in a manner that reflects the underlying clustering structure of the dataset. Intuitively, this choice of data-dependent propagation in the Dirichlet Learning classifier aims to achieve sample efficient exploration of clusters by DiAL which we discuss in the next section (Section~\ref{sec:query-point-selection}).

\paragraph{Bayesian prior interpretation.} Although the Dirichlet Learning model is not directly derived from a properly Bayesian setup, we can interpret the given model in a Bayesian framework.  For simplicity, assume that the set $\mcl X$  is finite, and consider the random matrix $\Pmat \in \mbb R_+^{n \times K}$ whose $i^{th}$ row corresponds to 
the Dirichlet random variable $\Pmat(x^i)$ defined at $x^i$; that is, $\Pmat$ is the concatenation of all the Dirichlet random variables over our set $\mcl X$. Let the $k^{th}$ entry of the probability vector $\Pmat(x^i) \in \Delta_{K}$ (i.e., the $(i,k)^{th}$ entry of the matrix $\Pmat \in \mbb R^{n \times K}$) be denoted as $\Pmat_k(x^i) \in [0,1]$.
% $\Pmat(x^i) \sim Dir(\alpha_1(x^i), \alpha_2(x^i), \ldots, \alpha_K(x^i))$ with concentration hyperparameter $\alpha(x^i) \in \mbb R_+^K$ containing entries $\alpha_k(x^i) = \alpha_0 + \sum_{z \in \mcl L_k} \Ker(z, x^i$ that are the sum of a uniform prior $\lp Dir(\alpha_0, \alpha_0, \ldots, \alpha_0) = Dir(\alpha_0 \mathbbm{1})\rp$ and the amount of ``pseudolabel'' propagated from each previously labeled point. That is, $\Pmat$ is the concatenation of all the Dirichlet random variables over our set $\mcl X$. 
Let $(X, Y) \in \mcl X \times [K]$ represent an observation of an input-output pair with instance $(X,Y) = (x, y)$ for some $x \in \mcl X$. 
% \todo{Don't love the notation here with the bold $p$... Thoughts?}

We seek a formula for the prior belief (with density $\pi(\Pmat | \alpha_0)$) on $\Pmat$ given our modeling assumption captured in the pseudo-label propagation from labeled points in Dirichlet Learning defined via $\Ker: \mcl X \times \mcl X \rightarrow \mbb R_+$. Recalling the definition of the probability density function $\varphi(\alpha, z)$ for a Dirichlet random variable in \eqref{eq:pdf-dir}, then by appealing to Bayes' law we can then write
% \begin{align*}
%      \mbb \Pmat( o | p , \alpha_0) \pi(p| \alpha_0) &= \mbb \Pmat(p | o, \alpha_0) \pi(o) \\
%     \implies \pi(p | \alpha_0) &= \frac{\mbb \Pmat(p | o, \alpha_0) \pi(o)}{\mbb \Pmat(o | p, \alpha_0)} \propto \frac{\prod_{z \in \mcl X} \frac{1}{B(\alpha_0 \mathbbm{1} + \Ker(x, z) e_{y}) }\prod_{k=1}^K [p_{k}(z)]^{\delta_{k,y}k(x, z) + \alpha_0 - 1}}{p_{y}(x)} \\
%     &= \lp \prod_{z \in \mcl X} \frac{\prod_{k=1} p_k(z)^{\alpha_0 - 1}}{B(\alpha_0 \mathbbm{1})} \rp \lp \prod_{z \in \mcl X } \frac{B(\alpha_0 \mathbbm{1})}{B(\alpha_0 \mathbbm{1} + \Ker(x,z)e_y)} [p_y(z)]^{\Ker(x,z) - \chi\{z = x\}} \rp   \\
%     &= \lp \prod_{z \in \mcl X} \varphi(\alpha_0\mathbbm{1}, p(z)) \rp \underbrace{\lp \prod_{z\in \mcl X} \frac{\Gamma(\alpha_0)}{\Gamma(K\alpha_0) } \rp}_{\text{const.}} \lp \prod_{z \in \mcl X } \frac{\Gamma(K\alpha_0 + \Ker(x,z))}{\Gamma(\alpha_0 + \Ker(x,z)) } [p_y(z)]^{\Ker(x,z) - \chi\{z = x\}} \rp  
%     % &= p^{(i)}_{y_i} \prod_{x \in \mcl X} \frac{\Gamma(K\alpha_0 + \Ker(x^i, x)) }{\prod_{k=1}^K \Gamma(\alpha_0 + \delta_{k,y_i}k(x^i, x)) [p_{k}(x)]^{\delta_{k,y_i[^_]\Ker(x^i, x) + \alpha_0 - 1}}
% \end{align*}
% }
% \todo{We have assumed that the observation of $o = (x, y)$, given the matrix $p$ is fixed, is simply the probability $p_y(x)$ with a uniform probability of selecting any $x \in \mcl X$. This still has $o$ dependence as seen in the fact that $\Ker(x,z)$ and $y$ appear in the rightmost term.}
% \km{
% Can consider a different approach to finding $\pi(p)$:
\begin{align*}
    \pi(\Pmat | \alpha_0) &= \sum_{x \in \mcl X} \sum_{y = 1}^K \mbb P(\Pmat| x, y,  \alpha_0) \pi(x,y)  \\
    &= \lp \prod_{z \in \mcl X} \varphi(\alpha_0\mathbbm{1}, \Pmat(z)) \rp \sum_{x \in \mcl X} \sum_{y = 1}^K \prod_{z \in \mcl X} \frac{\Gamma(\alpha_0)^{K-1} \Gamma(\alpha_0 + \Ker(x,z)) \Pmat_y(z)^{\Ker(x,z)}}{\Gamma(K\alpha_0 + \Ker(x,z))} \mbb P(Y = y|x) \pi(x) \\
    &\propto \lp \prod_{z \in \mcl X} \varphi(\alpha_0\mathbbm{1}, \Pmat(z)) \rp \sum_{x \in \mcl X} \lp \prod_{z \in \mcl X} \frac{\Gamma(\alpha_0 + \Ker(x,z))}{\Gamma(K\alpha_0 + \Ker(x,z))} \rp \pi(x) \left\{ \sum_{y = 1}^K \prod_{z \in \mcl X} \Pmat_y(z)^{\Ker(x,z)} \Pmat_y(x) \right\}.
\end{align*}
Now, if we set $\pi(x) = \frac{1}{|\mcl X|}$ 
% \todo{change to no longer be uninformative?} and $\mbb \Pmat(Y = y|x) = p_y(x)$ 
to be an \textit{uninformative} prior on the observation data's input, then we can further simplify as
\begin{align*}
    \pi(\Pmat | \alpha_0) &\propto \lp \prod_{z \in \mcl X} \varphi(\alpha_0\mathbbm{1}, \Pmat(z)) \rp \sum_{x \in \mcl X} \overbrace{\lp   \prod_{z \in \mcl X} \frac{\Gamma(\alpha_0 + \Ker(x,z))}{\Gamma(K\alpha_0 + \Ker(x,z))} \rp}^{w(x) :=}  \overbrace{\left\{ \sum_{y = 1}^K \Pmat_y(x) \prod_{z \in \mcl X} \Pmat_y(z)^{\Ker(x,z)} \right\}}^{s(\Pmat, \Ker_x) := } \\
    &= \lp \prod_{z \in \mcl X} \varphi(\alpha_0\mathbbm{1}, \Pmat(z)) \rp \sum_{x \in \mcl X} w(x) s(\Pmat, \Ker_x),
\end{align*}
where have recalled the definition of $\Ker_x : \mcl X \rightarrow \mbb R^+$ as the propagation function from $x \in \mcl X$.
The quantity $w(x)$ then corresponds to a measure of ``centrality'' that weights each $x \in \mcl X$, while $s(\Pmat, \Ker_x)$ captures a measure of ``alignment'' between the probability distribution represented by $\Pmat$ and the chosen kernel's propagation at $x \in \mcl X$, $\Ker_x(\cdot)$. The interpretation of $w(x)$ as a measure of centrality simply follows from the observation that the ratio of the $\Gamma$ functions increases as $\Ker(x,z)$ increases; thus, a more ``central'' point $x$ that is similar to a greater proportion of the dataset will have a larger weight $w(x)$. See Example~\ref{ex:kernel-prob-alignment} for a simplified example to demonstrate how the quantity $s(\Pmat, \Ker_x)$ captures the ``alignment'' between $p$ and the kernel propagation, $\Ker_x$.

\begin{example}[Kernel and probability alignment] \label{ex:kernel-prob-alignment}
Consider a binary classification setting where the dataset is clustered simply into two disjoint sets $\mcl X = \mcl X_1 \cup \mcl X_2$ and that the kernel $\Ker$ perfectly discriminates between these clusters:
\[
    \Ker(x,z) = \begin{cases}
        1 & \text{if there exists } k \in \{1,2\} \text{ s.t. } x,z \in \mcl X_k \\
        0 & \text{otherwise}
    \end{cases}.
\]
If the probability matrix $\Pmat$ reflects this clusteredness (e.g., $\Pmat_1(x) = 1$ if $x \in \mcl X_1$ and $0$ otherwise), then 
\begin{align*}
    \prod_{z \in \mcl X} \Pmat_y(z)^{\Ker(x,z)} &= \prod_{z \in \mcl X_y} \Pmat_y(z)^{\Ker(x,z)} \prod_{z \not\in \mcl X_y} \Pmat_y(z)^{\Ker(x,z)}= \prod_{z \in \mcl X_y} 1^{\Ker(x,z)} \prod_{z \not\in \mcl X_y} 0^{\Ker(x,z)} \\
        &= \begin{cases}
            \prod_{z \in \mcl X_y} 1^1 \prod_{z \not\in \mcl X_y} 0^0 & \text{ if } x \in \mcl X_y \\
            \prod_{z \in \mcl X_y} 1^0 \prod_{z \not\in \mcl X_y} 0^1 & \text{ if } x \not\in \mcl X_y \\
        \end{cases} = \begin{cases}
            1 & \text{ if } x \in \mcl X_y \\
            0 & \text{ if } x \not\in \mcl X_y \\
        \end{cases} \\
    \implies s(\Pmat, \Ker_x)&= \sum_{y=0,1} \Pmat_y(x) \cdot \chi\{x \in \mcl X_y\} = 1 ,
\end{align*}
where we have defined $0^0 = 1$.
\newline 
\indent In contrast, consider a $\tilde{\Pmat}$ that is very misaligned with the clustering structure, such as one that splits the clusters in half as shown in Figure \ref{fig:cluster-align}(b):
\[
    \tilde{\Pmat}_1(z) = \begin{cases}
        1 & \text{ if } z \in \mcl X_+ \\
        0 & \text{ if } z \in \mcl X_{-}
    \end{cases}, \qquad \tilde{p}_2(z) = \begin{cases}
        0 & \text{ if } z \in \mcl X_+ \\
        1 & \text{ if } z \in \mcl X_{-}
    \end{cases}.
\]
Then, defining $\mcl X_{k,\pm} = \mcl X_k \cap \mcl X_{\pm}$, we have 
\begin{align*}
    \prod_{z \in \mcl X} \tilde{\Pmat}_1(z)^{\Ker(x,z)} &= \prod_{z \in \mcl X_{1,+}} 1^{\Ker(x,z)} \prod_{z \in \mcl X_{1,-}} 0^{\Ker(x,z)} \prod_{z \in \mcl X_{2,+}} 1^{\Ker(x,z)} \prod_{z \in \mcl X_{2,-}} 0^{\Ker(x,z)}\\
        &= \begin{cases}
            0 & \text{ if } x \in \mcl X_1 \\
            0 & \text{ if } x \in \mcl X_2 \\
        \end{cases}.
\end{align*}
With a similar computation for $\tilde{\Pmat}_2(z)$, we can see then that $s(\tilde{\Pmat},\Ker_x)= 0$. This simple example highlights how this quantity $s(\Pmat, \Ker_x)$ measures the alignment between the class probabilities of an instance $\Pmat$ and the geometry of the data as reflected by the kernel. 
\newline 
\indent To summarize, the prior probability for a very misaligned $\tilde{\Pmat}$ is $\pi(\tilde{\Pmat} | \alpha_0) = 0$, while the prior probability for a very well-aligned $\Pmat$ is $\pi(\Pmat | \alpha_0) = \lp \prod_{z \in \mcl X} \varphi(\alpha_0\mathbbm{1}, \Pmat(z)) \rp \sum_{x \in \mcl X} w(x) > 0$. This demonstrates the corresponding prior's preference for probability outputs $p$ that are well-aligned with the inherent clustering structure of the dataset.
\end{example}

\begin{figure}
    \begin{subfigure}{0.49\textwidth}
        \includegraphics[width=\textwidth]{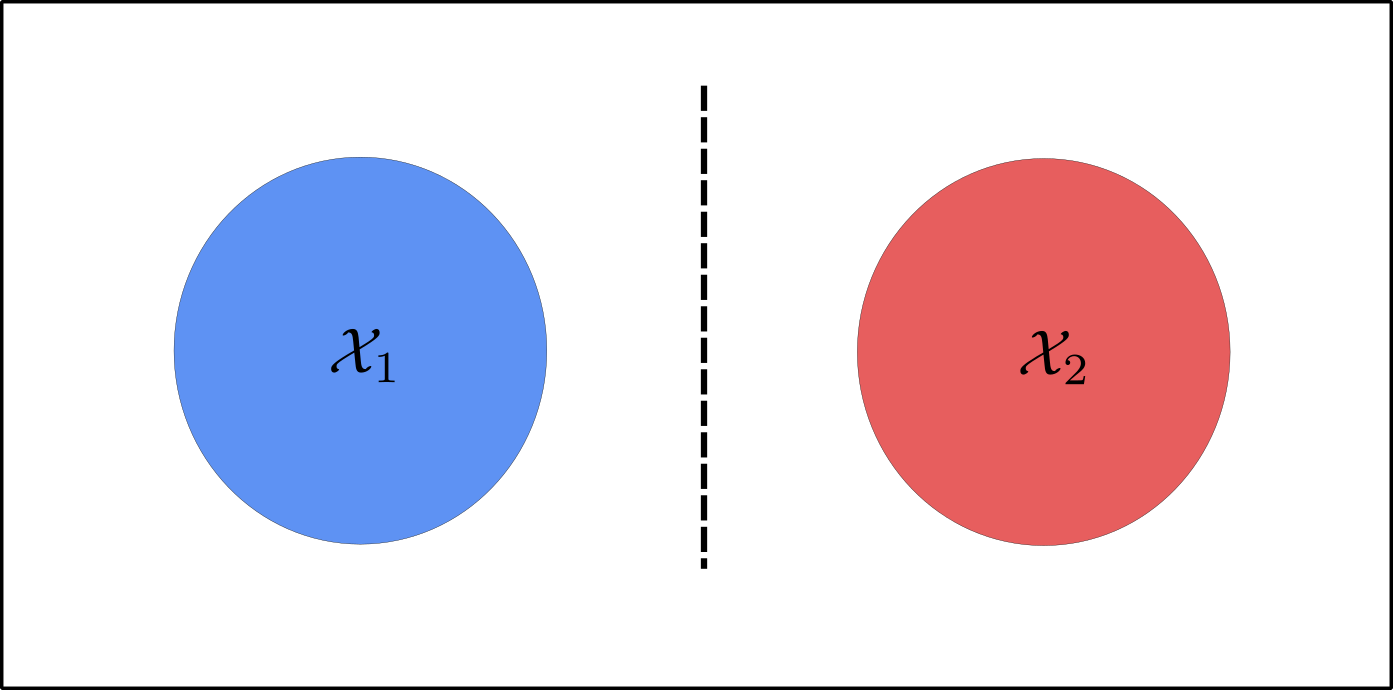}
            \caption{Perfectly aligned $\Pmat$ (i.e., $s = 1$)}
    \end{subfigure}
    \hfill
    \begin{subfigure}{0.49\textwidth}
        \includegraphics[width=\textwidth]{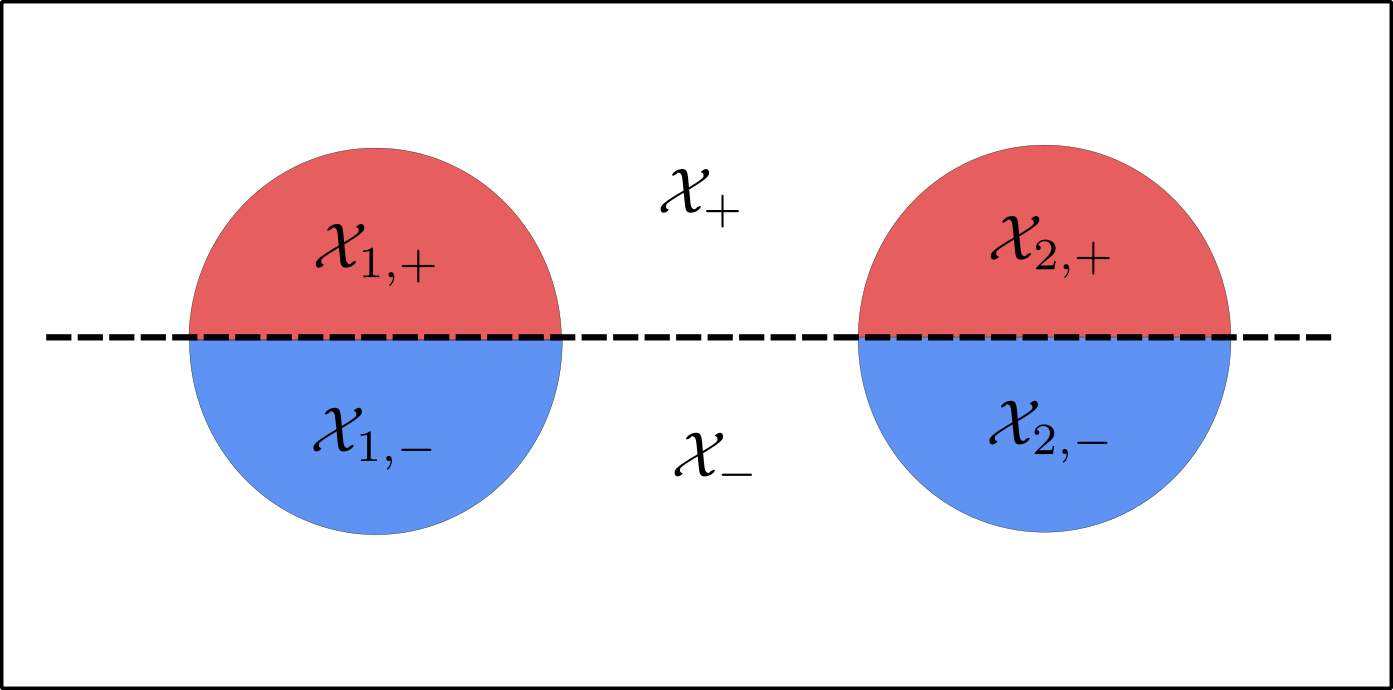}
            \caption{Misaligned $\tilde{\Pmat}$ (i.e., $s = 0$)}
    \end{subfigure}
    \caption{Example of clusters and probability matrix $\Pmat$ setup for Example \ref{ex:kernel-prob-alignment}. The dataset can be written as the union of 2 disjoint clusters, $\mcl X = \mcl X_1 \cup \mcl X_2$, where the circles respectively represent $\mcl X_1$ and $\mcl X_2$. The coloring of regions represents the classification of the points in each cluster and the dotted line denotes the decision boundary between classes according to the different probability matrices $\Pmat$. In the case of panel (a), the classification proposed by $\Pmat$ perfectly aligns with the clustering structure and leads to a large value of $s = 1$. In contrast, the horizontal decision boundary characterizing $\tilde{\Pmat}$ in panel (b) is exactly \textit{misaligned} with the clustering structure and results in a value of $s = 0$.  }
    \label{fig:cluster-align}
\end{figure}

\section{Query point selection} \label{sec:query-point-selection}

The vector-valued function $\alpha(x) = \lp \Ker_{\mcl L_1}(x), \Ker_{\mcl L_2}(x), \ldots , \Ker_{\mcl L_K}(x) \rp^T : \mcl X \to \mbb R_+^K$ (which may be organized as a matrix in the case where $\mcl X$ is finite) expresses our information about the probabilities of the different class labels at every point in $\mcl X$, in terms of the underlying Dirichlet random field. As discussed earlier, given a particular semi-supervised method there are many different possible approaches for identifying new points at which to acquire data. For clarity, we introduce now a distinction between an active learning \textit{acquisition function} and \textit{policy}. An acquisition function $\mcl A(x; \mcl L_1, \mcl L_2, \ldots, \mcl L_K)$ evaluated on inputs $x \in \mcl X$ quantifies how useful our model believes it would be for the active learner to query its label. This acquisition function is user-defined and is designed to reflect the desired properties of the query points to be labeled throughout the active learning process. While this function depends on the data points in $\mcl L_1, \ldots, \mcl L_K$ and their associated labels, we will forego explicitly writing this dependence in favor of readability; namely, we will write $\mcl A(x) = \mcl A(x; \mcl L_1, \ldots, \mcl L_K)$ with the understanding of the dependence on the currently labeled data. 
% \rwm{could there be confusion stemming from $A(x)$ vs $\mcl A(x)$?} \km{good point, we could consider using $\mcl Q(x)$ for the acquisition function?}

\begin{figure}[t]
    \centering
    \includegraphics[width=0.7\textwidth]{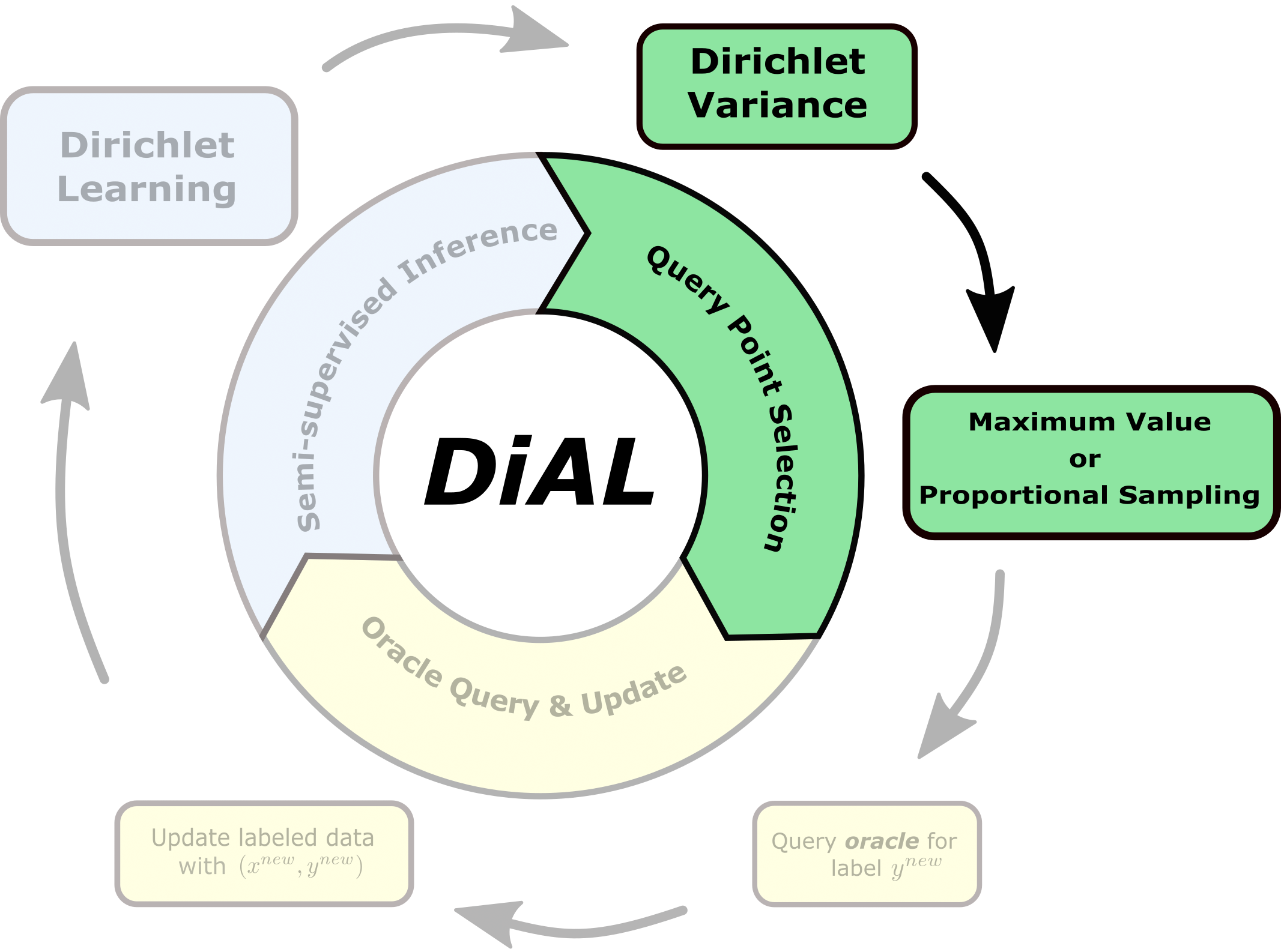}
    \caption{Utilizing the \textit{Dirichlet Learning} semi-supervised classifier introduced in Section \ref{sec:beta-learning}, \textit{DiAL} selects query points using the \textit{Dirichlet Variance} acquisition function \eqref{eq:dirichlet-variance-af} with either the \textit{Maximum Value} or \textit{Proportional Sampling} policies.}
    \label{fig:dial-qps}
\end{figure}

Now, with a chosen acquisition function it remains to decide how to select the next query point $x^\ast$ from the set of acquisition function values on the unlabeled data, $\{\mcl A(x)\}_{x \in \mcl U}$. We refer to the method for selecting the query point from said values as the active learning \emph{policy}. Among many possible choices, we focus our attention on two natural choices: (1) \textit{Maximum Value} and (2) \textit{Proportional Sampling}.

Maximum Value policy chooses to query the label of the point that maximizes $\mcl A$ on the unlabeled data
\[
    x^\ast_{MV} = \argmax_{x \in \mcl U} \ \mcl A(x).
\]
The majority of sequential active learning methods previously proposed fall under this category \citep{settles_active_2012, miller_model-change_2021, balcan_margin_2007, miller2023unc, zhu_combining_2003, ma_sigma_2013, ji_variance_2012, jiang_minimum-margin_2019}, where cases of acquisition functions that are to be minimized can be equivalently rephrased to maximize the negative of acquisition functions values.

Proportional Sampling selects query points via randomly sampling according to a probability distribution derived from the acquisition function values over the set of unlabeled inputs; for example, we can select $x^\ast_{PS} = x$ for $x \in \mcl U$ with probability 
\[
    q(x) \propto \exp\lp\lambda \mcl A(x)\rp 
\]
for scaling factor $\lambda > 0$. This distribution over $x \in \mcl U$ encourages the selection of points with larger acquisition function values at a given iteration. 
Note that as $\lambda \rightarrow \infty$, this distribution concentrates on the maximizer $x^\ast = \argmax \mcl A(x)$, whereas as $\lambda \rightarrow 0^+$, the distribution converges to a uniform distribution over the unlabeled data. In Section \ref{subsec:lambda-choice}, we discuss how we choose $\lambda > 0$ for our numerical experiments, and we identify some properties of the choice of $\lambda$ in an asymptotic regime of DiAL with Prop.~Sampling.

We note that this kind of ``softmax'' scaling for a sampling distribution has recently been used in other active learning works to encourage diverse batches of query points \citep{kirsch2022stochastic} and to correct for the sampling bias of uncertainty sampling with classifiers found via empirical risk minimization \citep{zhan2022asymptotic}. Furthermore, a similar idea of ``proportional sampling'' has been used in randomized numerical linear algebra methods such as \citep{deshpande2006adaptivesampling, musco2017adaptivesampling, chen2023randomly} for column subset selection and low-rank matrix approximation of positive semi-definite matrices. Depite this shared idea of proportional sampling distributions for selecting inputs, the nature of our theoretical results is quite distinct from these previous works.

With the concepts of acquisition functions and policies in hand, we next introduce a few natural examples of acquisition functions.

\subsection{Uncertainty sampling}

One common active learning approach is to query points where the current classifier has the most ``uncertainty'' about the class selection, a framework oftent referred to as \textit{uncertainty sampling} \citep{settles_active_2012, miller2023unc}. This is often expressed by selecting points where the classifier's output class probabilities are most alike. In the context of binary classification, we could quantify this uncertainty using the ``smallest margin'' acquisition function
\[
    \mcl A_{unc}(x) =  - \lp \max_{k=1,2} \hat p_k(x) - \min_{k=1,2} \hat p_k(x) \rp  = - |\hat p_1(x) - \hat p_2(x)|,
\]
where here the $\hat p$ are class probabilities outputted by a semi-supervised algorithm, for example using \eqref{eq:pki} from our proposed Dirichlet Learning model. We have defined the acquisition function to be the \textit{negative} of the margin value $|\hat p_1(x) - \hat p_2(x)|$ in order to obey our policy convention of \textit{maximization}. With this choice of acquisition function, then the MV policy would select the query point $x^\ast_{MV} = \argmax_{x \in \mcl X} \mcl A(x) = \argmin_{x \in \mcl X} |\hat p_1(x) - \hat p_2(x)|$: this should focus on labeling points that are ``closest'' to the current classifier's decision boundary. 
For the multi-class setting various generalizations are possible for the margin value given above: for example, one could use $\max_{k \in [K]} \hat p_k(x) - \min_{k \in [K]} \hat p_k(x)$ or the difference in the probabilities of the two most likely classes.

While one could directly apply this uncertainty sampling acquisition function to our Dirichlet Learning model outputs at each iteration, one of the advantages of the Dirichlet random field approach is that we possess significantly more information than just the mean probability estimator. In particular, we could instead define uncertainty in the classifier's outputs for $x$ to be the variance of the Dirichlet distribution of the class labels at the point $x$, which we recall from \eqref{eq:var} to be 
\begin{equation} \label{eq:dirichlet-variance-af}
    \mcl A_{var}(x) := Tr[C(x)] =  \frac{(\beta(x))^2 - \sum_{k=1}^K (\tilde{\alpha}_k(x))^2}{(\beta(x))^2( \beta(x) + 1)}.
\end{equation} 
where $\beta(x) =  \sum_{k=1}^K \tilde{\alpha}_k(x)$.
This approach uses a very different conceptual approach to defining uncertainty: instead of uncertainty being an issue of similarly probable class labels under a classifier, it instead becomes a lack of information about the probabilities of those class labels. We will term the acquisition function of \eqref{eq:dirichlet-variance-af} to be \textit{Dirichlet Variance}. Using this acquisition function, we may apply either active learning policy for selecting the next query point. To be clear, we respectively term the maximum value and proportional sampling policies of the Dirichlet Variance acquisition function to be ``Dir.\,Var.\,'' and ``Dir.\,Var.\,(Prop)''; see (Table \ref{table:dirvar-defs}) for a summary.

\begin{table}[h!] 
\centering
\begin{tabular}{@{}ccc@{}}\toprule
Name & Policy & Formula\\
\midrule 
Dir.\,Var.\, & \textit{Maximum Value} & $x^\ast = \argmax_{x \in \mcl U} \ \mcl A_{var}(x)$ \\
Dir.\,Var.\,(Prop), & \textit{Proportional Sampling} &  $x \in \mcl U,\ q(x) \sim \exp \lp \lambda \mcl A_{var}(x)\rp$ \\%$x \in \mcl U,\ q(x) \sim \exp \lp \frac{\mcl A_{var}(x)}{t}\rp$ \\
\bottomrule
\end{tabular}
\caption{Summary of two methods which we study based upon the proposed ``Dirichlet Variance'' acquisition functions, $\mcl A_{var}(x) = \frac{(\beta(x))^2 - \sum_{k=1}^K (\tilde{\alpha}_k(x))^2}{(\beta(x))^2( \beta(x) + 1)}$. The first, Dir.\,Var.\, selects a query point via the maximizer of the Dirichlet variance, while Dir.\,Var.\,(Prop) selects the next query point with probability $q(x) \propto \exp \lp \lambda \mcl A_{var}(x)\rp$ } %$q(x) \propto \exp\lp \mcl A_{var}(x)/t\rp$.}
\label{table:dirvar-defs}
\end{table}

\begin{remark} \label{remark:2types-unc}
    We briefly discuss here the various concepts of ``uncertainty'' in active learning. Namely, we suggest there are (at least) two types of uncertainty to consider and model in the active learning process: (1) data-conditional uncertainty and (2) underlying population-level classification uncertainty. Data-conditional uncertainty reflects the idea that given the current labeled data and assumed hypothesis class, how uncertain is the current classifier about the inferred classifications on the unlabeled data? Both the traditional notion of ``uncertainty sampling'' and our proposed Dirichlet random variable's measure of variance reflect two ways of modeling this data-conditional uncertainty. We notice that for Dirichlet Variance, we expect the data-conditional uncertainty to go to zero in the limit of infinitely many labeled data points (see Property~\ref{property:dec-var-more-data} at the end of this section). However, the variance goes to zero at different rates in regions with different population-level uncertainty, a phenomenon we explore in the computations below and in Section \ref{subsec:exploit-analysis}. 
    \newline
    \indent On the other hand, the underlying population-level uncertainty reflects the inherent uncertainty of the data-generating distribution (e.g., regions where class-conditional distributions are large for multiple classes). We also note that in the large, labeled data limit it is natural to guess that the data-conditional uncertainty associated with ``uncertainty sampling'' will approach what we call the population-level uncertainty.
    \newline
    \indent Various types of goals for active learning algorithms can be explained in terms of these uncertainties.  For example, in settings with very few labeled data points the data-conditional uncertainty is expected to be quite high, and the goal of an active learning algorithm is often to appreciably decrease this uncertainty across a wide range of points. This type of behavior is sometimes called exploratory behavior. On the other hand, as the number of labeled data points increases, a possible goal for active learning is to focus attention on regions with high population-level uncertainty, with the goal being to effectively learn high-quality decision boundaries. Good active learning algorithms likely need to balance these two goals and transition reasonably from one to the other as more labels are obtained: we discuss this more in Section~\ref{sec:theory}.
\end{remark} 

\begin{example}[1D visualization of two types of uncertainty] \label{example:2types-unc}
     Consider the binary classification case (with labels $y \in \{0,1\}$ as opposed to $y \in \{1,2\}$) with ground-truth, class-conditional distributions $\rho_0(x) = p(x | y=0)$ and $\rho_1(x) = p(x | y=1)$ shown as the green and gray shaded regions in Figure \ref{fig:unc-demo}(a). 
     % These are the true underlying class conditional distributions for the class $y = 0$ and $y=1$, respectively. 
     The \textbf{black} dashed line represents the {population-level uncertainty}, wherein these class-conditional probabilities are both large. In particular, this black dashed line is computed as $\min\{ p(x,y=0), p(x, y=1)\}$, peaking at the locations when the conditionals are both relatively large. The implicit assumption when applying uncertainty sampling acquisition functions for active learning is that they should focus on sampling in these regions of ``large'' population-level uncertainty. However, as we illustrate in this example, regions of high {data-conditional uncertainty} according to a given model class of functions do not necessarily reflect population-level uncertainty.
     
     \indent The example begins with two initially labeled points, chosen from the largest clusters of the respective class-conditional distributions; in panel (b) these are labeled as blue x's while in panel (c) these are labeled as red squares. Panels (b) and (d) show the evolution of the Dirichlet Variance acquisition function as thirty query points are sequentially selected to maximize Dirichlet Variance at each iteration. The selected query point at each iteration is randomly assigned a label of $y = 1$ with probability $p(y=1 | x)$. Similarly, panels (c) and (e) show the evolution of the smallest margin acquisition function (Unc.\,(SM)) using an SVM classifier using the RBF kernel (specifically for $\Ker(z,x) = \exp\lp -\gamma |z- x|^2\rp$ where we set the kernel bandwidth $\gamma = 2$). Note that while the blue line of panel (d) has larger values in the three regions between the clusters of opposing labels, the red line of panel (e) has maximum values at only the two rightmost regions between clusters. 
     
     \indent The blue and red dotted lines in Figure \ref{fig:unc-demo-kde} respectively show kernel density estimators from the selected query points of the two experiments; note that while query points selected by Dirichlet Variance concentrate around the regions where population-level uncertainty is large, the query points from Unc.\,(SM) sampling have not sampled from leftmost ``uncertainty region''. This illustrates that the regions where a classifier induces data-conditional uncertainty (e.g., smallest margin) do \textit{not} necessarily coincide with the true population-level uncertainty regions. 
     
     \indent We note that the smallest margin acquisition function has not identified the leftmost decision boundary between the large green cluster and the small gray cluster. This is simply due to the label of the initially chosen point and its influence in this setting with simple geometry. Namely, the initial point labeled by the Unc.\,(SM) acquisition function in panel (c) lies halfway in between the initially labeled points. The label of this first query point is assigned randomly according to the relative values of the class-conditional densities, which are approximately equal. With probability roughly 1/2, the resulting label will be class $y=1$ and then the subsequent margin values will focus only on the left half of the domain of the experiment. Our point is not to suggest corrections for this type of smallest margin uncertainty sampling, but rather illustrate that it is important to design acquisition functions that will properly explore the extent of the clustering structure of the dataset so that the resulting regions of data-conditional uncertainty are properly aligned with the regions of population-level uncertainty.

     \indent While this example demonstrates the explorative capabilities of using Dirichlet Variance as an acquisition function and the subsequent reflection of population-level uncertainty, we also note the potential gains of considering the proportional sampling active learning policy we have discussed (see Table \ref{table:dirvar-defs}) as opposed to always selecting the maximizer of the chosen acquisition function. For example, while the Unc.\,(SM) acquisition function is maximized in the right-hand side of the domain throughout the active learning process, proportional sampling would allow queries on the left-hand side and lead to a distribution of query points in all regions of large population-level uncertainty.
\end{example}

\begin{figure} 
    \centering
    \begin{subfigure}{0.7\textwidth}
        \hspace{4em}
        \includegraphics[width=\textwidth]{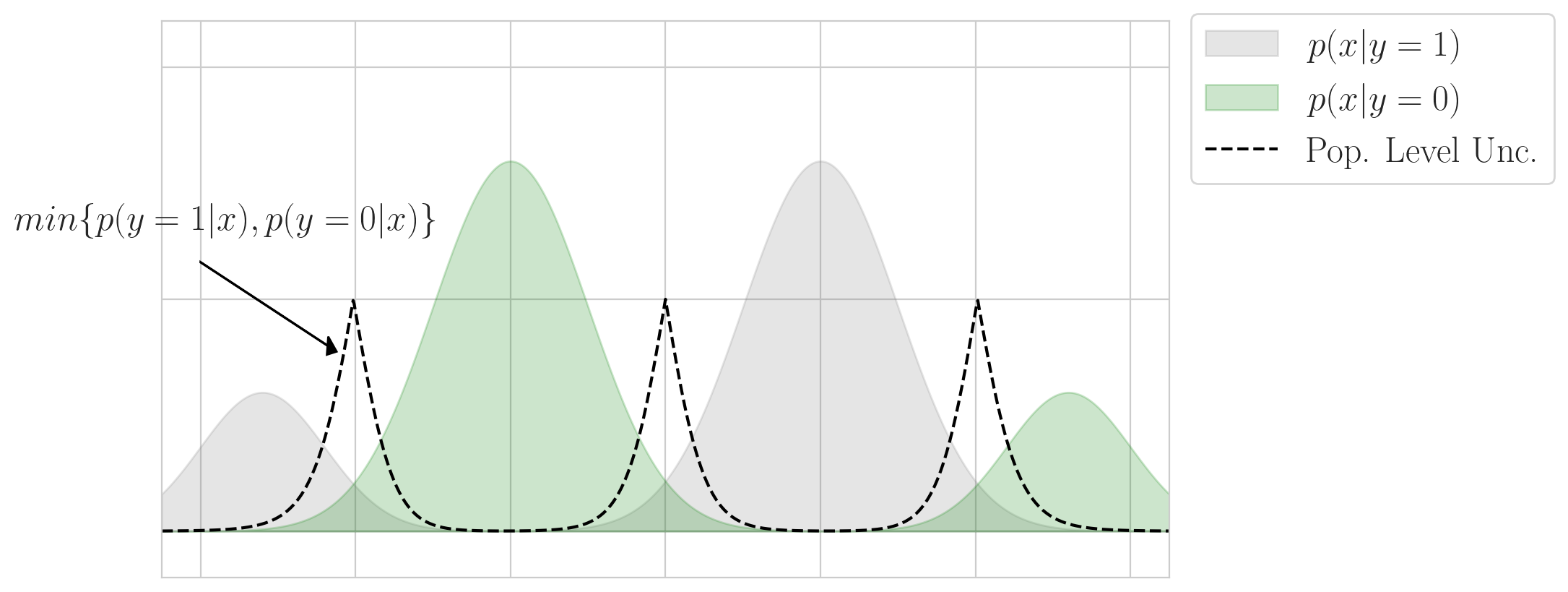}
            \caption{Setup of 2 classes}
    \end{subfigure}
    \\
    \begin{subfigure}{0.49\textwidth}
        \centering
        \includegraphics[width=\textwidth]{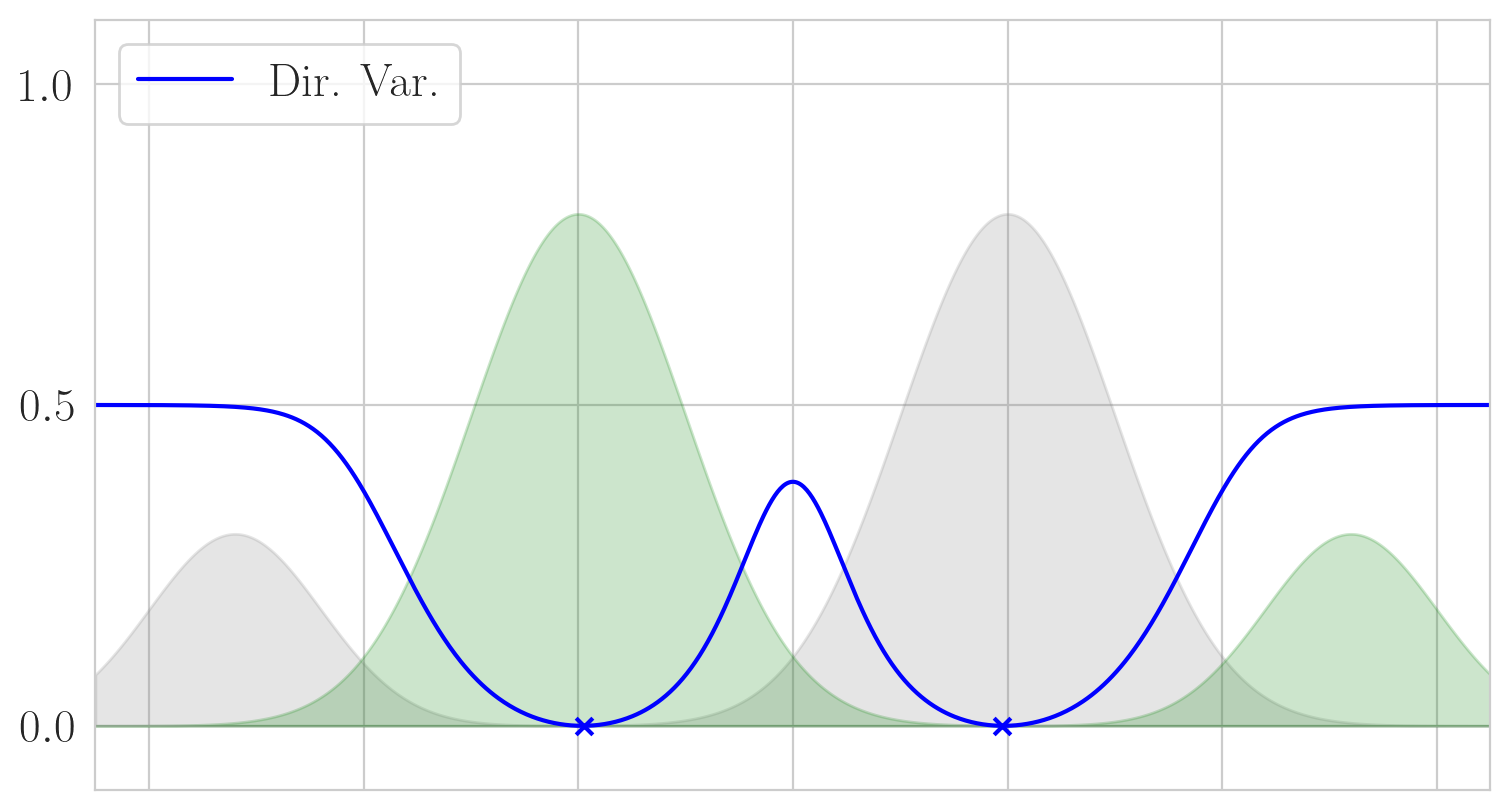}
        \caption{Dir.\,Var.\, at Iter 1}
    \end{subfigure}
    % \hspace{-3em}
    \begin{subfigure}{0.49\textwidth}
        \centering
        \includegraphics[width=\textwidth]{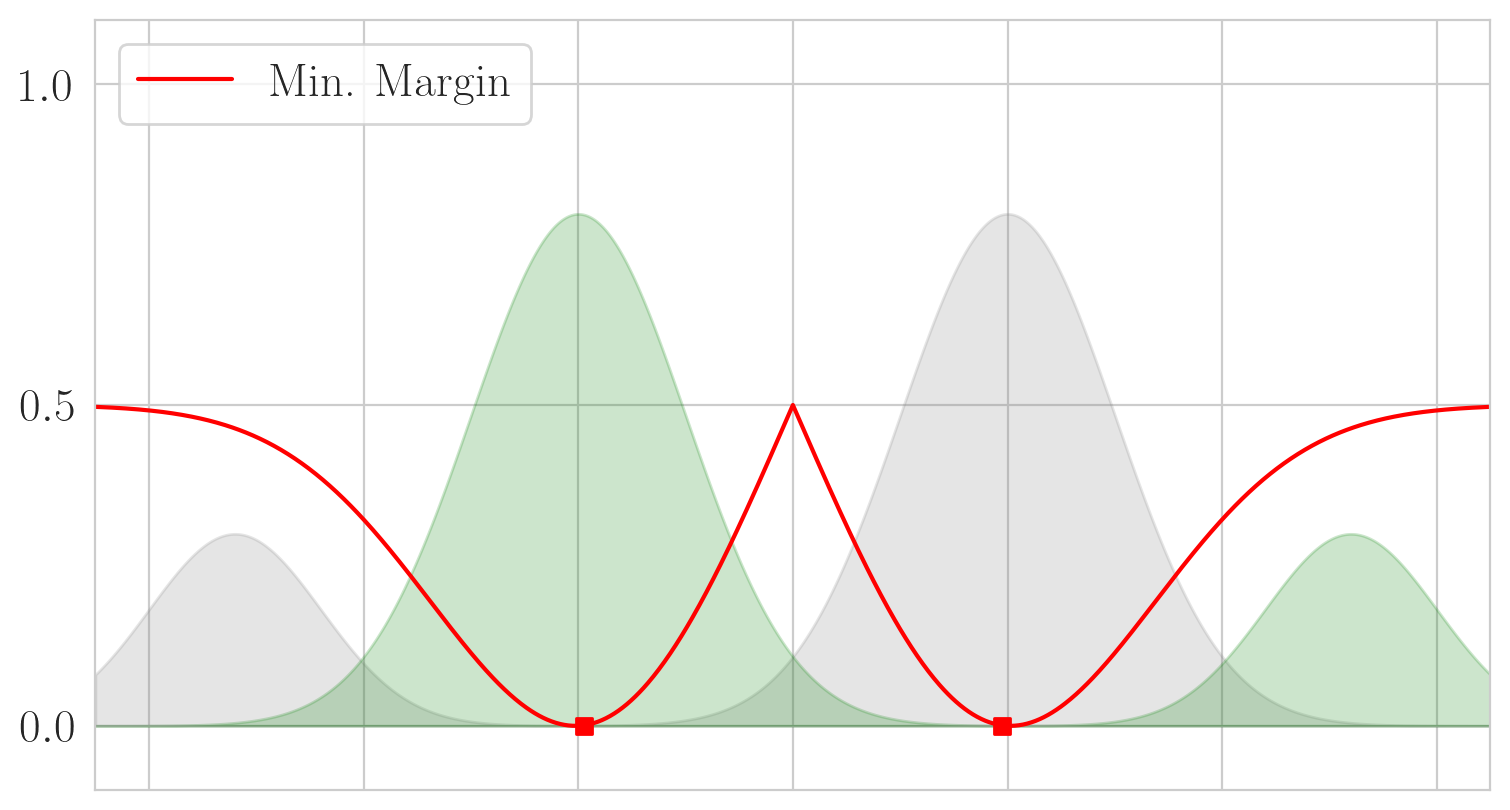}
        \caption{Unc.\,(SM) at Iter 1}
    \end{subfigure}

    % \hspace{-5em}
    \begin{subfigure}{0.49\textwidth}
        \centering
        \includegraphics[width=\textwidth]{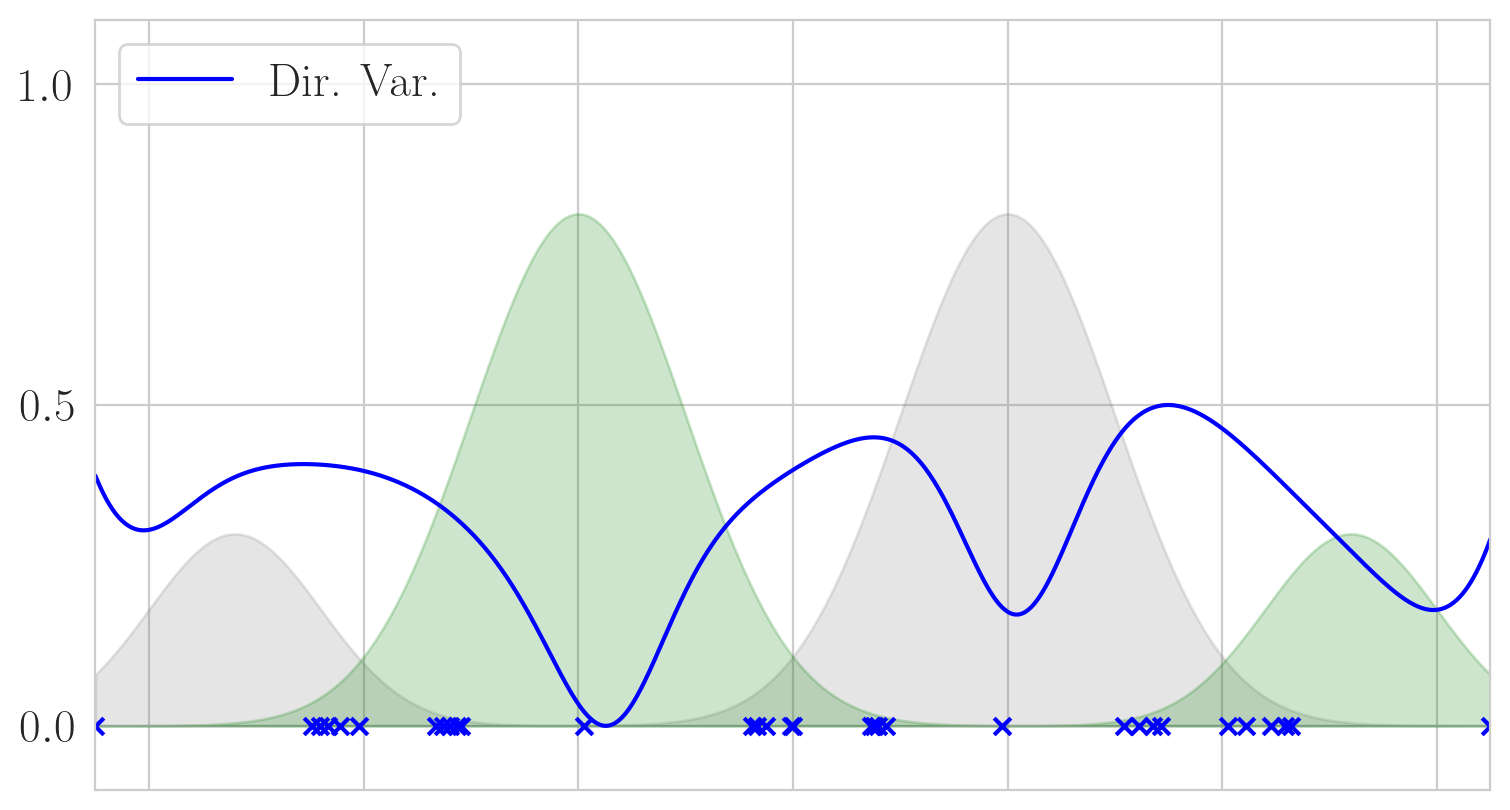}
        \caption{Dir.\,Var.\, at Iter 30}
    \end{subfigure}
    % \hspace{-3em}
    \begin{subfigure}{0.49\textwidth}
        \centering
        \includegraphics[width=\textwidth]{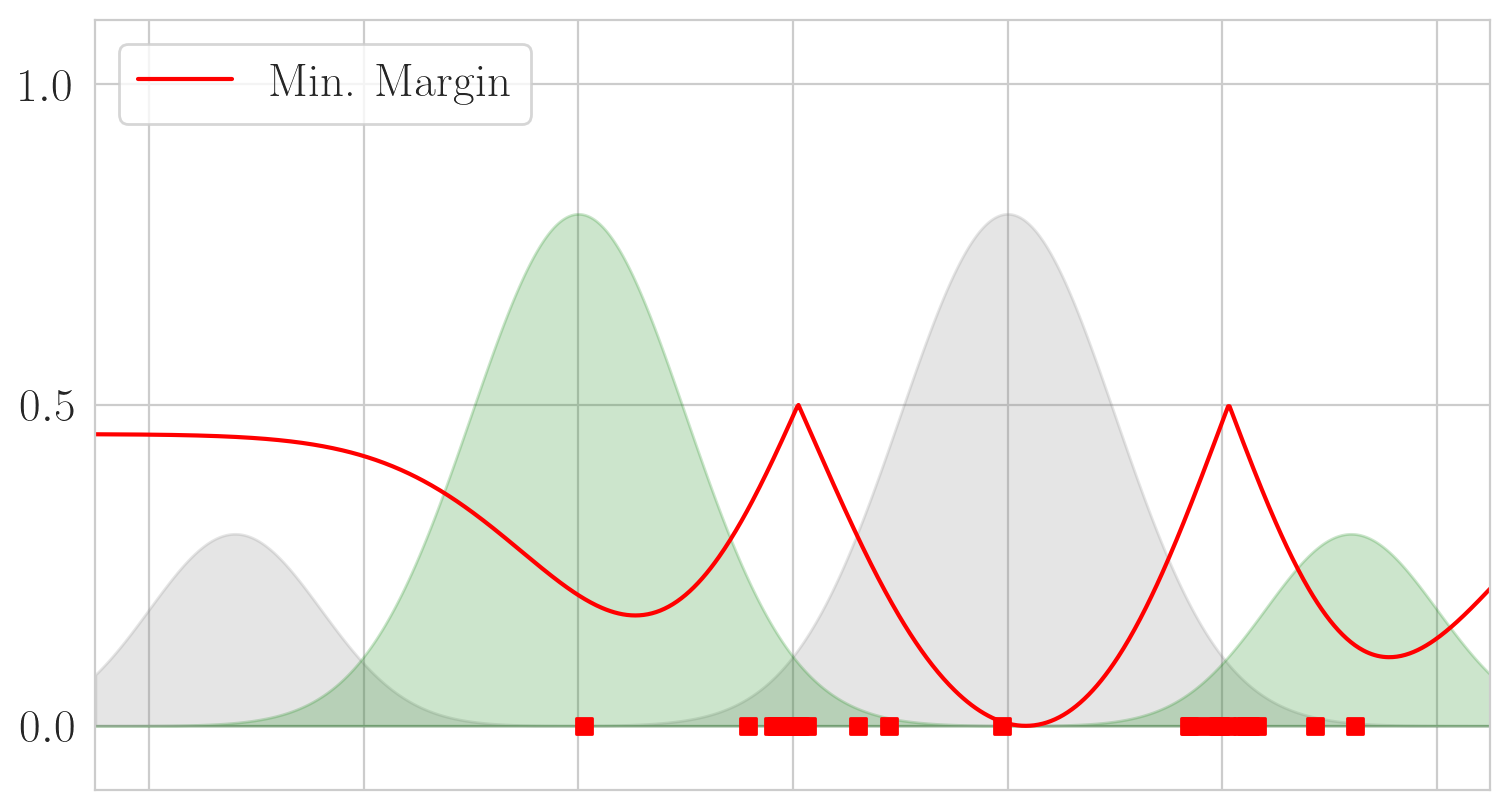}
        \caption{Unc.\,(SM) at Iter 30}
    \end{subfigure}

    % \hspace{-5em}
    \begin{subfigure}{0.49\textwidth}
        \centering
        \includegraphics[width=\textwidth]{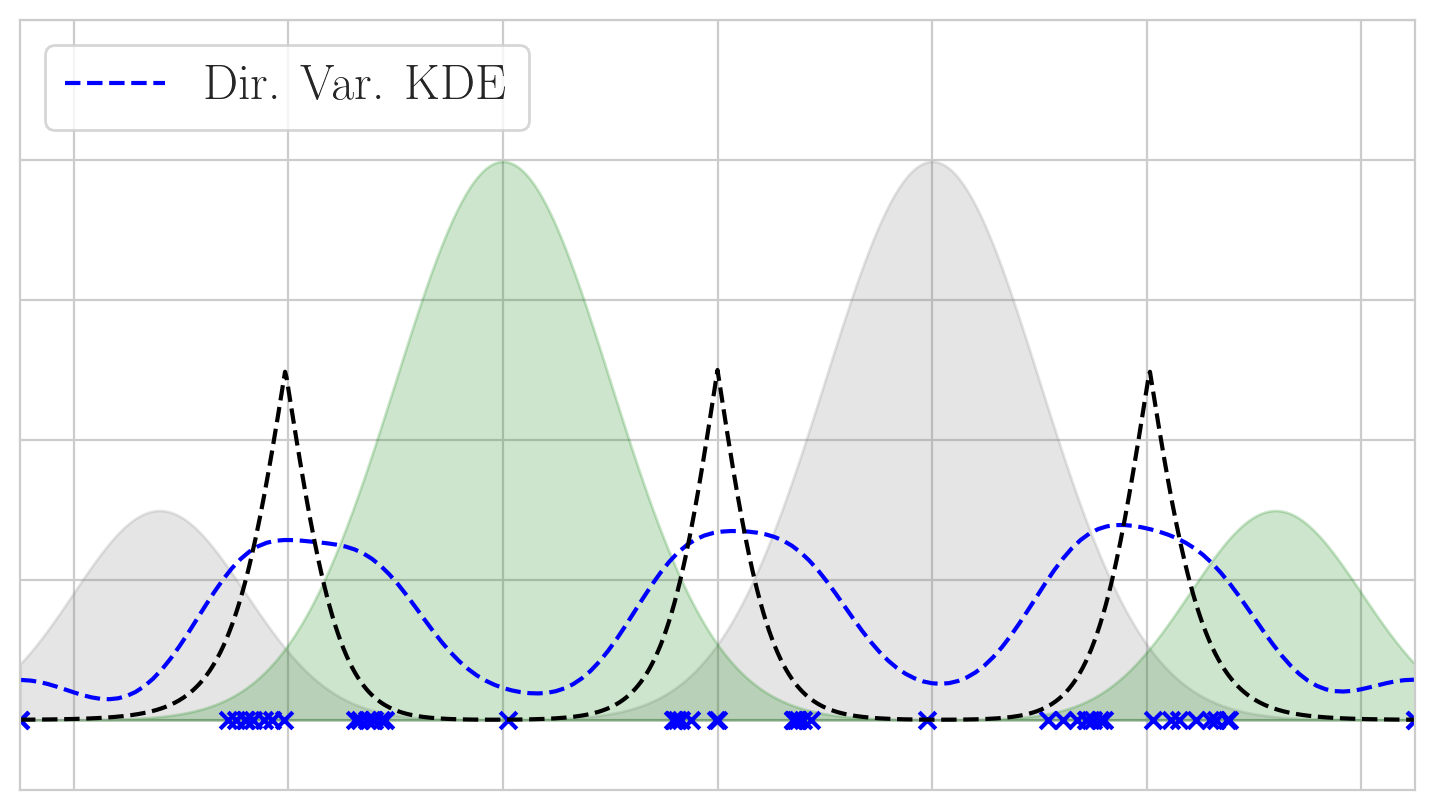}
        \caption{KDE of choices from Dir.\,Var.\,}
    \end{subfigure}
    % \hspace{-3em}
    \begin{subfigure}{0.49\textwidth}
        \centering
        \includegraphics[width=\textwidth]{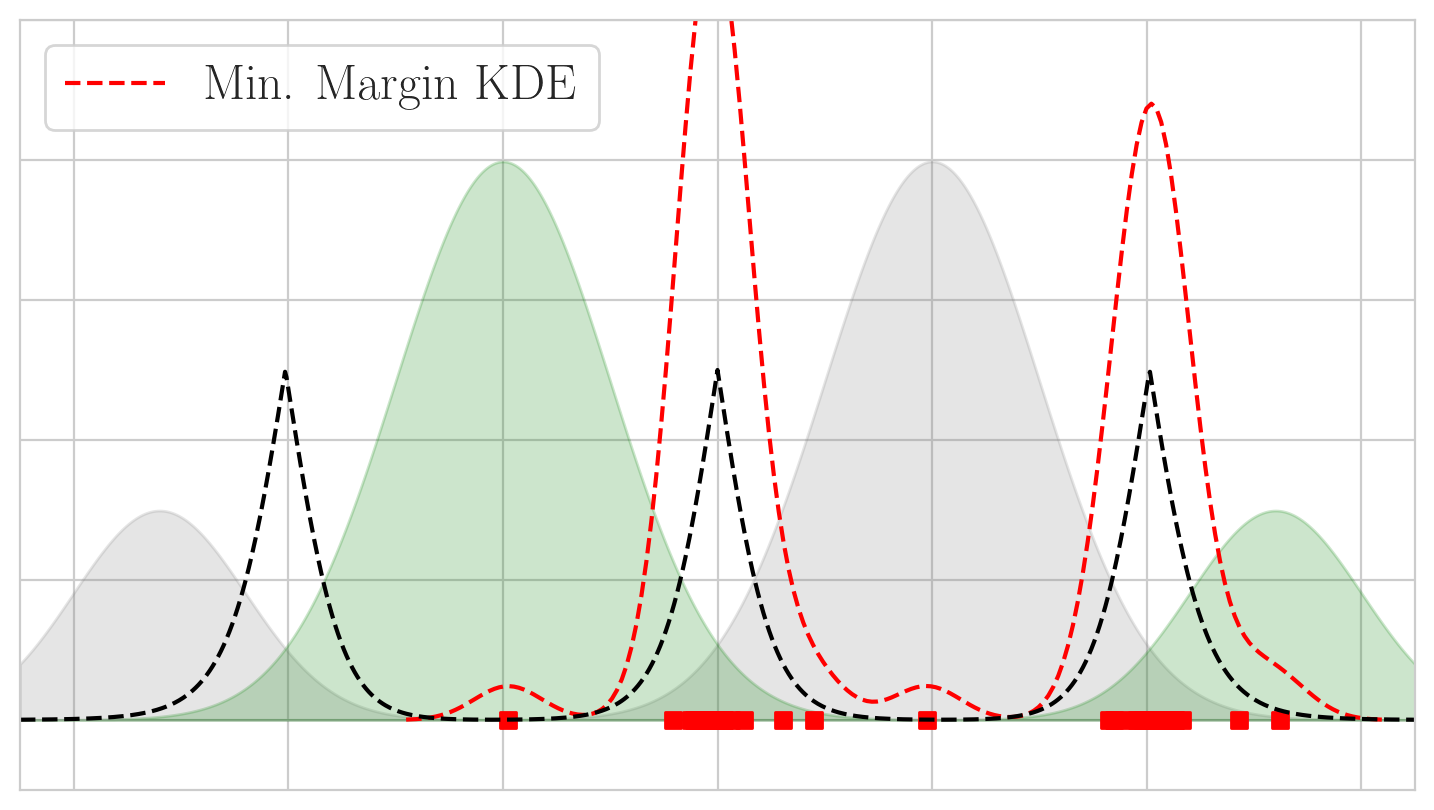}
        \caption{KDE of choices from Unc.\,(SM)}
    \end{subfigure}
    \caption{Demonstration of \textbf{population-level} and \textbf{data-conditional} uncertainties in a toy example of a mixture of Gaussians in one dimension. The dashed black line represents a measure of the population-level uncertainty for the given setup; this is maximized in regions where the class-conditional densities are equal, $p(x|y=1) = p(x|y=0)$. In panels (b)-(e), we plot data-conditional uncertainties for the Dirichlet Variance (blue) and Unc.\,(SM) (red) acquisition functions at various iterations. See Example \ref{example:2types-unc} for further experiment details.}
    \label{fig:unc-demo}
\end{figure}

\begin{figure}
    \centering
    \begin{subfigure}{0.49\textwidth}
        \centering
        \includegraphics[width=\textwidth]{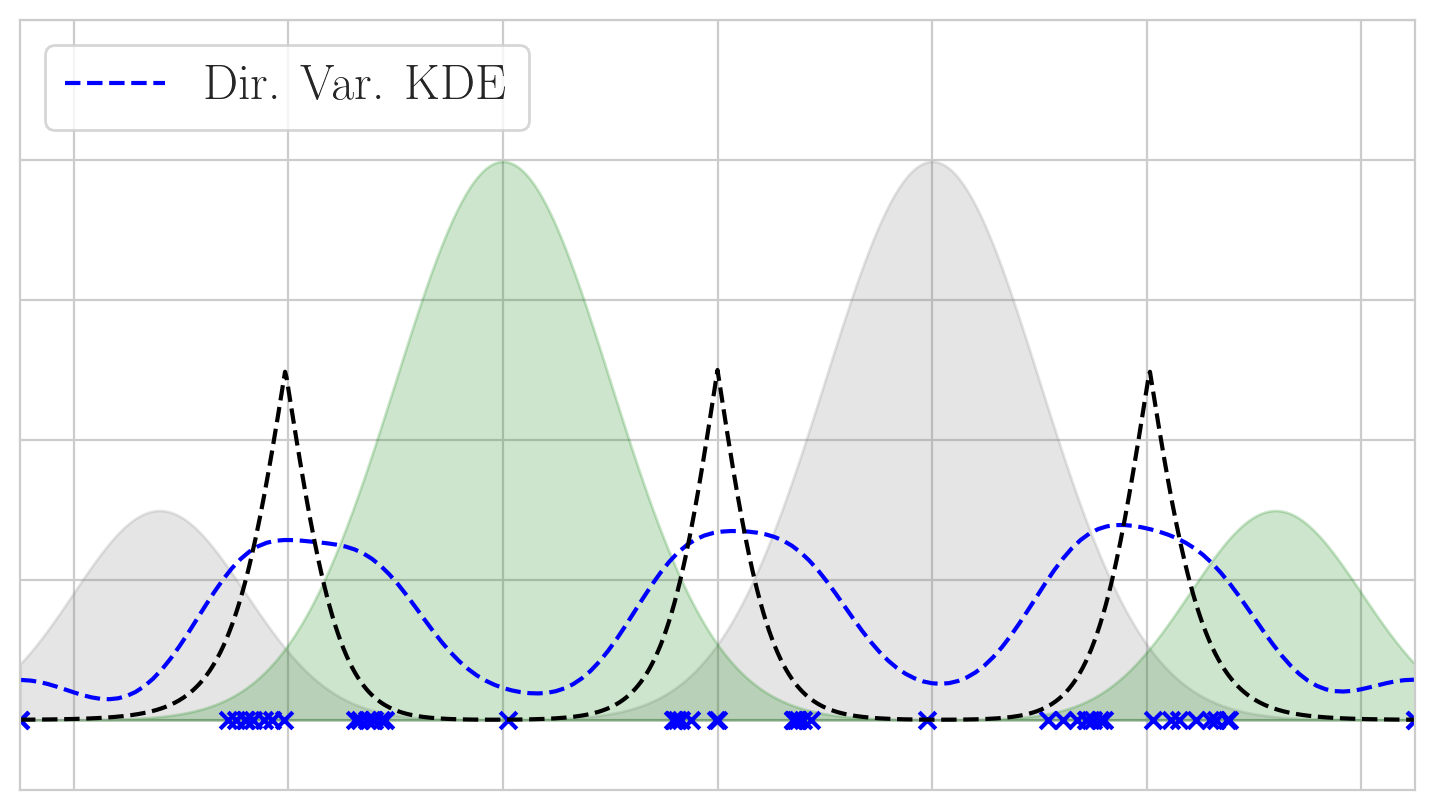}
        \caption{KDE of choices from Dir.\,Var.\,}
    \end{subfigure}
    % \hspace{-3em}
    \begin{subfigure}{0.49\textwidth}
        \centering
        \includegraphics[width=\textwidth]{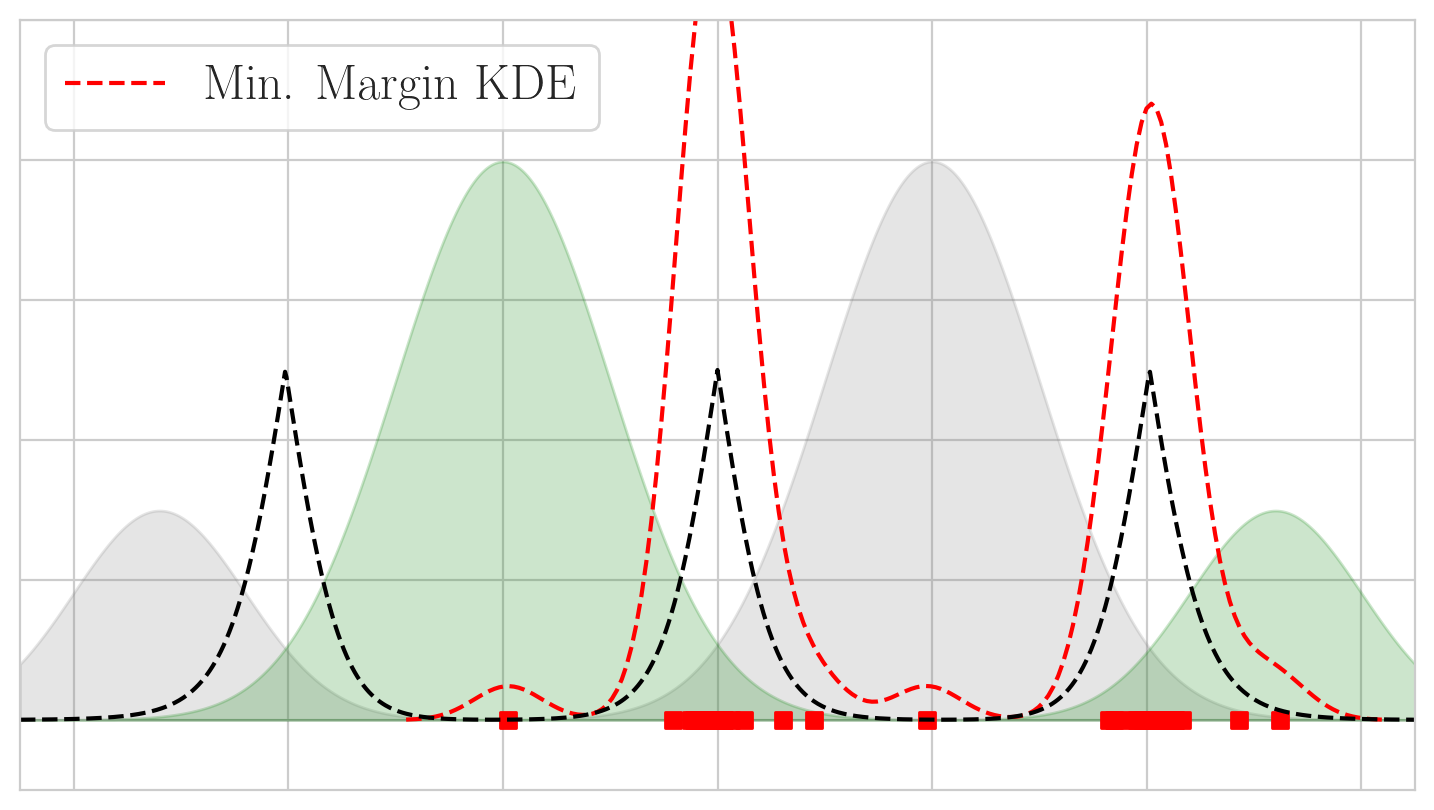}
        \caption{KDE of choices from Unc.\,(SM)}
    \end{subfigure}
    \caption{Visualization of a kernel density estimator (KDE) of the choices made from the first $50$ query points by Dir.\,Var.\, and Unc.\,(SM) from the toy example shown in Figure \ref{fig:unc-demo}. Note that the empirical distribution (reflected by the KDE) of the query points selected by Unc.\,(SM) implies that it has not sampled from the leftmost region of large population-level uncertainty due to insufficient exploration earlier on in the active learning process.}
    \label{fig:unc-demo-kde}
\end{figure}

We remark that other authors have considered variance as a selection criterion for active learning \citep{ji_variance_2012, ma_sigma_2013}. For example, both Variance Optimization (VOpt) \citep{ji_variance_2012} and $\Sigma$-Optimality ($\Sigma$Opt) \citep{ma_sigma_2013} utilize functions of the variance of Gaussian random fields with covariance structure dependent on a similarity graph as a measure of uncertainty. However, in those contexts, the observed class labels do not affect the variance of the random fields, and hence VOpt and $\Sigma$Opt serve primarily to ensure that labeled points are spread evenly across the available data: in the language of the previous remark these algorithms are primarily exploratory in nature, and in a way that is label ambivalent. In contrast, the use of Dirichlet random fields here allows the acquisition of labeled points to be informed by the labels of the points themselves.

It is useful here to recall a few properties of the variance of a Dirichlet distribution, and give their interpretation in the context of active learning.
\begin{property}[Variance decreases in total observations] \label{property:dec-var-more-data}
    Given Dirichlet distributions with parameter $t\alpha$ with $t > 0, \alpha \in \mbb R^K_+$ we have that the variance is \textit{monotonically decreasing in $t$}. When $t$ is large, the variance is of order $t^{-1}$. This indicates that, holding the proportions of class observations constant, the variance decreases as we observe more data. This means that after each new acquisition of data in our active learning algorithm our variance will decrease in expected value.
\end{property}
\begin{property}[Non-monotonicity of variance]\label{property:non-monotone}
    The variance of the Dirichlet distribution is not monotone in its individual components. This is readily seen in Figure \ref{fig:var-beta-heatmap}. This means that following a new query, it is possible that the variance of the Dirichlet distribution at some points may actually increase. Consequently, our proposed algorithm is not guaranteed to be monotone in its uncertainty, a property which is used to provide performance guarantees for some classes of acquisition functions (e.g., VOpt and $\Sigma$Opt \citep{ma_sigma_2013}).
\end{property}
\begin{property}[Reduction to uncertainty sampling] \label{property:reduce-to-unc-sampling}
    Given a particular constant $\mcl O = \sum_{k=1}^K \alpha_k$ the variance of the Dirichlet distribution is maximized when $\alpha_{k_1} = \alpha_{k_2}$ for all $k_1,k_2 \in [K]$. This implies that if the number of total pseudo-labels (i.e., implicit class observations) is comparable across different $x$'s then the uncertainty is measured to be higher at points where class probabilities are similar. Described in another way, if $\beta(x)$ for each $x$ were approximately constant during the active learning process, then the algorithm would reduce to a form of uncertainty sampling that would select query points in regions where the current classifier has large levels of ``data-conditional'' uncertainty, see Remark \ref{remark:2types-unc}.
\end{property}

\begin{figure}
    \begin{subfigure}{0.5\textwidth}
        \centering
        \vspace{-1em}
        \includegraphics[width=0.88\textwidth]{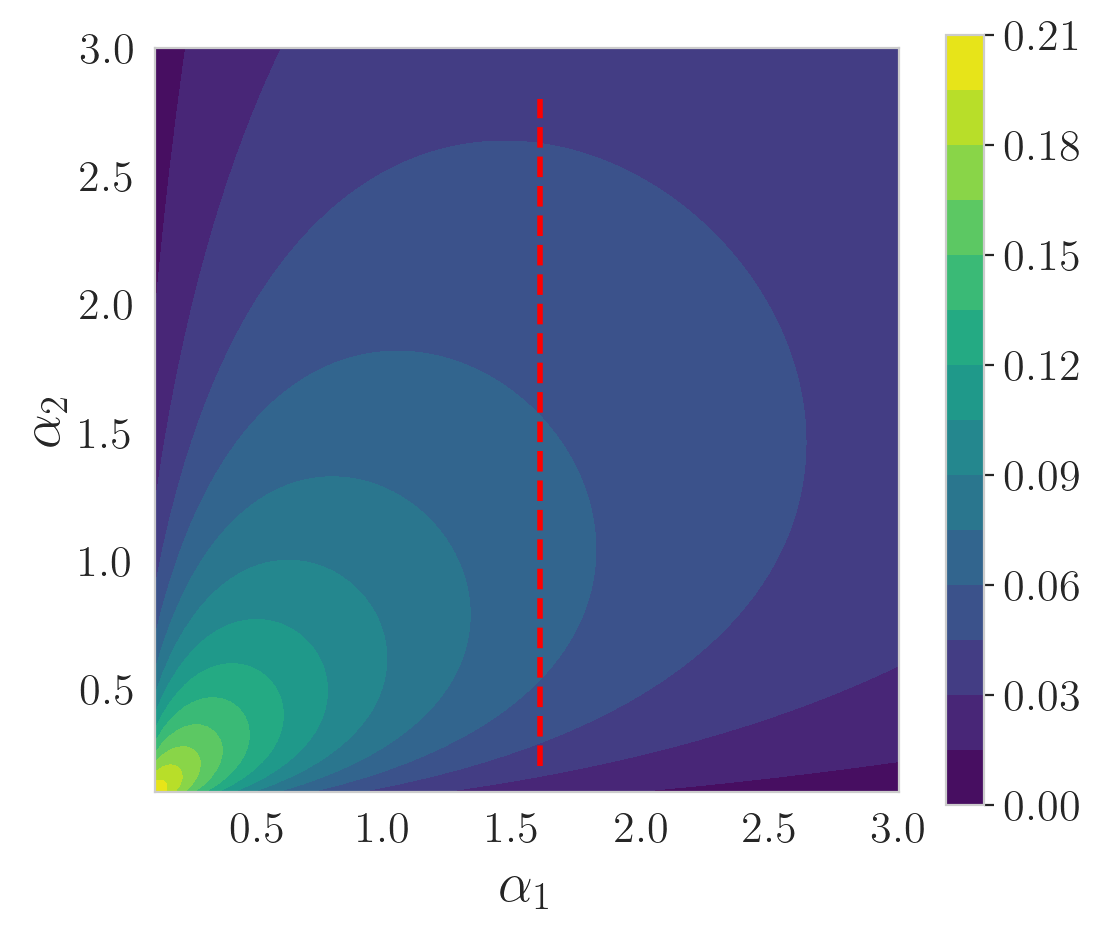}
        \caption{Variance of $Dir(\alpha_1, \alpha_2)$}
    \end{subfigure}
    \begin{subfigure}{0.5\textwidth}
        \centering
        \includegraphics[width=0.8\textwidth]{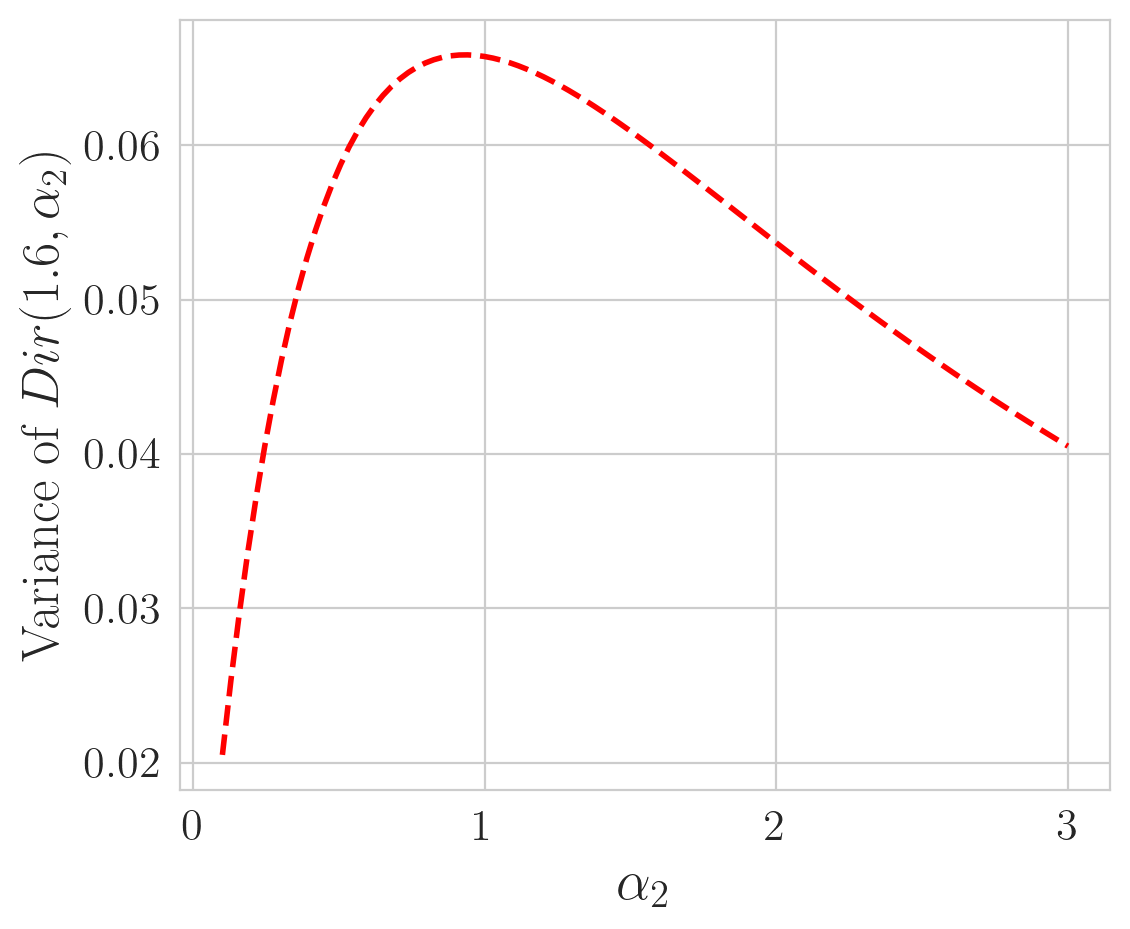}
        \vspace{0.5em}
        \caption{Variance as function of $\alpha_2$, holding $\alpha_1 = 1.6$ constant.}
    \end{subfigure}
    \caption{Demonstration of the non-monotonicity of variance of Dirichlet random variables \textit{in individual components, $\alpha_i$}. Panel (a) is a heatmap of $Var[P] = \frac{\alpha_1 \alpha_2}{(\alpha_1 + \alpha_2)^2 (\alpha_1 + \alpha_2 + 1)}$ for $P \sim Dir(\alpha_1, \alpha_2)$. The red dotted line in (a) represents a slice where $\alpha_1 = 1.6$ is held constant, and panel (b) highlights the corresponding non-monotonicity of the variance as $\alpha_2$ is varied.}
    \label{fig:var-beta-heatmap}
\end{figure}

% \km{
% \begin{table}[h!] 
% \centering
% \begin{tabular}{@{}ccc@{}}\toprule
% Term & Description\\
% \midrule 
% Dirichlet Learning & the semi-supervised classifier: \\
%  & $\hat{y}(x) = \argmax_{k=1,2,\ldots, K} \alpha_k(x) = \argmax_{k=1,2,\ldots, K} \sum_{x^\ell \in \mcl L_k} \Ker(x^\ell, x)$ \\
% Dirichlet Variance & the acquisition function: \\
% & $\mcl A_{var}(x) = Tr[C(x)]$ \eqref{eq:dirichlet-variance-af} \\
% Dir.\,Var.\, & query selection using Dirichlet Variance + MV policy: \\
% & $x^\ast = \argmax_{x \in \mcl U} \ \mcl A_{var}(x)$ \\
% Dir.\,Var.\,(Prop), & query selection using Dirichlet Variance + PS policy: \\
% & $x \in \mcl U,\ q(x) \sim \exp \lp \lambda \mcl A_{var}(x)\rp$ \\
% Dirichlet Active Learning & the overall active learning process \\
% \bottomrule
% \end{tabular}
% \caption{}
% \label{table:defs}
% \end{table}
% }

\section{Numerical Experiments} \label{sec:results}

% box of summary of Results section, to really highlight.
% sketch of diagram of uncertainties? Ryan to inkscape

% \todo{Need to add brief discussion of choice of $\lambda$ in proportional sampling.}

In this section, we present numerical results to evaluate the utility of our proposed DiAL framework utilizing a graph-based Dirichlet Learning model and the straightforward use of Dirichlet Variance as an acquisition function for active learning. In Subsection \ref{subsec:hsi-results}, we begin with an application to pixel classification in hyperspectral imagery (HSI) in order to compare with the Learning by Active Non-linear Diffusion (LAND) algorithm introduced in \citep{murphy_unsupervised_2019}. The experiments in Subsection \ref{subsec:exploration-results} use a setup involving the common machine learning benchmark datasets MNIST \citep{lecun-mnisthandwrittendigit-2010} and FASHIONMNIST \citep{xiao2017fashionmnist} to assess the explorative capabilities of active learning methods; this setup was introduced recently in \citep{miller2023unc}. We calculate the average semi-supervised classification accuracy as a function of the size of the labeled set over 10 trials in each experiment to compare our proposed method with various acquisition functions that have previously been proposed in graph-based active learning. In the experiments of Subsection \ref{subsec:exploration-results}, we also track the proportion of clusters that each acquisition function has sampled from as a function of the labeled set size in order to better understand the explorative capabilities of the compared methods. In Table \ref{table:acq_funcs}, we list the various acquisition functions that we use in our experiments to compare with our proposed acquisition functions.

\begin{table}[h!] 
\centering
\begin{tabular}{@{}cccccc@{}}\toprule
Abbr. Name & Name & Policy & Ref.\\
\midrule 
Dir.\,Var.\, & Dirichlet Variance & MV & (present work) \\
Dir.\,Var.\,(Prop) & Proportional Dirichlet Variance & PS & (present work) \\
LAND & Learning by Active Nonlinear Diffusion & MV & (Murphy et al.\,, 2019) \\ %\citep{murphy_unsupervised_2019} \\
Unc.\,(SM) & Smallest Margin Uncertainty Sampling & MV & \citep{settles_active_2012} \\
VOpt & Variance Optimization & MV & \citep{ji_variance_2012} \\
$\Sigma$Opt & $\Sigma$-Optimality & MV & \citep{ma_sigma_2013} \\
MCVOpt & Model Change with VOpt heuristic & MV & \citep{miller_spie_2022} \\
Rand.\, & Random Selection & N/A  & N/A \\
\bottomrule
\end{tabular}
\caption{Acquisition functions of interest in the experiments in Subsections \ref{subsec:hsi-results} and \ref{subsec:exploration-results}. See Section \ref{sec:query-point-selection} for a detailed explanation of MV (maximum value) and PS (proportional sampling) policies. }
\label{table:acq_funcs}
\end{table}

% \km{ \rwm{Perhaps this is better summarized in a table? Listing unabbreviated names next to abbreviated ones may look better too. A table could have different columns (or maybe bolded text) for acquisition function vs policy.}
% The acquisition functions we report here are:
% \begin{itemize}
%     \item Dir. Var. : Select query point via MV active learning policy with Dirchlet Learning variance \eqref{eq:var}.
%     \item Dir. Var. Prop : Select query point via SP active learning policy with Dirichlet Learning variance \eqref{eq:var}.
%     \item Random : Select query point uniformly at random from current unlabeled data.
%     \item LAND : Select query point from mode estimation with diversity encouraged by diffusion distances, \citep{murphy_unsupervised_2019}.
%     \item Unc.\,(SM) : Select query point via Smallest Margin Uncertainty Sampling \citep{settles_active_2012} in the Laplace Learning semi-supervised model \citep{zhu_combining_2003}.
%     \item MCVOpt : Select query point that maximizes the Model Change VOpt acquisition function \citep{miller_spie_2022}.\rwm{Maybe it's good to emphasize that the ``variance'' here means something different.}
%     \item VOpt : Select query point that maximizes Variance Optimization \citep{ji_variance_2012}.
%     \item $\Sigma$Opt : Select query point that maximizes $\Sigma$-Optimality \citep{ma_sigma_2013}.
% \end{itemize}
% }

For all experiments, we compute the corresponding accuracy of our Dirichlet Learning model on the unlabeled data at each iteration of the active learning process. For each experiment (dataset), we have simply used a value of $\tau = 0.1$, which was chosen via trial-and-error. We note that an interesting line of inquiry is to investigate how to properly choose this parameter value; we leave this for future work. 
%and we hypothesize that the proper choice would relate the clustering structure of the dataset to the expected value of 

We highlight two %\km{(three?)} 
main takeaways from our experiments: 

\fbox{ \label{box:takeaways} % 
    \parbox{0.9\textwidth}{%
    \begin{minipage}{0.85\textwidth}
        \begin{enumerate}
            \item The Dirichlet Variance acquisition functions \ref{table:dirvar-defs} select query points that empirically improve classifier accuracy comparable to the prior state of the art (LAND, $\Sigma$Opt, VOpt, MCVOpt) with a much lower computational cost and much greater interpretability.
            \item The Dir.\,Var.\,(Prop) acquisition function performs most favorably for exploration and overall performance across the range of experiments compared to the array of acquisition functions using Maximum Value policies.  
            We provide some initial steps in understanding the advantages of proportional sampling in Section \ref{sec:theory}, but further work in this vein is warranted.
            % We suggest that future work investigating the theoretical advantages of proportional sampling (PS) as an active learning policy is warranted. 
        \end{enumerate} 
    \end{minipage}
    }%  
}

\vskip1em
We first present the experimental setups and results in Sections \ref{subsec:hsi-results} and \ref{subsec:exploration-results}. Then in Section \ref{sec:results-discussion} we will address these takeaways in more detail with the context of the results from the experiments.

\subsubsection{Choices of hyperparameters $\alpha_0$ and $\lambda$} \label{subsec:lambda-choice}

We briefly comment on the choice of the uniform prior constant $\alpha_0 \geq 0$ and the inverse temperature parameter $\lambda > 0$ (for the Proportional Sampling active learning policy, see Table~\ref{table:dirvar-defs}) in our experimental results. While the derivation of ``optimal'' values for either of these hyperparameters would be an interesting line of inquiry with practical importance, it lies outside the scope of this current work.

We use a simple heuristic for the selection of both parameters that focuses on locality in the clustering structure. Namely, the choice of $\alpha_0$ ensures that it is on the order of the expected value of the Poisson graph-based propagation $\Ker_P(z,x)$ for nodes $z,x \in \mcl X$ that are relatively ``local'' to one another in the graph. This expectation is roughly approximated via a small random sample of source nodes and the notion of locality comes from the $100(\hat{K}-1)/\hat{K}$th percentile of propagation values, where $\hat{K}$ is an overestimate of the number of clusters $K$ contained in the dataset.

Similarly, the choice of $\lambda$ emphasizes 
the $100(\hat{K}-1)/\hat{K}$th percentile of acquisition function values at each iteration. This has the effect of biasing the sampling toward the $1/\hat{K}$ fraction of largest values in order to ``focus the sampling'' on regions where the acquisition function value is larger. While one doesn't necessarily have access to the true value of $K$ in practice, we suggest that \textit{overestimating} the number of clusters in the dataset is a reasonable idea since the limiting case of $\lambda \rightarrow \infty$ corresponds to the Maximum Value policy--which already performs quite well in our experiments here. In our experiments, we take $\hat{K} = 2K$ as an overestimate of the number of clusters.

\begin{comment}
def set_eps(self, K=10, verbose=False):
    if K is None:
        self.eps = 1.0
    else:
        # epsilon calculation
        MULT = 5
        num_eps_props = K * MULT
        rand_inds = np.random.choice(self.graph.num_nodes, num_eps_props, replace=False)
        props = self.poisson_prop(rand_inds)
        props_to_inds = np.max(props, axis=1)
        epsilons = np.array([np.percentile(props[:,i], 100.*(K-1.)/K) for i in range(props.shape[1])])
        self.eps = np.max(epsilons)
        
    if verbose:
        print(f"\tSetting eps = {self.eps} for Beta Learning, tau = {self.tau}")
    return 
\end{comment}

% \rwm{I can imagine reviewers wanting more information here: we should spell out the computational cost for our algorithm a little more explicitly. The cost of including a new data point is essentially a single backslash operation? How much cheaper is this than VOpt etc.?} \km{To be explicit, the cost of updating the Dir.\, Var.\, acquisition functions values upon inclusion of a recently labeled point requires only a sparse\footnote{We use $k$-nearest neighbor graphs which yield sparse linear systems to be solved.} linear solve, while the VOpt and $\Sigma$Opt acquisition functions require updating a dense covariance matrix and computing the corresponding column norms.}
% \todo{ need to explain more clearly} Therefore, our proposed method for graph-based active learning constitutes a computationally efficient acquisition function for an efficient graph-based semi-supervised learning model. 

\subsection{HSI Results} \label{subsec:hsi-results}

We first demonstrate the effectiveness of using DiAL for improving pixel classification in two commonly studied hyperspectral imagery (HSI) datasets, Salinas-A and Pavia, mimicking the setup presented in \citep{murphy_unsupervised_2019, cloninger_cautious_2021}. The goal is to classify the pixels in the image into material classes based on the samples from the different wavelengths.
The Salinas A dataset is a common HSI dataset that contains 7,138 total pixels in an $83 \times 86$ image in $d = 224$ wavelengths. This is an image of Salinas, USA taken with the Aviris sensor and contains $6$ classes of plant types arranged in a diagonal pattern (see Figure \ref{fig:salinas-gt}). 
The Pavia dataset we use here consists of a $270 \times 50$ subset of the original Pavia dataset and has spatial resolution 1.3m/pixel. The image contains 6 spatial classes, and was taken over Pavia,
Italy by the ROSIS sensor (Figure \ref{fig:pavia-gt}). Both of these hyperspectral datasets are available online at \url{http://www.ehu.eus/ccwintco/index.php/Hyperspectral_Remote_Sensing_Scenes}.

\begin{figure}
\begin{subfigure}{0.3\textwidth}
    \centering
    \includegraphics[width=0.8\textwidth]{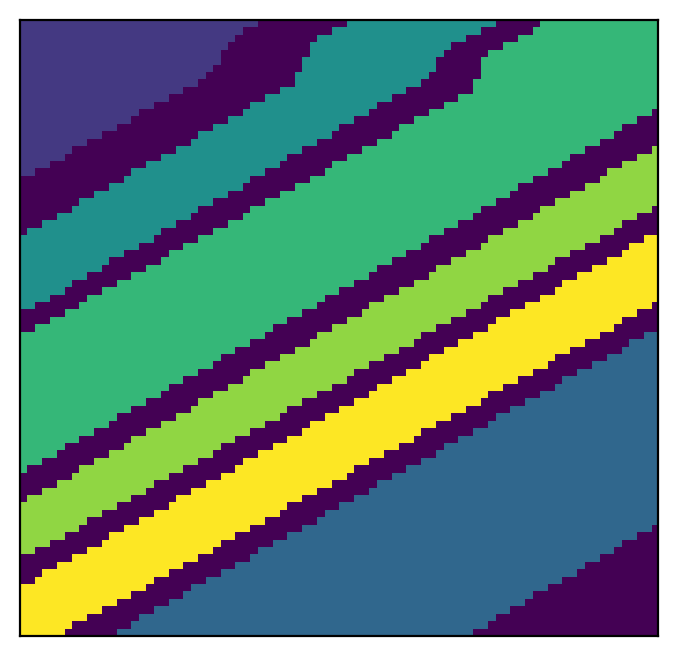}
    \caption{Salinas A}
    \label{fig:salinas-gt}
\end{subfigure}
\hspace{-4em}
\begin{subfigure}{0.9\textwidth}
    \centering
    \includegraphics[width=0.8\textwidth]{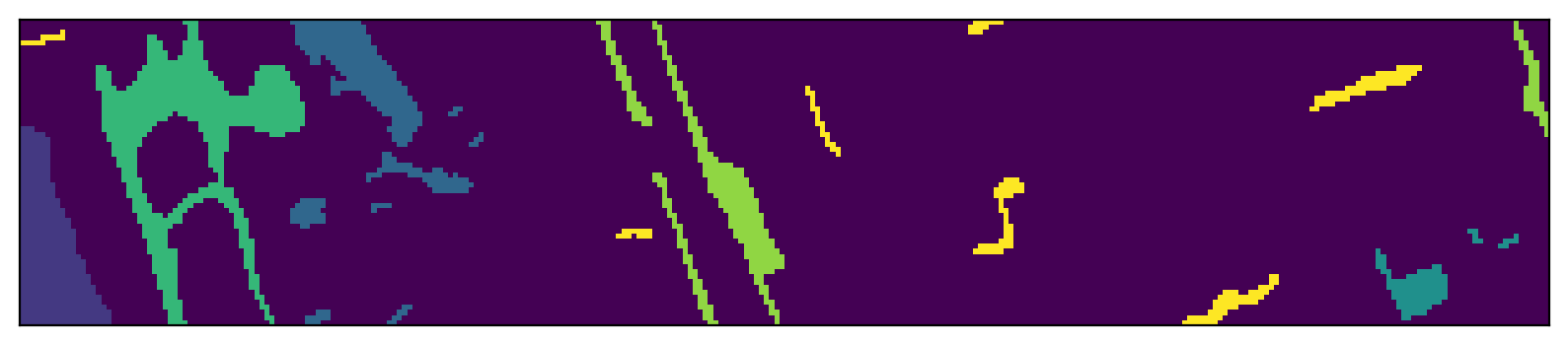}
    \caption{Pavia}
    \label{fig:pavia-gt}
\end{subfigure}
\caption{Ground truth classifications of Salinas A (a) and Pavia (b) HSI datasets.}
\end{figure}

\begin{figure}
    \begin{subfigure}{0.5\textwidth}
        \centering
        \includegraphics[width=\textwidth]{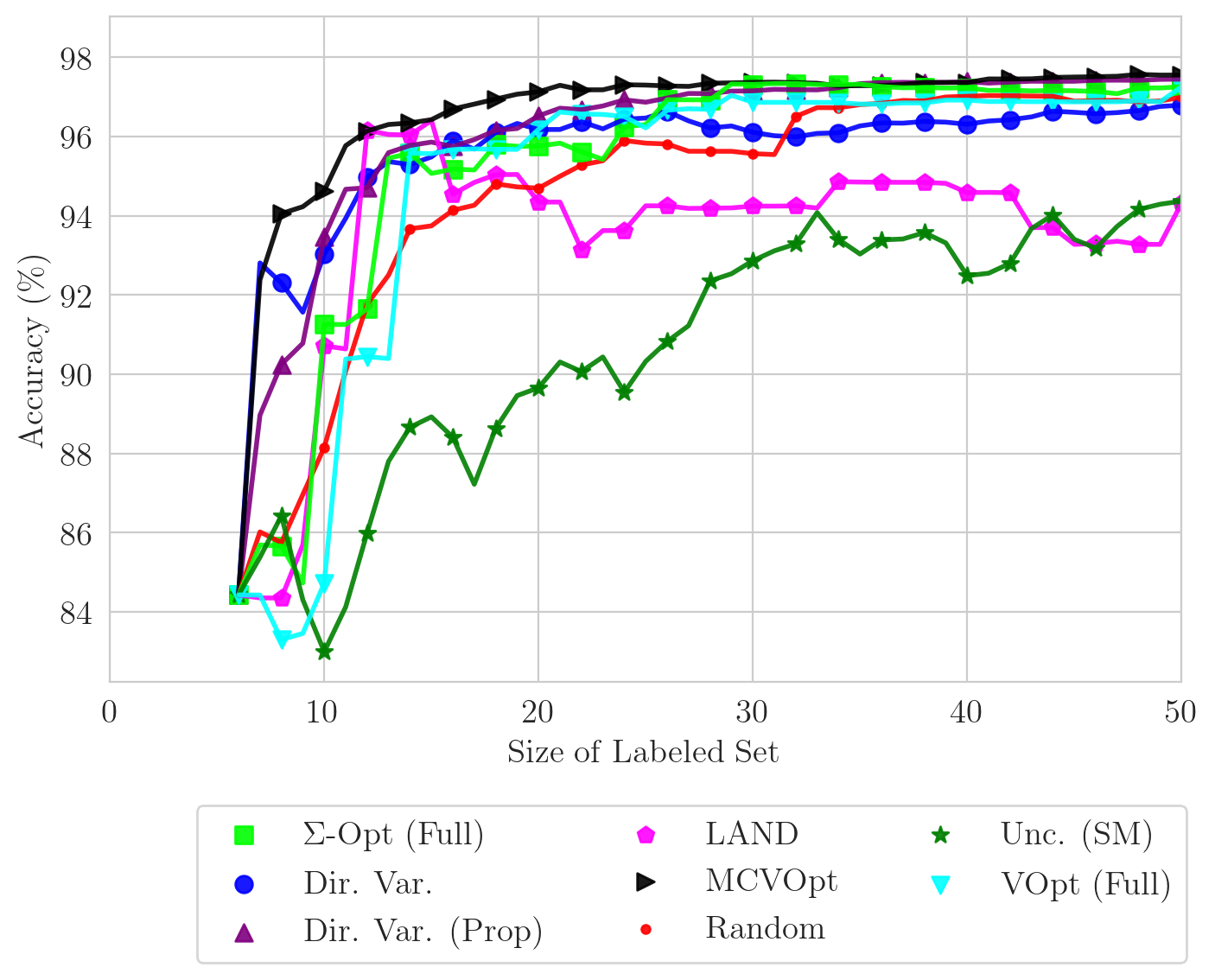}
        \caption{Salinas}
        \label{fig:salinas-metrics}
    \end{subfigure}
    \begin{subfigure}{0.5\textwidth}
        \centering
        \includegraphics[width=\textwidth]{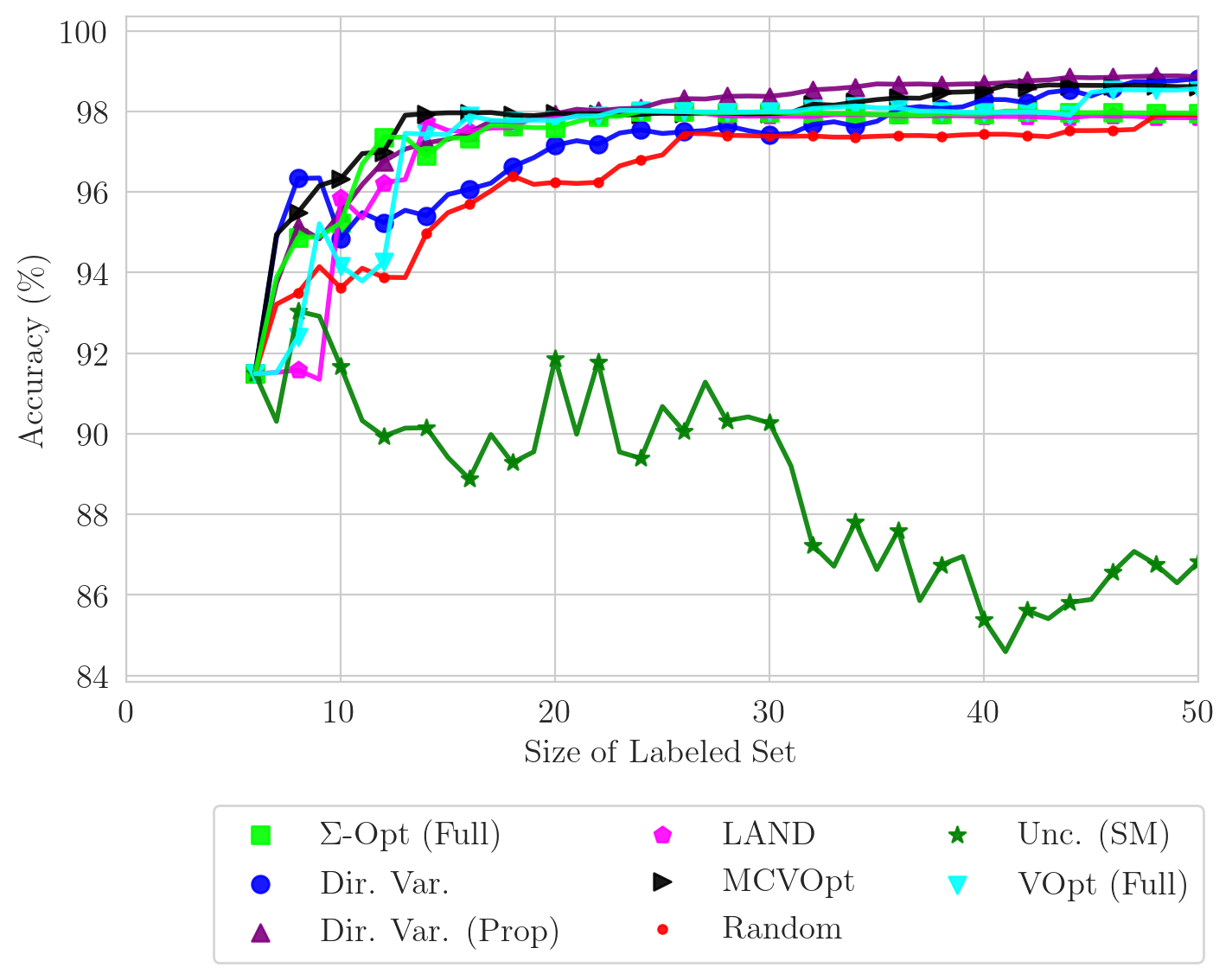}
        \caption{Pavia}
        \label{fig:pavia-metrics}
    \end{subfigure}
    \caption{Plots of accuracy results on HSI datasets, Salinas-A and Pavia datasets.}
    \label{fig:hsi-plots}
\end{figure}

We construct $k$-nearest neighbor graphs with $k=20$ using cosine similarity 
\[
    \Ker(x^i,x^j) = \langle x^i, x^j \rangle/\|x^i\|_2 \|x^j\|_2,
\]
a common similarity metric\footnote{It should be noted that this similarity metric is \textit{different} than the kernel used for propagation operator in DiAL. This similarity kernel is used for comparing input pixels for graph construction.} for HSI applications. The semi-supervised classification task is to infer the classification of unlabeled hyperspectral pixels in the image into one of a predetermined number of classes given a subset of labeled pixels. Beginning with one initially labeled pixel per class, we select pixels sequentially via the acquisition functions described in the previous section. 
% We evaluate the corresponding accuracy at each iteration in the Dirichlet Learning model, with the exception of LAND \citep{murphy_unsupervised_2019}, which we report in both Dirichlet Learning as well as in the original SSL model by the same name, LAND (see Figure \ref{fig:hsi-plots}). That is, the dark green curves in Figures \ref{fig:salinas-metrics} and \ref{fig:pavia-metrics} represent the quality metrics of the active learning choices made by LAND \textit{evaluated in the Dirichlet Learning SSL model} while the magenta curves are the quality metrics \textit{evaluated in the LAND SSL model of \citep{murphy_unsupervised_2019}}. 

In both Figure \ref{fig:hsi-plots} (a) and (b), we plot the accuracy of our Dirichlet Learning graph-based semi-supervised classifier \ref{eq:inference} on the unlabeled datapoints as labeled sets are chosen via the different acquisition functions throughout the active learning process.

We briefly comment on the relatively poor performance of smallest-margin uncertainty sampling in the Laplace learning classifier (i.e., Unc.\,(SM)) in our comparison. Smallest margin uncertainty sampling has been known to produce especially poor results due to the selection of overly-exploitative query points; that is, this commonly used method for uncertainty sampling can lead to query points that do not properly explore the extent of the dataset and result in poor empirical performance. 

We note that the LAND acquisition function is not originally designed for this underlying semi-supervised learning model. For a consistent comparison of results, we have reported the accuracy in our Dirichlet Learning semi-supervised learning model, though we do note that the corresponding accuracy in the Learning by Active Nonlinear Diffusion (LAND) \citep{murphy_unsupervised_2019} was poorer than the results in our model shown in Figure \ref{fig:hsi-plots}.

\subsection{Exploration Experiments}  \label{subsec:exploration-results}

In \citep{miller2023unc}, the authors introduce an experimental setup designed to evaluate how effectively an active learning method selects query points that both (1) explores the clustering structure of the dataset and (2) leads to optimal increases in the accuracy of the underlying semi-supervised classifier. The ground truth classes in the MNIST \citep{lecun-mnisthandwrittendigit-2010} and FASHIONMNIST \citep{xiao2017fashionmnist} benchmark machine learning datasets are used to define ``clusters'' in an auxiliary problem wherein a modified classification structure is imposed; namely, the 
% original classes (e.g. digits) are mapped to the equivalence classes of $\mbb Z_3$ according to their corresponding class numbers. 
% That is, we take the 
true class labelings $y_i \in \{0, 1, \ldots, K\}$ (e.g., digits 0-9 for MNIST) and reassign them to one of $k < K$ classes by taking $y_i^{new} \equiv y_i \operatorname{mod} k$; see Table \ref{table:mod-classes} below. 

\newcommand{\ra}[1]{\renewcommand{\arraystretch}{#1}}
\begin{table}[h!] 
\centering
\ra{1.3}
\begin{tabular}{@{}cccccc@{}}\toprule
Resulting Mod Class & 0 & 1 & 2  \\
\midrule 
MNIST & 0,3,6,9 & 1,4,7 & 2,5,8\\
FASHIONMNIST & 0,3,6,9 & 1,4,7 & 2,5,8\\
\bottomrule
\end{tabular}
\caption{Mapping of ground truth class label to $\operatorname{mod} k$ labeling for experiments. Each ground truth class, is interpreted as a different ``cluster'' and the resulting class structure for the experiments have multiple clusters per class. For MNIST and FASHIONMNIST, there 10 total ground truth classes and we take labels modulo $k=3$.}
\label{table:mod-classes}
\end{table}

%%% From PWLL paper
For each trial with an acquisition function, we select one initially labeled point per \textit{``modulo''} class; therefore, only a subset of clusters (i.e., the original true classes) has an initially labeled point. In order to perform active learning successfully in these experiments, query points chosen by the acquisition function over the trial must sample from each cluster. In this way, we have created an experimental setup with high-dimensional datasets with potentially more complicated clustering structures.
We perform 10 trials for each acquisition function, where each trial begins with a different initially labeled set, and 100 query points are chosen sequentially by each of the different acquisition functions. To clarify, trials begin with only 3 labeled points in the MNIST and FASHIONMNIST experiments, so only 3 out of the 10 clusters begin with labeled points.

With this experimental setup, we can track not only the accuracy of the underlying semi-supervised classifier on the unlabeled data (accuracy plots), but also the proportion of clusters that contain labeled data (cluster exploration plots) throughout the active learning process. While most work in active learning has only analyzed accuracy (and similar metrics) as a function of labeled set size to evaluate the performance of acquisition functions, we suggest that other metrics such as these cluster exploration plots can be informative.

Additionally, we run active learning experiments on a predetermined, random subset (10\% of the total) of each of these ``modified'' datasets. This reduced setting allows us to compare with the LAND algorithm with the publicly-available MATLAB implementation the authors use \citep{landcode2021} since we ran into memory issues when applying to the full MNIST and FASHIONMNIST datasets. Furthermore, the original forms of the VOpt and $\Sigma$Opt acquisition functions are prohibitively expensive for the full MNIST and FASHIONMNIST datasets; these reduced datasets allow us to also compare against the full calculation of VOpt and $\Sigma$Opt. 
The random subset is chosen by sampling uniformly at random 10\% of the points from each original class, and we refer to these smaller datasets as MNIST-SMALL and FASHIONMNIST-SMALL. Figures \ref{fig:mnistsmall-acc} and \ref{fig:fashionmnistsmall-acc} display the accuracy results while Figures \ref{fig:mnistsmall-clusterexp} and \ref{fig:fashionmnistsmall-clusterexp} display the cluster exploration results for the MNIST-SMALL and FASHIONMNIST-SMALL experiments, respectively.

\begin{figure}
    \begin{subfigure}{0.5\textwidth}
        \centering
        \includegraphics[width=\textwidth]{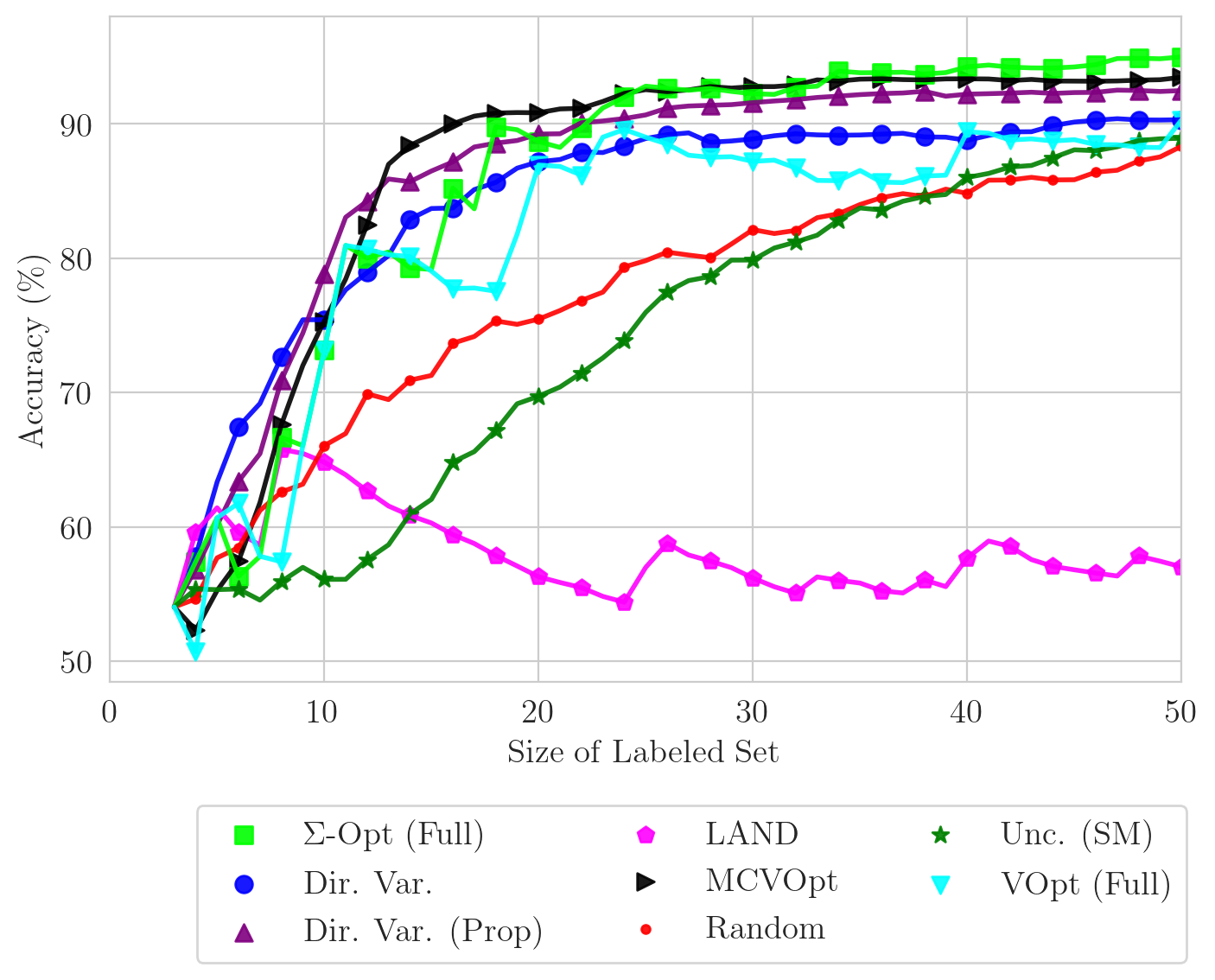}
        \caption{Accuracy}
        \label{fig:mnistsmall-acc}
    \end{subfigure}
    \begin{subfigure}{0.5\textwidth}
        \centering
        \includegraphics[width=\textwidth]{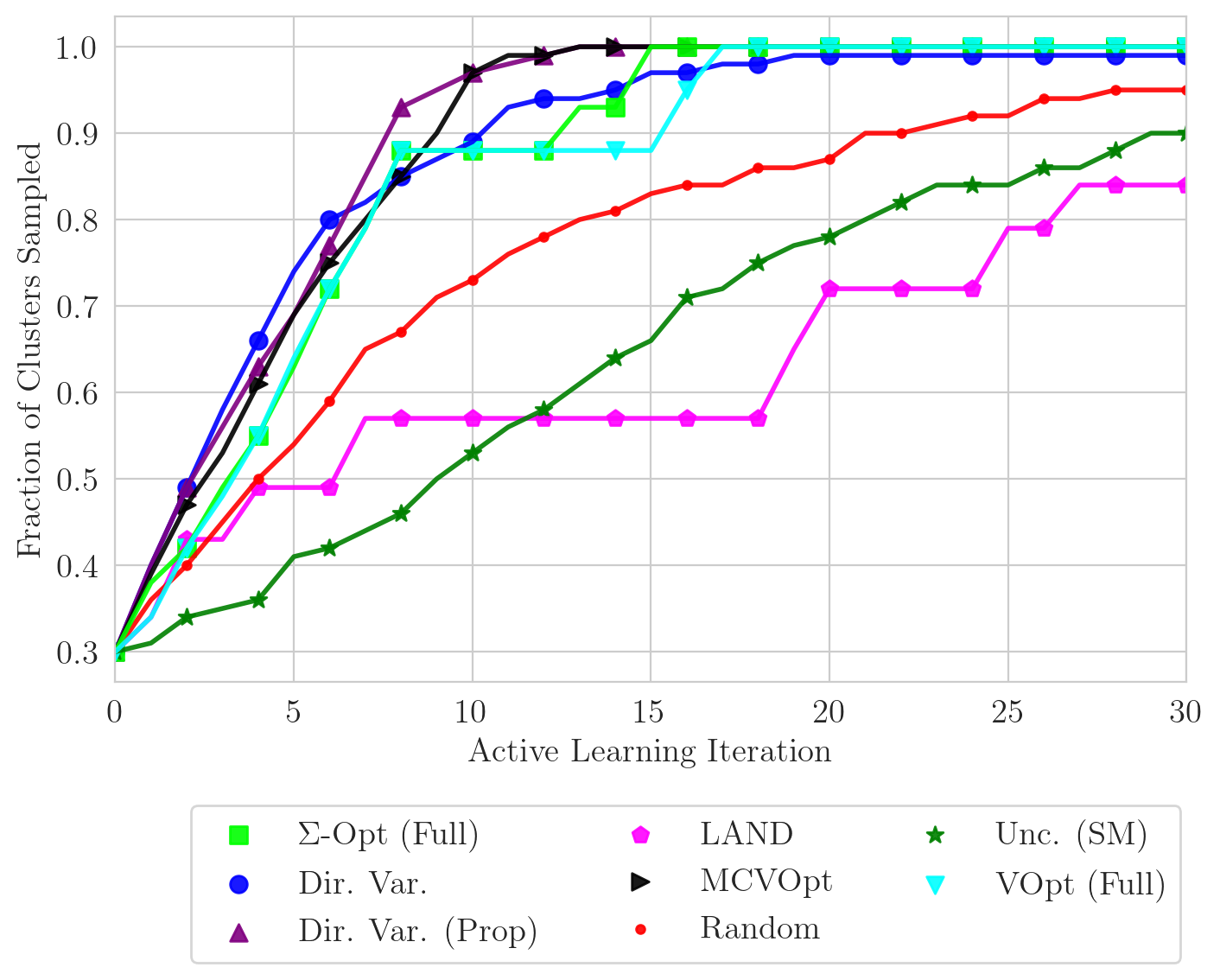}
        \caption{Cluster Exploration}
        \label{fig:mnistsmall-clusterexp}
    \end{subfigure}
    \caption{Plots of results on MNIST-SMALL dataset, showing both accuracies (a) and cluster proportion (b) as a function of the active learning iteration (i.e., the size of the labeled data). 
    }
\end{figure}

\begin{figure}
    \begin{subfigure}{0.5\textwidth}
        \centering
        \includegraphics[width=\textwidth]{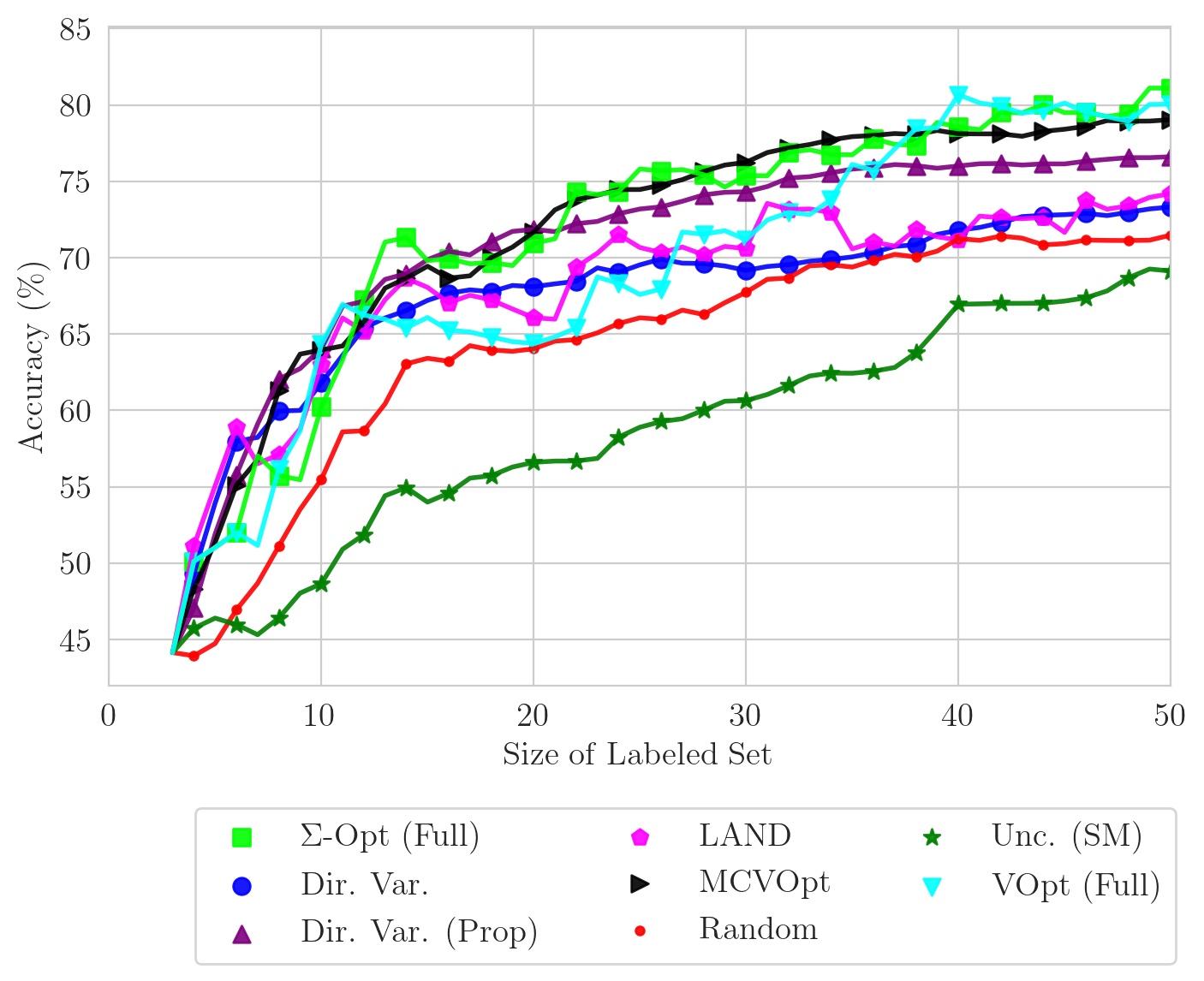}
        \caption{Accuracy}
        \label{fig:fashionmnistsmall-acc}
    \end{subfigure}
    \begin{subfigure}{0.5\textwidth}
        \centering
        \includegraphics[width=\textwidth]{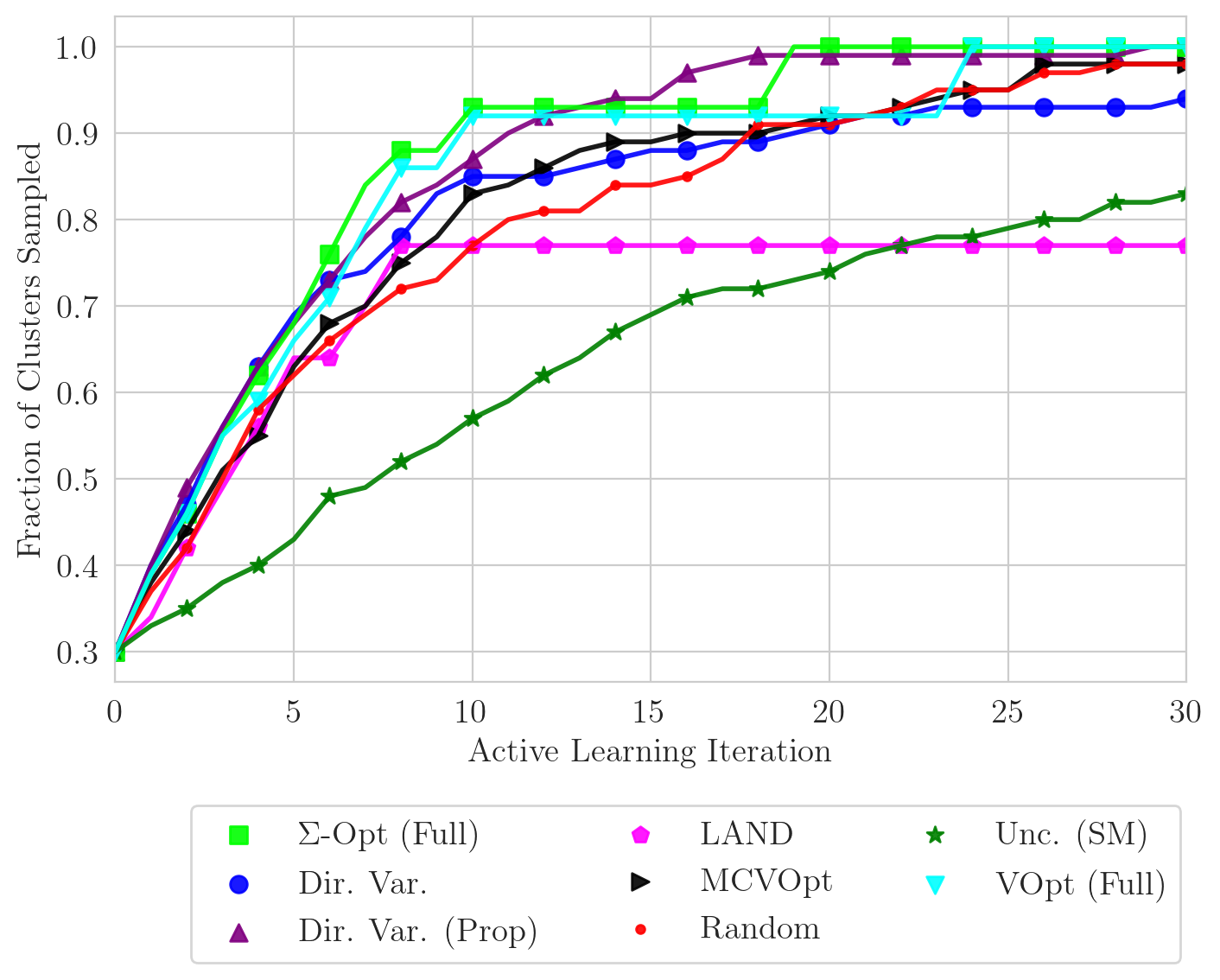}
        \caption{Cluster Exploration}
        \label{fig:fashionmnistsmall-clusterexp}
    \end{subfigure}
    \caption{Plots of results on FASHIONMNIST-SMALL dataset, showing both accuracies (a) and cluster proportion (b) as a function of the active learning iteration (i.e., the size of the labeled data). 
    }
\end{figure}

The full MNIST and FASHIONMNIST experiments then exclude comparison with the LAND acquisition function as well as the full VOpt, and $\Sigma$Opt acquisition functions. We use an approximate VOpt and $\Sigma$Opt calculation that utilizes a dimensionality reduction by projecting onto the eigenvectors corresponding to the $r=50$ smallest eigenvalues of the graph Laplacian, similar to what is done with the MCVOpt \citep{miller_spie_2022} criterion. Accuracy and cluster exploration results for both MNIST and FASHIONMNIST datasets are reported in Figures \ref{fig:mnist} and \ref{fig:fashionmnist}.

\begin{figure}
    \begin{subfigure}{0.5\textwidth}
        \centering
        \includegraphics[width=\textwidth]{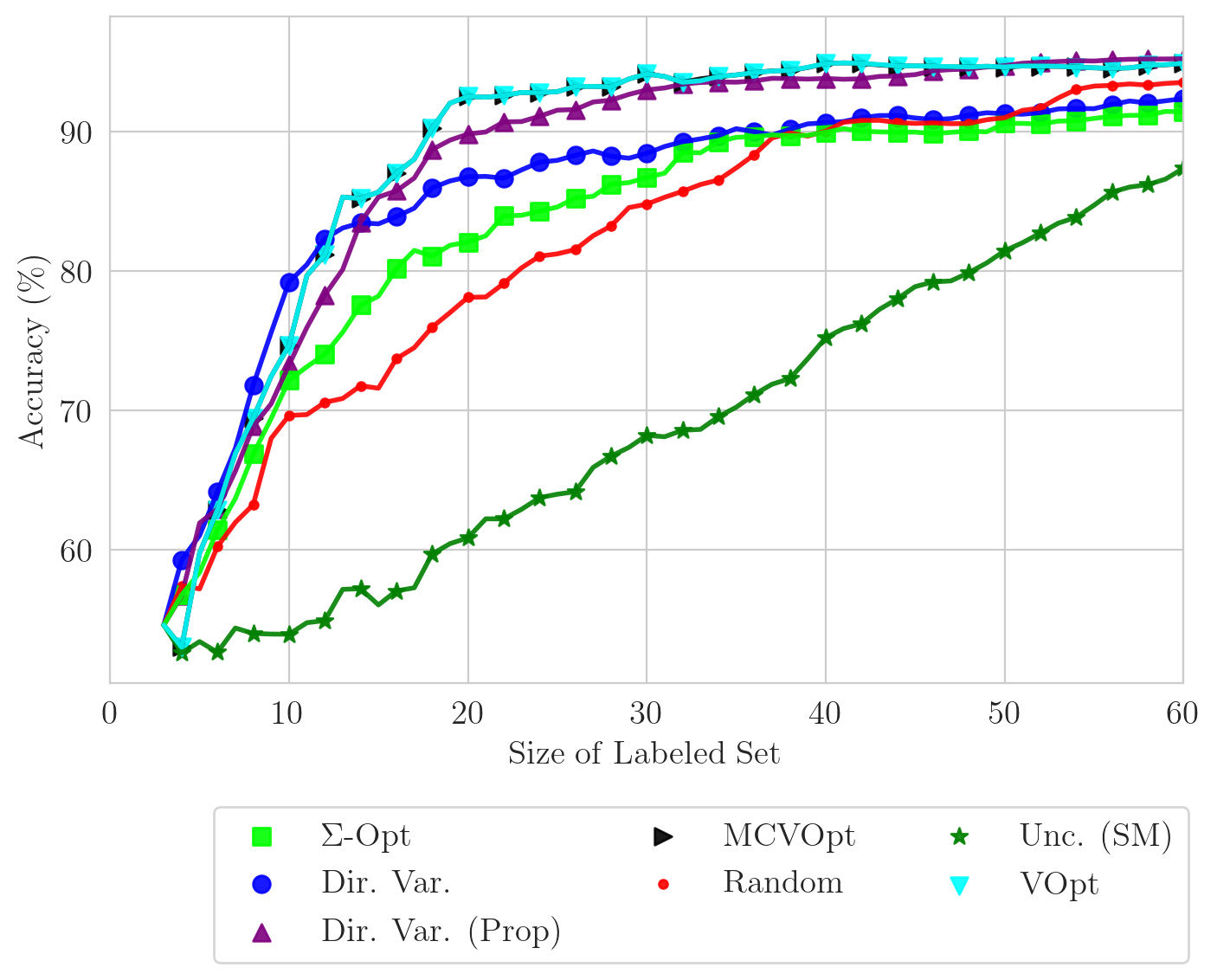}
        \caption{Accuracy}
        \label{fig:mnist-acc}
    \end{subfigure}
    \begin{subfigure}{0.5\textwidth}
        \centering
        \includegraphics[width=\textwidth]{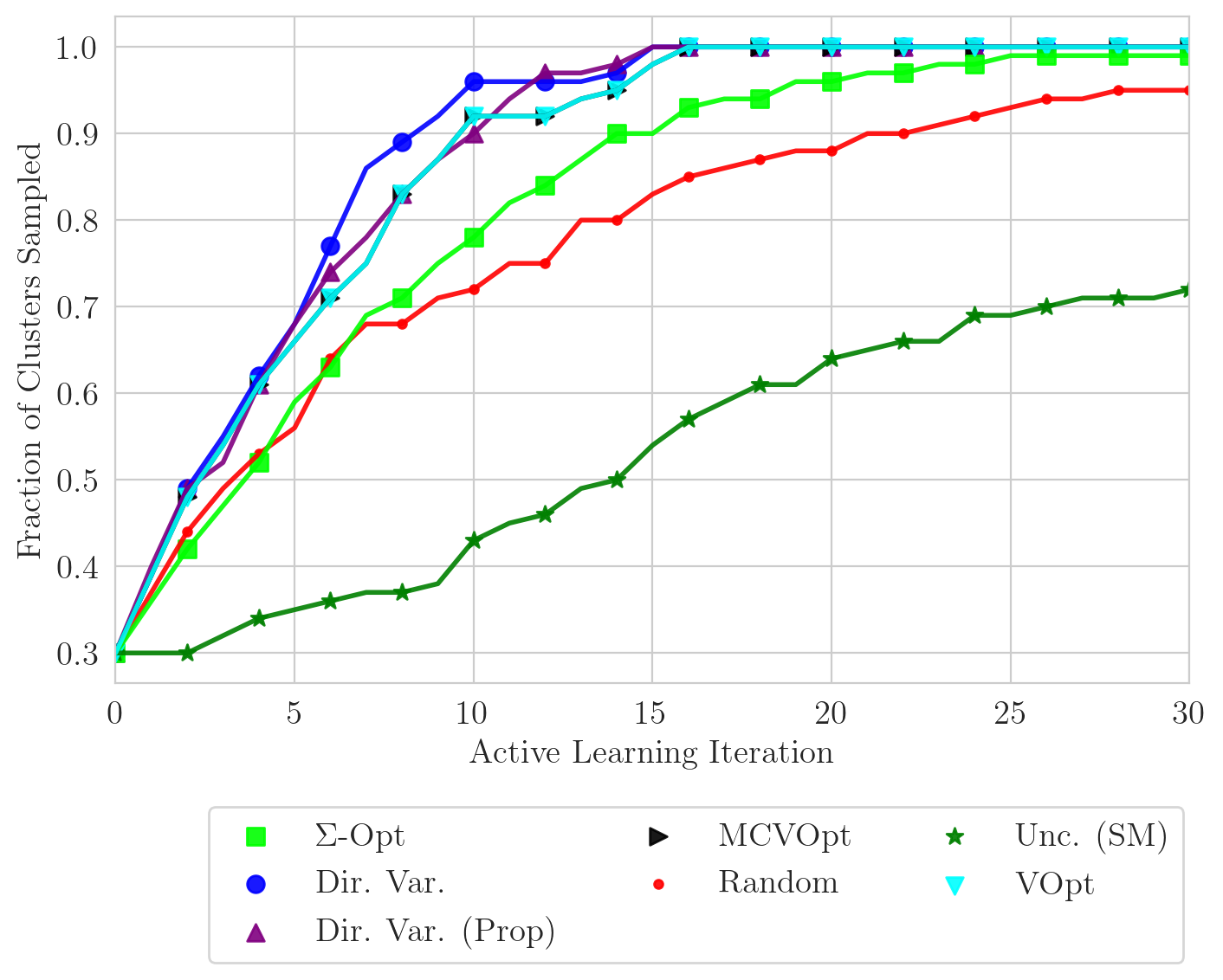}
        \caption{Cluster Exploration}
        \label{fig:mnist-clusterexp}
    \end{subfigure}
    \caption{Plots of Results on MNIST dataset, showing both accuracies (a) and cluster proportion (b) as a function of the active learning iteration (i.e., the size of the labeled data). 
    }
    \label{fig:mnist}
\end{figure}

\begin{figure}
    \begin{subfigure}{0.5\textwidth}
        \centering
        \includegraphics[width=\textwidth]{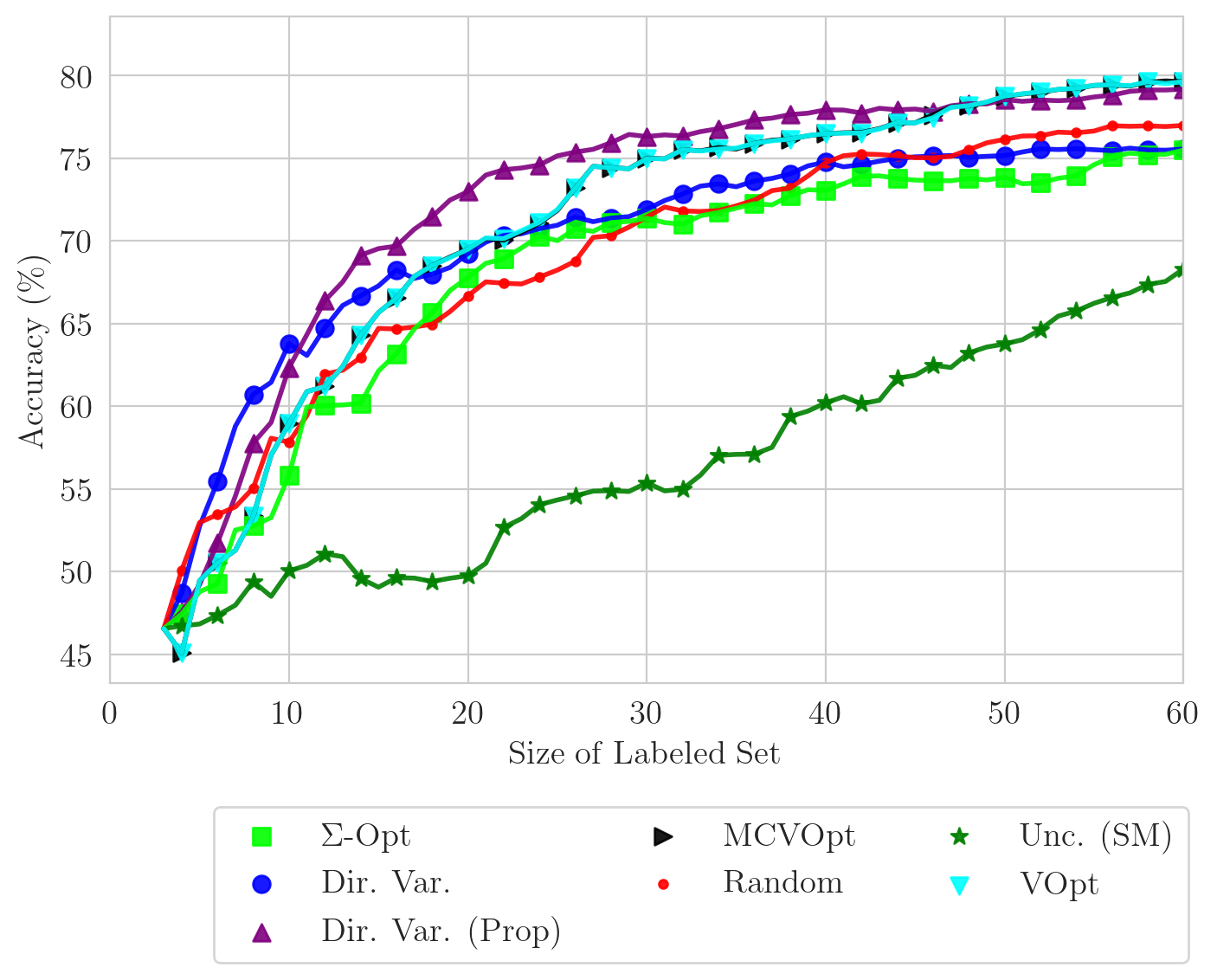}
        \caption{Accuracy}
        \label{fig:fashionmnist-acc}
    \end{subfigure}
    \begin{subfigure}{0.5\textwidth}
        \centering
        \includegraphics[width=\textwidth]{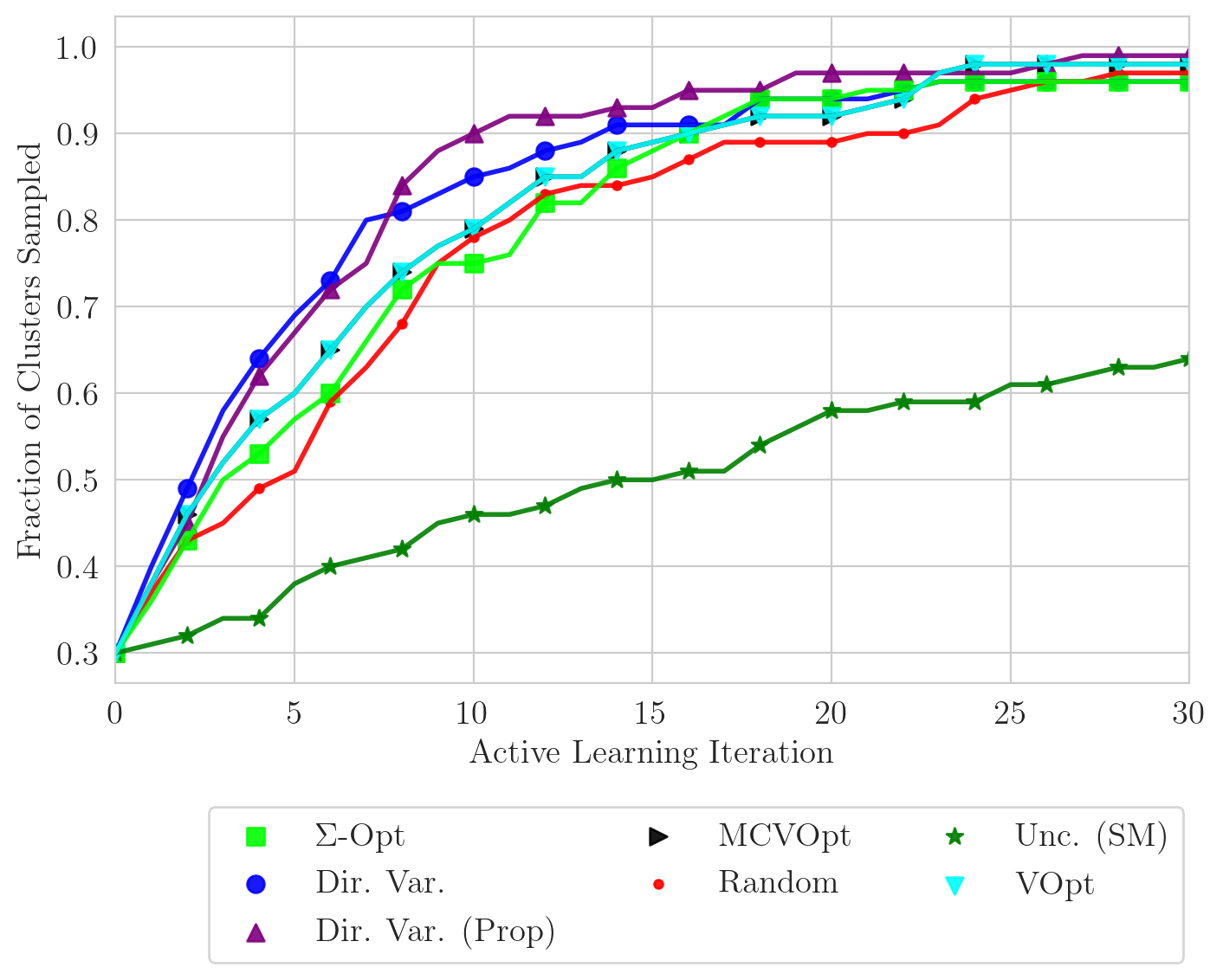}
        \caption{Cluster Exploration}
        \label{fig:fashionmnist-clusterexp}
    \end{subfigure}
    \caption{Plots of results on FASHIONMNIST dataset, showing both accuracies (a) and cluster proportion (b) as a function of the active learning iteration (i.e., the size of the labeled data).
    }
    \label{fig:fashionmnist}
\end{figure}

\subsection{Discussion} \label{sec:results-discussion}

We now discuss the results and subsequent implications of the various experiments of Sections \ref{subsec:hsi-results} and \ref{subsec:exploration-results}, as stated briefly in the summary box at the start of Section \ref{sec:results}.
% \todo{how to refer to the box?}. \km{in the summary box at the start of the section}.

\subsubsection{Assessing exploration capabilities}

Recall that plotting cluster proportion as a function of active learning iteration (labeled set size) allows us to monitor the \textit{explorative} nature of the selected query points by each acquisition function; namely, we can characterize a ``sufficiently'' explorative method as one that samples from \textit{every} cluster in fewer iterations than other methods. Then, by computing the accuracy at each active learning iteration, we can measure how useful the corresponding query points were for the classification task. Combining both measurements yields an arguably more complete picture of the utility of a proposed acquisition function.  

In the MNIST-SMALL (Figure \ref{fig:mnistsmall-acc}) and FASHIONMNIST-SMALL (Figure \ref{fig:fashionmnistsmall-acc}) experiments, we see that the MCVOpt and $\Sigma$Opt (Full) acquisition functions often led to the greatest marginal gains in accuracy compared to all the other methods shown. Our Dir.\,Var.\,(Prop) acquisition function performed comparably in terms of accuracy in both experiments, but with a more consistent exploration of clusters across the two experiments (see Figures \ref{fig:mnistsmall-clusterexp} and \ref{fig:fashionmnistsmall-clusterexp}). By comparison, it seems that the query points selected by LAND led to poor performance in the MNIST-SMALL task while merely sub-optimal performance in the FASHIONMNIST-SMALL task. An investigation of the cluster exploration plots of both experiments suggests that the LAND acquisition function struggled to identify query points from each of the 10 clusters in as few queries as the other presented methods. We suggest that this points to the importance of cluster exploration for the subsequent performance of the classifier in the active learning process.

In the larger experiments, MNIST and FASHIONMNIST, we observe similar results with the exception of the $\Sigma$Opt acquisition function (possibly due to the spectral approximation performed to reduce the computational costs for these larger experiments). We highlight that our proposed Dir.\,Var.\,(Prop) method still performs comparably to the best-performing acquisition functions on the MNIST task, and achieves the best performance of all methods on the FASHIONMNIST task. These results are encouraging, as across both the large and small tasks our proposed method has consistently performed well in terms of both accuracy and cluster exploration.

As a final observation from these experiments, we highlight the consistently poor performance of smallest margin uncertainty sampling (Unc.\,(SM)) computed in the Laplace Learning \citep{zhu_semi-supervised_2003} model. Such catastrophic behavior of this type of uncertainty sampling has been observed previously \citep{ji_variance_2012, miller2023unc}, and further supports our suggestion that exploration of cluster behavior is crucial for success in the active learning task. We posit that early on, when only a few of the clusters contain labeled samples, the resulting Laplace Learning classifier's decision boundaries are most likely too poor to utilize as the sole mechanism for selecting query points.

\subsubsection{Comparing computational expense} \label{subsubsec:compare-comp}

% \rwm{One important part of the case for Dirichlet Learning is that it is much more computationally efficient than many of its competitors \dots}
% In addition to comparable explorative behavior of the resulting query points, we highlight the computational efficiency of Dirichlet Variance as an acquisition function. 

In addition to the favorable explorative behavior of our proposed DiAL method, we emphasize the scalability of both the underlying semi-supervised learning model (Dirichlet Learning) and the Dirichlet Variance acquisition function. That is, updating the Dirichlet Learning classifier at each active learning iteration and the subsequent calculation of Dirichlet Variance are relatively cheap to compute.
% as compared to other acquisition functions that have been proposed to encourage explorative behavior. 
Expanding on this second point, the computation of Dirichlet Variance scales similarly to the ``optimal'' scaling of Uncertainty Sampling (Table~\ref{table:comp_comparison} and Figure~\ref{fig:timing-comparison}) for this pool-based active learning setting. This scaling is very favorable compared to other acquisition functions like VOpt and $\Sigma$Opt that have been proposed to encourage exploration of clusters \citep{ma_sigma_2013, ji_variance_2012, miller_model-change_2021, miller2023unc}.
% For example, VOpt and $\Sigma$Opt are informative acquisition functions that 
% Namely, VOpt and $\Sigma$Opt have shown to perform well at cluster exploration.
In their originally proposed form, both VOpt and $\Sigma$Opt require the computation and storage of the inverse of a perturbed graph Laplacian matrix of size $n^2$, where $n$ is the size of the dataset $\mcl X$. We have referred to this as VOpt/$\Sigma$Opt (Full) in the plots of results. MCVOpt \citep{miller_spie_2022} was proposed as a more computationally-efficient heuristic for combining VOpt and a type of uncertainty sampling, without requiring the inversion of said graph Laplacian. However, MCVOpt  still requires the computation of a subset of $r \ll n$ eigenvalues and eigenvectors of said graph Laplacian matrix to provide a low-rank approximation to its inverse. We refer to these additional variables (i.e., the inverse of the graph Laplacian matrix or its low-rank approximation) as \textit{auxiliary matrices} since they are not directly used by the underlying semi-supervised classifiers. 

In contrast, the Dirichlet Variance acquisition function requires no such auxiliary matrix prior to the start of the active learning process and therefore has no additional computational overhead. Furthermore, VOpt, $\Sigma$Opt, and MCVOpt all require updating these auxiliary matrices in addition to updating the classifier outputs at each iteration. By requiring no such auxiliary matrix, Dirichlet Variance requires fewer operations per unlabeled point than these other acquisition functions.

In Table \ref{table:comp_comparison}, we display a comparison of the computational requirements of the Dir.\,Var.\,, VOpt, $\Sigma$Opt, Unc.\,Sampling, and MCVOpt acquisition functions. \textit{Initial Compute} refers to the computational cost to initially calculate the associated auxiliary matrix if used by the acquisition function. \textit{Cost Per Unl.\,} refers to the cost of the acquisition function applied to a single unlabeled point and \textit{Aux.\,Update Cost} refers to the cost to update said auxiliary matrix. Finally, \textit{Class.\,Update Cost} refers to the cost of updating the corresponding classifier at each iteration (if the acquisition function uses the classifier outputs). Costs with an asterisk ($\ast$) are shown in their originally proposed form--such as full matrix inversion of the graph Laplacian matrix and storage of manipulations of this $n \times n$ dense matrix throughout the active learning process for VOpt and $\Sigma$Opt. 

Figure \ref{fig:timing-comparison} shows a timing comparison between these acquisition functions. For each acquisition function, we computed the average time over 10 trials to perform one iteration of the active learning process (i.e., the time required to compute the acquisition function over the unlabeled data and subsequent selection of query point). Due to the near-identical form of $\Sigma$Opt and VOpt, we only include VOpt in the plot. Figure \ref{fig:timing-comparison} empirically verifies the \textit{Cost Per Unl.\,}column of Table \ref{table:comp_comparison}; namely, one iteration of the active learning process incurs computational cost roughly like
\[
    (\text{cost per iteration}) = \underbrace{(\text{size of unlabeled data})}_{\approx\mcl O(n)} \times (\text{cost per unlabeled point}).
\] 
Thus, VOpt (Full) scales quadratically with dataset size, $n$, while the other acquisition functions scale roughly linearly with $n$, with additional overhead for MCVOpt due to the extra computations associated with the $r$ eigenvalues and eigenvectors of the graph Laplacian matrix.

\begin{table}[h] 
\centering
\begin{tabular}{@{}c|c|cc|c@{}}\hline
& Aux.\,Overhead  & \multicolumn{2}{c|}{Query Point Selection} & Semi-Sup.\,Inference\\
\textit{Abbr.\,Name} & \textit{Initial Compute} & \textit{Cost Per Unl.}  & \textit{Aux.\,Update Cost}  &  \textit{Class. Update Cost}\\
\hline
Unc. Sampl. & - & $\mcl O(K)$ & -  & $K \cdot SpSolve(L,n) $\\
Dir.\,Var.\, & - & $\mcl O(K)$ & -  & $SpSolve(L,n)$ \\
VOpt (Full) & $\mcl O(n^3)^\ast$ & $\mcl O(n)$ & $\mcl O(n^2)^\ast$  & - \\
$\Sigma$Opt (Full) & $\mcl O(n^3)^\ast$ & $\mcl O(n)$ & $\mcl O(n^2)^\ast$  & - \\
MCVOpt & $\mcl O(n^2r)$ & $\mcl O(r + K)$ & $\mcl O(nr)$  & $K \cdot SpSolve(L,n) $\\
\hline
\end{tabular}
\caption{Computational comparison between acquisition functions, where $n = |\mcl X|$, $K$ is the number of classes, and $r \ll n$ is the number of eigenvalues computed in the auxiliary matrix used in MCVOpt. $SpSolve(L,n)$ represents the cost of a sparse linear solve with the graph Laplacian matrix $L \in \mbb R^{n \times n}$, which we assume to be sparse pursuant our use of $k$-nearest neighbor graphs.
% \rwm{This doesn't include the cost associated with evaluating propagation operators, correct? Did you use SpSolve anywhere here?}
% \km{Need to check this in our discussion}\rwm{In terms of computation, if we're only ever needing to solve SpSolve(L,n), then we can do lots of the computations up front (factorizations etc)...}
}
\label{table:comp_comparison}
\end{table}

\paragraph{Low-rank approximation of VOpt and $\Sigma$Opt in larger experiments.} For the larger experiments, the matrix inversion involved in the ``full'' VOpt and $\Sigma$Opt is impractical, and so we use a computational workaround similar to the MCVOpt acquisition function. As was done in \citep{miller_model-change_2021, miller_spie_2022}, we approximate the VOpt and $\Sigma$Opt acquisition functions by projecting onto the first $r=50$ eigenvectors of the corresponding graph Laplacian in order to reduce the computational burden of these methods. This constitutes a low-rank approximation of the inverse of the graph Laplacian matrix for the corresponding calculations of these acquisition functions. The error in this low-rank approximation may explain the comparative degradation in performance of the $\Sigma$Opt acquisition from the smaller to the larger datasets, though we remark that this approximation did not seem to significantly affect the results of VOpt and MCVOpt. An interesting direction for empirical and theoretical analysis would be to quantify how such low-rank approximations affect these acquisition functions that are intimately tied to covariance operators associated with the underlying graph-based Gaussian random field. 

\begin{figure}[h]
    \hspace{10.5em}
    \includegraphics[width=0.7\textwidth]{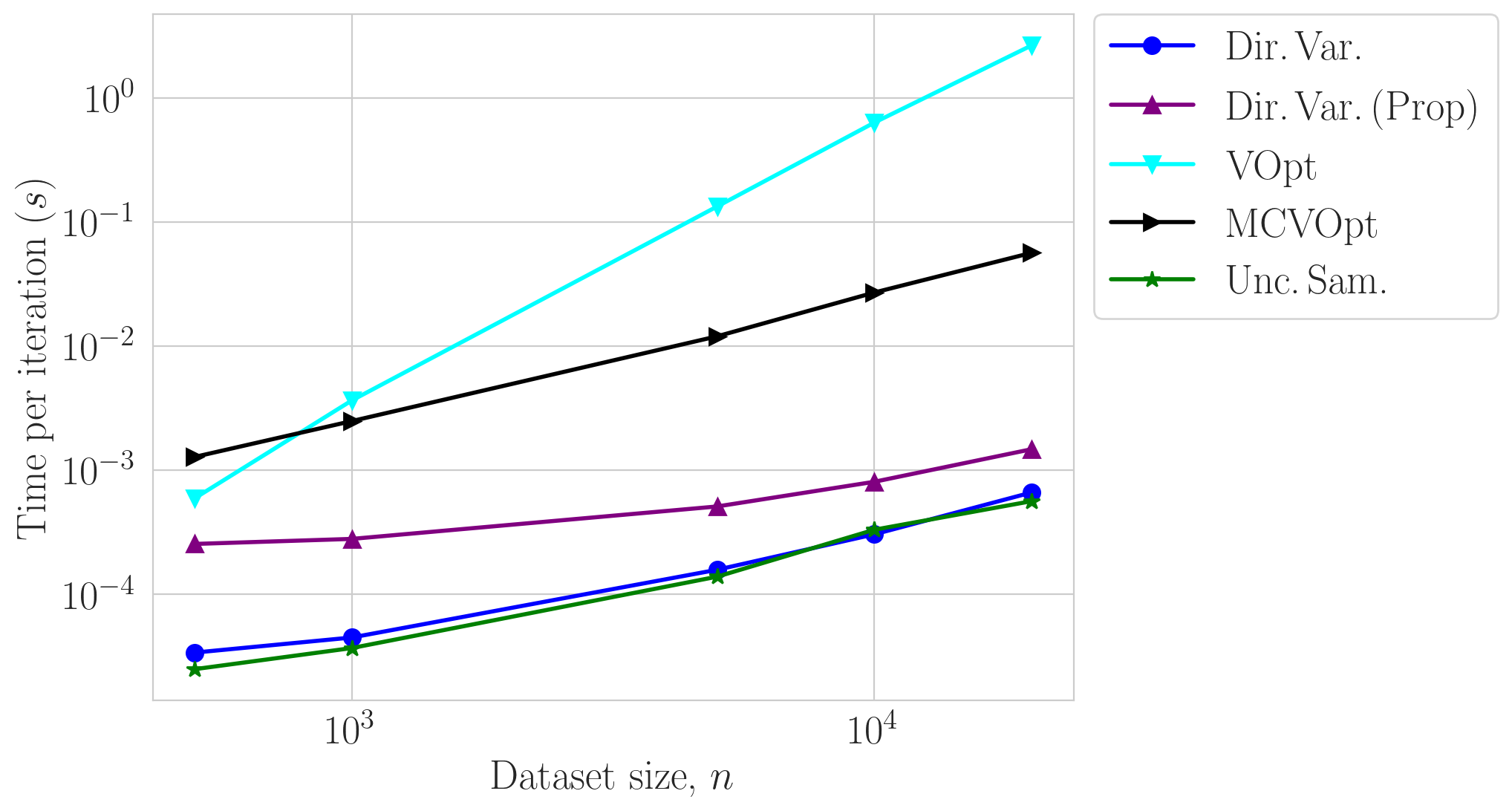}
    \caption{Timing comparison of acquisition function computation per iteration. For each acquisition function and dataset size ($n$), we plot the average time it took to compute the corresponding acquisition function over the unlabeled data. Note that both Dir.\,Var.\, acquisition functions scale like $\mcl O(n)$ similar to Unc.\,Sampling, while VOpt scales quadratically as $\mcl O(n^2)$. MCVOpt scales roughly like $\mcl O(n)$, with the additional overhead due to computations of order $\mcl O(r)$, where $r$ is the number of eigenvalues of the graph Laplacian used. Note that the additional cost incurred by sampling proportional to acquisition function values accounts for the timing difference between Dir.\,Var.\,and Dir.\,Var.\,(Prop).}
    \label{fig:timing-comparison}
\end{figure}

% \subsubsection{Dirichlet Learning as a useful model for semi-supervised classification}

% We conclude this discussion section with a remark about the utility of our proposed Dirichlet Learning as a semi-supervised classifier. The current work proposes a novel graph-based semi-supervised learning classifier (Dirichlet Learning) wherein the modeling of uncertainty in classifier outputs is a straightforward function of the outputs at each node (Dir.\,Var.\,). It is an important question whether or not the proposed classifier provides useful or ``accurate'' inferences for unlabeled nodes given the labeled data; that is, one may ask whether Dirichlet Learning as a classifier itself performs comparably to other graph-based semi-supervised classifiers such as Laplace Learning \citep{zhu_semi-supervised_2003}, Poisson Learning \citep{calder_poisson_2020}, Properly-Weighted Laplacian Learning \citep{calder_properly-weighted_2020}, or Multiclass Graph MBO \citep{merkurjev_mbo_2013}. We include in Appendix \todo{ref} a brief experiment that demonstrates that Dirichlet Learning performs comparably to Poisson Learning and Properly-Weighted Laplacian Learning in the assumed setting of low-label rate semi-supervised learning, which has been a setting that these methods have shown to provide the greatest benefit over alternative semi-supervised methods. \todo{Still want to include?}

\section{Theory} \label{sec:theory}

In this section, we present a theoretical analysis regarding the use of Dirichlet Active Learning (DiAL) to both explore in low-data regimes and to asymptotically exploit in high-data regimes. For analytical convenience, we will work in the continuum regime 
and make various assumptions upon the propagation operator and acquisition functions. However, these theoretical results help to match and explain the type of performance we observed in our numerical results (Section \ref{sec:results}), and we will point out where some of the analysis could be extended to the discrete data setting or to other kernels and acquisition functions.

For the sake of concreteness, we will consider a domain $\mcl X \subset \mbb R^d$ that is open and bounded throughout this section, though the analysis could be extended without significant change to a compact manifold with proper boundary conditions. Furthermore, we consider a joint distribution over inputs $(X, Y) \in \mcl X \times [K]$ that is identified by a density $\nu(x,y)$, as well as the marginal distribution over the inputs identified by a density $\rho : \mcl X \rightarrow [0, \infty)$.  Lastly, we will assume that the class-conditional distributions have densities $\{\rho_k(x)\}_{k=1}^K$ and that the marginal density over inputs can be written as the mixture model $\rho(x) = \sum_{k=1}^K w_k \rho_k(x)$ where $\sum_{k=1}^K w_k = 1, w_k \geq 0$. It is worth noting that the properties of $\rho(x)$ capture the continuum analog of the clustering structure of a finite dataset of points $\{x^{i}\}_{i=1}^n \subset \mcl X$ when they are sampled $x \sim^{iid} \rho$; that is, regions of $\mcl X$ wherein a class-conditional $w_k \rho_k(x)$ is significantly larger than the other components can be used to model clustering structure in data. 

% \km{\paragraph{Acquisition function evaluated on all of $\mcl X$.} 
Finally, in contrast to how we have previously defined the domain of the acquisition function to be the unlabeled data, $\mcl U = \mcl X \setminus \mcl L$, we will consider the Dirichlet Variance acquisition function evaluated on the entire domain $\mcl X$ at each iteration in this continuum setting. This can be interpreted as allowing for ``repeated trials'' or observations at points in the domain $\mcl X$. Furthermore, in the continuum setting, the labeled data $\mcl L$ constitutes a discrete set of measure zero with respect to the underlying marginal density $\rho$, and so considering the acquisition function on the entired domain $\mcl X$ seems natural.  
% }

\subsection{Summary of results and discussion}

The numerical experiments in Section \ref{sec:results} indicate that DiAL flexibly transitions from low-label to high-label regimes, or in other words from an exploration phase to an exploitation phase. As such, our theory seeks to address both regimes. We begin in Section \ref{subsec:exploration-guarantees} by defining a class of \emph{discriminating kernels} which are, with high probability, able to differentiate between different components of the mixture model (see Definition \ref{def:class-separator}). This allows us to describe (Proposition \ref{prop:cluster-exploration}) the explorative tendencies of DiAL to effectively cover the underlying distribution by providing reasonable guarantees that the algorithm will eventually capture all of the distributional structure and not miss components or clusters.

This flexible framework of discriminating kernels generalizes several different notions of clustering previously used in the literature.  In particular, we show in Section \ref{subsec:laplacian-based-separation} that these assumptions are satisfied under a flexible definition of clustering described in \citep{trillos2021geometric}, and we build upon their analysis in the context of spectral clustering to show that Laplacian-based kernels will be discriminating for such clustering structures. 
%The definition we give for discriminating clusters would also likely be applicable in other settings: for example \todo{TODO} \km{Not sure if we need this last sentence, given the next paragraph}.% This  for modeling such clustering structure. We adopt the framework presented in \citep{trillos2021geometric}, wherein well-clustered data is modeled via a mixture model for the data-generating input distribution where the overlap, coupling, and clusteredness of the components is directly modeled. 

Of course, other notions of clustering have previously been utilized to provide exploration guarantees for active learning algorithms. These include, for example, explicit conditions on the inter- and intra-cluster distances \citep{murphy_unsupervised_2019}, $\ell^p$ balls that are well-separated \citep{karzand_maximin_2020}, and high-density regions of the data-generating distribution that are well-separated \citep{miller2023unc}. As in all of those works, at this stage we focus on exploration guarantees for our active learning algorithm. We follow the definitions in \citep{trillos2021geometric} due to (i) the direct connection to Laplacian-based methods, upon which the proposed methods in this work focus, and (ii) the fact that the assumptions given below are much less restrictive than many utilized in other similar works. 

In Section \ref{subsec:exploit-analysis} we then analytically study the manner in which DiAL will asymptotically seek more information near classification boundaries. In this case, our analysis is more formal particularly because one needs to be careful in designing a proper mathematical model that can capture what one means by asymptotic exploitation. To this end, we derive a large-sampling limit for DiAL which takes the form of an integro-differential equation \eqref{eq:Int-Diff-alpha}. In a small kernel bandwidth limit, this equation demonstrates a clear sampling bias of DiAL towards regions with greater population-level classification uncertainty: a concrete statement of this phenomenon can be found in Equation \eqref{eqn:barq-exp}. We consider this behavior to be asymptotically exploitative, in the sense that DiAL with Prop.\,Sampling (\ref{table:dirvar-defs}) will, in high data regimes, spend most of its time sampling near decision boundaries if the scaling of $\lambda$ is chosen properly. This is comparable to the behavior observed for uncertainty sampling, which focuses all of its attention on such regions, and contrasts strongly with the behavior of VOpt, which never transitions its attention towards decision boundaries. Our formal analysis also suggests that the scaling of the inverse-temperature parameter $\lambda$ plays an important role in ensuring convergence to this stationary distribution that focuses along decision boundaries; we provide a numerical demonstration of this behavior in 1D in Section \ref{subsubsec:num-demo-uncertainty}. 

Other works have previously considered this broad question of exploitative behavior  (i.e., the focusing of query point selection along ground truth decision boundaries between classes) \citep{dasgupta_two_2011, dasarathy_s2_2015, balcan_margin_2007, rittler2023twostage}. To our knowledge, however, there are relatively few algorithms with rigorous guarantees that can successfully \textit{transition from exploring to exploiting}, and the analysis in this section demonstrates that this is the case with DiAL. 

Finally, in Section \ref{sec:dir-learning-consistency} we present a proof of consistency of the underlying Dirichlet Learning classifier~\eqref{eq:inference} associated with DiAL. Although our proposed work is motivated by representing uncertainty of classifications at each point $x \in \mcl X$ for use in active learning, the classifier
\[
    \hat{y}(x) = \argmax_{k = 1, 2, \ldots, K} \alpha_k(x) = \argmax_{k = 1, 2, \ldots, K} \sum_{x^\ell \in \mcl L_k} \Ker(x^\ell, x)  
\]
reduces to a kernel-based decision rule. This is reminiscent of methods such as Nadarya-Watson kernel regression estimators \citep{nadarya1964onestimating, watson1964smooth}, $k$-nearest neighbor classifiers \citep{coverhart1967}, and moving window decision rules \citep{rosenblatt1956remarks, parzen1962estimation}. An important question for such decision rules concerns the statistical consistency of these methods as more labeled data is observed; as such, we establish the asymptotic consistency of the Dirichlet Learning classifier \eqref{eq:inference} in the setting of labeled pairs observed passively from the data-generating distribution, $(X_i,Y_i) \sim^{iid} \nu$.

\subsubsection{Heat kernel propagation for theoretical analysis}

While we have focused on the graph-based Poisson propagation operator introduced in \eqref{eq:poisson-prop} and \eqref{eq:poisson-prop-shift} for our numerical results, we turn our attention in the continuum regime to a heat-kernel propagation operator as opposed to the continuum-limit analog of the Poisson propagation. For a domain $\mcl X \subset \mbb R^d$, the \textit{density-dependent} heat kernel $\Ker_t(z, x)$ with source $z \in \mcl X$ that solves 
\begin{equation} \label{eq:heat-eqn-start}
\begin{dcases}
\begin{aligned}
    \partial_t\Ker_t(z, x) - \Delta_\rho\Ker_t(z, x)&= 0  &x \in \mcl X, t > 0 \\
   \Ker_0(z, x) &= \delta_z(x) & \\
\end{aligned}
\end{dcases}
\end{equation}
where $\Delta_\rho = \frac{1}{\rho} \operatorname{div} \lp \rho^2 \nabla \cdot \rp$ is a \textit{self-adjoint} diffusion operator with respect to the $\rho$-weighted inner product. We note that the differential operators in \eqref{eq:heat-eqn-start} are with respect to the variable $x$, while the variable $z$ is the source. Furthermore, depending on the domain $\mcl X$, one must assume boundary conditions to make \eqref{eq:heat-eqn} well-defined. In general, we will state the necessary assumptions on the heat kernel, $\Ker_t(z,x)$, and provide reasonable example situations in which these assumptions hold. 

We choose to use the heat kernel for a number of reasons. One main reason that we choose not to analytically study Poisson propagation is that it induces analytical complications in continuum settings. For example, it is  still an open problem to establish a formal continuum limit for the graph-based Poisson propagation operator \citep{calder_poisson_2020} due to the highly singular nature of the right-hand side of the corresponding governing equation
\[
    -\Delta_\rho u  + \tau u = (\tau I - \Delta_\rho) u = \delta_{z}(x).
\]
Furthermore, by considering the fundamental solution of the Laplacian on $\R^d$, we expect solutions of this equation to go to infinity at $z$, a point which is handled in a delicate way by solution rescaling in \citep{calder_poisson_2020}. The heat equation, in contrast, is much more amenable to analysis in the continuum setting.

We note, however, that the Poisson propagation can be viewed as an approximation to a corresponding heat-kernel propagation. Informally, one could consider using a single backwards Euler step of the heat equation to write
\begin{equation*}
   \Ker_t(z,x) = e^{t\Delta_\rho}\delta_z(x) \approx (I - t \Delta_\rho)^{-1} \delta_z(x) = t u_{1/t}(x),
\end{equation*}
where $(\tau I - \Delta_\rho) u_{\tau}(x) = \delta_z(x)$ is the continuum-limit analog of the Poisson propagation \eqref{eq:poisson-prop}. This shows the relationship between using the heat kernel propagation $\Ker_t(z,x)$ and the Poisson propagation, and so we continue with our theory in the heat kernel setting in the continuum.

\subsection{Cluster discovery guarantees} \label{subsec:exploration-guarantees}

Our first goal will be to establish high-probability cluster discovery guarantees. At a high level, we will show that given $K$ classes which are derived from $K$ distinct clusters, then with high probability  Dirichlet Learning will sample from each of the classes in $K$ steps. The subsequent generalization to $C > K$ clusters follows in similar fashion and is addressed in Remark \ref{rem:C-clusters}. Of course, such behavior is intimately connected with the particular choice of kernel, propagation function, and underlying probability densities. In this section, we will always assume that $\mcl X$ is a subset of $\mbb R^d$. 
%and that the densities of the underlying classes are all \todo{TODO} \km{I don't remember what ``assumption'' we were going to make on the densities...}.

Our first aim will be to provide an abstract notion describing the ability of a particular kernel to separate classes. To this end, we give the following definition: 

% \rwm{Do we want to use $\mcl X$ or $E$ in the following definition?}

% \km{What do you think of $\mcl X_i$ instead of $E_i$ for the disjoint sets?}

\begin{definition}\label{def:class-separator}
    We call a particular propagation operator $\Ker$ a $(\delta,\zeta, \varepsilon)$ class separator of a mixture model $\{ \rho_i\}$ if there exist disjoint sets $\mcl X_i$ such that $\rho_i(\mcl X_i) \geq 1-\delta$ and so that for any $x \in \mcl X_i$ we have that  $\Ker(x,x') > \zeta$ if $x' \in \mcl X_i$ and $\Ker(x,x') < \varepsilon$ if $x' \in \mcl X_j$ with $j \neq i$.
\end{definition}

This notion, of course, will be highly dependent upon the kernel and underlying densities, and we will prove that such a property holds for specific situations in Section \ref{subsec:laplacian-based-separation}. Under the assumption that our operator is separating, we then provide the following concrete result about cluster exploration, which follows from a direct, probabilistic argument.

\begin{proposition} \label{prop:cluster-exploration} Suppose that $\Ker$ is a $(\delta,\zeta, \varepsilon)$ class separator and that we select query points for our active learning algorithm using the kernel $\Ker$ as a propagation function and using Dirichlet Variance with Prop.\,Sampling (i.e., select $x \in \mcl X$ to label with probability $q(x) \propto \exp ( \lambda \mcl A_{var}(\alpha))$). 
% We furthermore assume that our initial Dirichlet prior uses parameter $\alpha_0$ such that $K\varepsilon \leq \alpha_0 \leq \frac{r \zeta}{K}$ for some $r \in (0,1)$.
Then, with probability at least 
{\small 
\begin{multline*}
    \left\{ 1 - \frac{\delta}{(1 - \delta) w_{min}} \exp\left( \frac{\lambda(K+1)(2\alpha_0 + 1)}{K\alpha_0^2 (K\alpha_0 + 1)} \right) \right. \\
    \left. - \frac{\lp  1 - w_{min}(1 - \delta)\rp}{(1 - \delta)w_{min}} \exp\left( \lambda \lp C(\alpha_0, \varepsilon, \zeta, K)  - 1\rp \frac{ K(K+1)\alpha_0^2}{K^2(\alpha_0 + \varepsilon)^2(K(\alpha_0 + \varepsilon) + 1)} \right) \right\}^K
\end{multline*}
}
Dirichlet Learning will sample from each of the $K$ classes in $K$ steps, where 
\begin{equation} \label{eq:const-exploration-bound}
    C(\alpha_0, \varepsilon, \zeta, K) := \lp 1 + \frac{4(\alpha_0 + 1)}{K(\alpha_0 + \varepsilon)} \rp \lp \frac{(K-1)(\alpha_0 + \varepsilon)^4}{(K+1)\alpha_0^2 \lp \alpha_0 + \frac{\zeta}{K}\rp^2} \rp
\end{equation}
and $w_{min} = \min_{k=1, 2, \ldots, K} w_k$ is the smallest class weighting. 
% \km{@Ryan: See blue section at end of proof for derivation of $1 - w_{min}(1-\delta)$ value on rightmost term of probability bound.}
\end{proposition}

\begin{remark} \label{rem:C-clusters}
    We briefly note that while Definition \ref{def:class-separator} and Proposition \ref{prop:cluster-exploration} are stated in terms of $K$ clusters (one for each of the $K$ different classes), this setup straightforwardly generalizes to the situation of multiple, disjoint clusters in each class. If there are a total of $M$ different clusters--each one belonging to precisely one of the $K$ distinct classes--then the same reasoning used to lower bound the Dirichlet Variance values on the unexplored classes extends to the unexplored clusters belonging to classes that have labeled points in other clusters.
\end{remark}

\begin{remark} \label{rem:explain-bound}
    We also remark that the bound in Proposition \ref{prop:cluster-exploration} is meaningful when the quantity $C(\alpha_0, \varepsilon, \zeta, K) < 1$ so that the rightmost term can be reasonably bounded as  $\lambda > 0$ is increased. The inner term will always worsen as $\lambda$ increases, but this term is kept small by the constant $\delta \ll 1$. Consider the elucidating case when $\varepsilon = 0, \zeta = 1,$ and $\delta = 0$ (i.e., $\Ker$ ``perfectly'' separates the classes in the mixture $\rho$) and $\alpha_0 = \frac{1}{K^2}$, then straightforwardly we have that
    \[
        C\left(\frac{1}{K^2}, 0, 1, K\right) = \lp 1 +\frac{4(K^2 + 1)}{K}\rp \frac{(K-1)}{(K+1)^3} \leq  \frac{1}{2} <  1,
    \]
    and so our probability estimate becomes 
    \[
        \lp 1 - w_{min}^{-1} \exp \lp  -  \frac{\lambda(K+1)}{2K(K\alpha_0 + 1)}  \rp \rp^K,
    \] 
    which goes to $1$ as $\lambda \to \infty$.
\end{remark}

% \paragraph{Proof of Proposition \ref{prop:cluster-exploration}}
\begin{proof}[Proof of Proposition \ref{prop:cluster-exploration}]
We will provide a proof by induction. Suppose that the first $J < K$ samples belong to the sets $\mcl X_{i_j}$ with $i_j \neq i_k$ when $j \neq k$, and that $y_j = i_j$, where the $\mcl X_i$ are as in Definition \ref{def:class-separator}. With $C(\alpha_0, \varepsilon, \zeta, K)$ as defined in \eqref{eq:const-exploration-bound}, we claim that with probability at least
\begin{align*}
    1 &- \frac{\delta}{(1 - \delta) w_{min}} \exp\left( \frac{\lambda(K+1)(2\alpha_0 + 1)}{K\alpha_0^2 (K\alpha_0 + 1)} \right)   \\ 
    & \qquad - \frac{\lp  1 - w_{min}(1 - \delta)\rp}{(1 - \delta)w_{min}} \exp\left( \lambda \lp C(\alpha_0, \varepsilon, \zeta, K)  - 1\rp \frac{ K(K+1)\alpha_0^2}{K^2(\alpha_0 + \varepsilon)^2(K(\alpha_0 + \varepsilon) + 1)} \right)
\end{align*}
that $x_{J+1}$ will belong to $\mcl X_{i_{J+1}}$, with $i_{J+1} \neq i_j$, for all $j \leq J$, and that $y_{J+1} = i_{J+1}$. In words: with the given probability the subsequent sample will be from a $\mcl X_i$ for which we have not yet observed a label.

We notice that for any $x \in \mcl X_{i_j}$, for some $j = 1 \dots J$, we have that $\alpha_{i_j}(x) > \zeta + \alpha_0$, whereas all other $\alpha_i(x)$ are less than $\varepsilon + \alpha_0$. 
% \km{Have : $0 \leq \alpha_i \leq \varepsilon, \zeta < \alpha_\ell < 1, K\varepsilon \leq \alpha_0 \leq \frac{r\zeta}{K}$}. 
This then implies for $x \in \mcl X_{i_j}$ that 
\begin{align*}
    \mcl A_{var}(x) &= \frac{2 \sum_{i \neq k} \alpha_i \alpha_k}{\beta^2(\beta + 1)} \leq \frac{4(K-1)(\alpha_0 + \varepsilon) (\alpha_0 + 1) + K(K-1)( \alpha_0 + \varepsilon)^2}{(K\alpha_0 + \zeta)^2  (K \alpha_0 + \zeta + 1)} =: \mcl A_{var}^{lab}.
\end{align*}

On the other hand, for any $x \in \mcl X_i$ such that $i_j \neq i$ for all $j = 1 \dots J$ we have that $\alpha_k(x) < \varepsilon$ for all $k$. We then obtain that 
\begin{align*}
    \mcl{A}_{var}(x) &= \frac{2 \sum_{i \neq k} \alpha_i \alpha_k}{\beta^2(\beta + 1)} \geq \frac{K(K+1)\alpha_0^2}{K^2(\alpha_0 + \varepsilon)^2(K(\alpha_0 + \varepsilon) + 1)} =: \mcl A_{var}^{unl}.
\end{align*}
Let $\hat{\mcl X}$ denote the complement of all of the  $\mcl X_i$. We notice that at all points we have that the variance is smaller than 
\begin{align*}
    \mcl A_{var}(x) &\leq \frac{K(K+1)(\alpha_0 + 1)^2}{(K\alpha_0)^2 (K\alpha_0 + 1)} = \frac{(K+1)(\alpha_0 + 1)^2}{K\alpha_0^2 (K\alpha_0 + 1)} =: \mcl A_{var}^{back}.
\end{align*}

Now, we notice that the normalizing constant for our sampling distribution will be at least 
\[
    \int_\mcl X e^{\lambda V(\alpha(x))} \rho(x)\, dx \ge \sum_{i_k : k > J} w_{i_k} \int_{\mcl X_{i_k}} e^{\lambda V(\alpha(x))} \rho_{i_k}(x)\, dx \ge (1 - \delta) w_{min} \exp\lp \lambda \mcl A_{var}^{unl}\rp. %\exp\lp   \frac{\lambda K(K+1)\alpha_0^2}{K^2(\alpha_0 + \varepsilon)^2(K(\alpha_0 + \varepsilon) + 1)} \rp.
\]
Now, we can bound
\begin{align*}
    \mcl A_{var}^{back} - \mcl A_{var}^{unl} &\leq \frac{(K+1)(2\alpha_0 + 1)}{K\alpha_0^2(K\alpha_0 + 1)}  \\ 
    \frac{\mcl A_{var}^{lab}}{\mcl A_{var}^{unl}}= &  \lp \frac{4(K-1)(\alpha_0 + \varepsilon) (\alpha_0 + 1) + K(K-1)( \alpha_0 + \varepsilon)^2}{K(K+1)\alpha_0^2} \rp \lp  \frac{K^2(\alpha_0 + \varepsilon)^2(K(\alpha_0 + \varepsilon) + 1)}{(K\alpha_0 + \zeta)^2  (K \alpha_0 + \zeta + 1)} \rp \\ 
    &\leq \frac{(K-1)(\alpha_0 + \varepsilon)^4}{(K+1)\alpha_0^2\lp \alpha_0 + \frac{\zeta}{K}\rp^2 } \lp 1 + \frac{4 (\alpha_0 + 1)}{K(\alpha_0 + \varepsilon)} \rp \\
    &= C(\alpha_0, \varepsilon, \zeta, K),
\end{align*}
so that 
\[
    \mcl A_{var}^{lab} - \mcl A_{var}^{unl} \leq \lp C(\alpha_0, \varepsilon, \zeta, K)  - 1 \rp  \mcl A_{var}^{unl}.
\]

Hence we have that the probability of sampling in either $\hat{\mcl X}$ or from one of the already labeled clusters $\mcl X_{i_j}$ will be at most 
\begin{align*}
    \mbb P\lp X_{J+1} \in \cup_{i_k: k \leq  J} \mcl X_{i_k} \cup \hat{\mcl X} \rp &=  \mbb P\lp X_{J+1} \in  \hat{\mcl X} \rp +  \mbb P\lp X_{J+1} \in \cup_{i_k: k \leq J} \mcl X_{i_k}  \rp \\ 
    &\leq  \frac{\int_{\hat{\mcl X}} e^{\lambda V(\alpha(x))} \rho(x)\, dx +  \int_{\cup_{i_k: k \leq J}\mcl X_{i_k}} e^{\lambda V(\alpha(x))} \rho(x)\, dx}{\int_{\mcl X} e^{\lambda V(\alpha(z))} \rho(z)\, dz} \\
    &\leq \frac{\delta}{(1 - \delta) w_{min}} \exp\left(\lambda \lp \mcl A_{var}^{back} - \mcl A_{var}^{unl}\rp\right) \\
    &\qquad + \frac{\lp  1 - w_{min}(1 - \delta)\rp}{(1 - \delta)w_{min}} \exp\left(\lambda \lp \mcl A_{var}^{lab} - \mcl A_{var}^{unl}\rp  \right) \\
    &< \frac{\delta}{(1 - \delta) w_{min}} \exp\left( \frac{\lambda(K+1)(2\alpha_0 + 1)}{K\alpha_0^2 (K\alpha_0 + 1)} \right)   \\
    & \qquad + \frac{ \lp  1 - w_{min}(1 - \delta)\rp}{(1 - \delta)w_{min}} \exp\left(\lambda \lp C(\alpha_0, \varepsilon, \zeta, K) - 1\rp \mcl A_{var}^{unl} \right),
\end{align*}
where the constant $C(\alpha_0, \varepsilon,\zeta, K)$ satisfies
\[
    \frac{\mcl A_{var}^{lab}}{\mcl A_{var}^{unl}} \leq C(\alpha_0, \varepsilon, \zeta, K) := \lp 1 + \frac{4(\alpha_0 + 1)}{K(\alpha_0 + \varepsilon)} \rp \lp \frac{(K-1)(\alpha_0 + \varepsilon)^4}{(K+1)\alpha_0^2 \lp \alpha_0 + \frac{\zeta}{K}\rp^2} \rp  < 1.
\] 
This proves our induction step, and then by iterating this bound (which was constructed to be independent of the particular group of indices $i_j$) we obtain the result.

%%%%%%%%%%%%%%%%%%%%%%%%%%%%%%%%%%%%%%%%%%%%%%%%%%%%%%%%%%%%%%%%%%%%%%%%%%%%%%%%%%%%%%%
%%%%%%%%%%%%%%%%%%% calculation of probability associated with selecting a point in an already labeled cluster used above
%%%%%%%%%%%%%%%%%%%%%%%%%%%%%%%%%%%%%%%%%%%%%%%%%%%%%%%%%%%%%%%%%%%%%%%%%%%%%%%%%%%%%%%
\begin{comment}
\km{Include this calculation in the above expression? The probability associated with selecting a point in an already labeled cluster was bounded by
\begin{align*}
    \int_{\cup_{i_k: k \leq J}\mcl X_{i_k}} e^{\lambda V(\alpha(x))} \rho(x)\, dx &\leq e^{\lambda \mcl A_{var}^{lab}} \int_{\cup_{i_k: k \leq J}\mcl X_{i_k}} e^{\lambda V(\alpha(x))} \rho(x)\, dx \\ 
    &\leq e^{\lambda \mcl A_{var}^{lab}}\lp  1 - \sum_{j=1}^K w_j  \int_{\cup_{i_k: k > J}\mcl X_{i_k}}  \rho_j(x)\, dx  \rp\\ 
    &\leq e^{\lambda \mcl A_{var}^{lab}}\lp  1 - w_{min}\sum_{i_k: k > J} \int_{\mcl X_{i_k}}  \rho_{i_k}(x)\, dx  \rp\\ 
    &\leq e^{\lambda \mcl A_{var}^{lab}}\lp  1 - w_{min}(1 - \delta)\rp.
\end{align*}
}

\end{comment}

\end{proof}

\begin{example}
The most straightforward application of the previous results would be in the setting when we use radial basis functions with support on $B(0,R)$ and the supports of the different mixture components are are separated by a distance greater than $R$ and each cluster has diamater less than $R$. In that context it is immediate to check that the kernel would be a class separator with $\delta , \varepsilon = 0$ and with $\zeta$ determined by the kernel and the maximum diameter of the clusters. By making appropriate choices of $\alpha_0,\lambda$ we can then make the probability of sampling from the unexplored clusters arbitrarily close to one. This type of result is analogous to the cluster exploration results given in \citep{karzand_maximin_2020}.
\end{example}

Of course, the previous example requires rather restrictive assumptions regarding the data components. In the next subsection, we provide details for a more flexible, data-adaptive framework in which we can demonstrate the separation property.

\subsubsection{Laplacian-based class separation} \label{subsec:laplacian-based-separation}

Many of the examples throughout the paper focus on Laplacian-based propagation operators. These kernels are data-adapted, and it is natural to guess that they will therefore be well-adapted to cluster separation. An analog of the type of separation that we consider here has recently been studied in the context of spectral clustering in \citep{trillos2021geometric}. In the interest of clarity, we only consider population-level distributions and mixture models, but their work also considers guarantees for finite samples of such a mixture, and the results we give in this section should extend to that setting as well.

One starting point from \citep{trillos2021geometric} is the concrete description that they give for a mixture model to be ``well-separated'' in an appropriate sense.

\begin{definition}[\citet{trillos2021geometric}]
    Consider a probability density $$\rho(x) = \sum_{k=1}^K w_k \rho_k(x)$$ on $\R^d$ with $w_k \geq 0$ and $\rho_k$ being probability densities. We call such a density a \emph{mixture model}, and each of the $\rho_k$ the mixture components. We then define the following parameters:
    \begin{enumerate}
        \item \textbf{Overlapping.} The overlapping, $\mathcal{S}$, of a mixture model is defined to be
        \[
        \mathcal{S} := \max_{i \neq j} \int \frac{\rho_i \rho_j}{\rho}\,dx,
        \]
        \item \textbf{Coupling.} The coupling, $\mathcal{C}$, of a mixture model is defined to be
        \[
        \mathcal{C} := \max_k \int \left|\frac{\nabla \rho_k}{\rho_k} - \frac{\nabla \rho}{\rho}\right| \rho_k \,dx,
        \]
        \item \textbf{Indivisibility.} The indivisibility, $\Theta$, of the mixture model is defined to be
        \[
        \Theta := \min_k \min_{u \perp 1, \int u^2 \rho_k = 1} \int |\nabla u|^2 \rho_k.
        \]
\end{enumerate}
\end{definition}

In \citet{trillos2021geometric} the authors' definition also applies to probability distributions defined on manifolds. We choose not to state our results in terms of manifolds as it doesn't align with other portions of this work, but all of the results in this subsection would apply in the manifold setting as well. In that work one of the main assumptions (Assumption 7 in \citet{trillos2021geometric}) is that
\begin{equation}\label{eqn:spectral-assumption}
    \begin{cases} &\rho, \rho_k \in C^1 \\
    &\Delta_\rho,\Delta_{\rho_k} \text{  have discrete spectrum } \\
    &L^2(\rho),L^2(\rho_k) \text{  each have an orthogonal eigenbasis.} \end{cases}\tag{S}
\end{equation}
We will take this as a standing assumption throughout this section. 
% \rwm{I added a few sentences here, feel free to improve. -RM}

In the previous definition, the overlapping $\mcl S$ describes the amount to which pairs of the mixtures have coincident densities, relative to the underlying density. In a case where the mixtures have disjoint supports, this parameter will be zero. The indivisibility $\Theta$ is a measure of how difficult it is, in terms of Dirichlet energy, to divide any one of the clusters: this parameter would be large in the case of strongly log-concave mixture components. The coupling parameter $\mcl C$ is somewhat more subtle, and measures a type of relative entropy between $\rho$ and $\rho_k$: again this parameter would be zero for disjoint mixture components.

\citet{trillos2021geometric} then define a \textit{well-clustered mixture model} via the relationships 
\begin{equation} \label{eq:wellcluster-cond} 
    \mcl S \ll 1 \quad \text{and} \quad \frac{\mcl C}{\Theta} \ll 1.
\end{equation}
These conditions, of course, require some knowledge of the underlying distributions, and may be difficult to verify for a particular data set. However, most natural models of well-clustered data would satisfy these assumptions. For example, in the case where we have $K$ clusters with disjoint support we will have $\mcl S$ and $\mcl C$ both equal to zero. They would also hold for Gaussian mixtures which are sufficiently separated (and $\mcl S$ and $\mcl C$ could be quantified in terms of the means and variances). In this sense, we view these conditions as applicable to many different models of separated data components. 

With these definitions in hand, we now state our main exploration result, which guarantees that clusters are explored efficiently: this proposition largely turns out to be a direct consequence of the estimates given by \citet{trillos2021geometric}.

\begin{proposition}\label{prop:heat-class-sep}
    Consider a mixture model with parameters $(\mcl{S},\mcl{C},\Theta)$ which satisfies Assumption \eqref{eqn:spectral-assumption}. Let $\Ker_t(x^*, x)$ be the heat kernel with parameter $t$, namely the solution to the heat equation satisfying $\Ker_t(x^*,x) = e^{t \Delta_\rho} \delta_{x^*}$ with Neumann boundary conditions. Consider parameters $a,b,\sigma$ as in Proposition \ref{prop:low-mode-estimate}. Then, under the assumptions on parameters stated in Proposition \ref{prop:low-mode-estimate}, the kernel $\Ker_t$ will be a $(\delta,\zeta,\varepsilon)$ class separator with
    \[
    \delta = \tilde \delta + a^2, \quad \zeta= \frac{(1-\sin(b))^2}{w_{max}} \cos(2(\sigma + b)) - \xi, \quad \varepsilon = \frac{(1+\sin(b))^2}{w_{min}} \cos(\pi/2 - 2(\sigma + b)) + \xi,
    \]
    where
    \[
    \xi :=    |e^{\lambda_K t}-1| \frac{(1+\sin(b))^2}{w_{min}}  \|\rho\|_\infty + \min_{t_1,s > 0, t_1 + s < t}  e^{\lambda_{K+1} (t-t_1-s)}\|e^{t_1 \Delta}\|_{L^1(\rho) \to L^2(\rho)}\| e^{s \Delta}\|_{L^2(\rho) \to L^\infty(\rho)}
    \]
\end{proposition}
\begin{proof}
We consider a sequence of eigenvalue, eigenfunction pairs $(\lambda_k, e_k)$ of the operator $\Delta_\rho$, normalized in the $\rho$-weighted $L^2$ norm. 
We then define $g(y) = e^{t_1\Delta}\delta_{x^*} - \sum_{k=1}^K e^{\lambda_k t_1} e_k(x^*)\rho(x^*)e_k(y)$, with $t_1$ being a parameter that we choose later. We can interpret $g$ as the projection of the solution to the heat equation onto the ``high'' Fourier modes. We find that
    \begin{equation}\label{eqn:spectral-decomp-heat}
        \Ker_t(x^*, x) = \sum_{k=1}^K e_k(x^*)\rho(x^*)e_k(x) + \sum_{k=1}^K (e^{\lambda_k t}-1) e_k(x^*)\rho(x^*)e_k(x) + e^{s\Delta}\sum_{k=K+1}^\infty e^{\lambda_k (t-t_1-s)} e_k(x) \langle e_k, g \rangle_{\rho}.
    \end{equation}
    We can then bound 
    % \km{For this first inequality, I don't think we have defined $E_k$ yet, or am I not seeing it? Or should it be $L^\infty(\rho)$?} \rwm{I agree that we didn't define that yet. In the estimate below we define it and give bounds in Prop 4. We could just use the inifinity bound, but in Prop 4 we end up throwing away the set where that is big anyway. Maybe here we say something like ``We will only need to estimate these terms in regions where the eigenfunctions $e_k$ are not too large: in terms that are quantified in Proposition 4 we only estimate quantities on a family of disjoint sets $E_k$.'' }
    \begin{equation*}
\left| \sum_{k=1}^K (e^{\lambda_k t}-1) e_k(x)e_k(x^*)\rho(x^*) \right| \leq |e^{\lambda_K t}-1| \max_{k=1 \dots K} \|e_k\|_{L^\infty(E_k)}^2\|\rho\|_\infty,
    \end{equation*}
    where we will only need to estimate this term in regions (i.e., a family of disjoint sets $E_k$ that are defined in Proposition \ref{prop:low-mode-estimate}, for the purposes of proving $(\delta,\zeta,\e)$ class separation) where the eigenfunctions $e_k$ are not too large. We can also bound
    \begin{align*}
    \left| e^{s \Delta} \sum_{k=K+1}^\infty e^{\lambda_k (t-t_1)} e_k(x) \langle e_k, g \rangle_\rho \right| &\leq \| e^{s \Delta}\|_{L^2(\rho) \to L^\infty(\rho)} \left\| \sum_{k=K+1}^\infty e^{\lambda_k (t-t_1)} e_k(x) \langle e_k, g \rangle \right\|_{L^2(\rho)} \\
    &\leq e^{\lambda_{K+1} (t-t_1)} \|g\|_{L^2(\rho)}\| e^{s \Delta}\|_{L^2(\rho) \to L^\infty(\rho)} \\
    &\leq e^{\lambda_{K+1} (t-t_1-s)}\|e^{t_1 \Delta}\|_{L^1(\rho) \to L^2(\rho)}\| e^{s \Delta}\|_{L^2(\rho) \to L^\infty(\rho)}.
    \end{align*}
   % \km{Why can we say that $\sum_{k=K+1}^\infty e_k^2(x) \leq 1?$ If I've followed the second inequality here correctly, I get
   % \begin{align*}
   %  \left\| \sum_{k=K+1}^\infty e^{\lambda_k (t-t_1)} e_k(x) \langle e_k, g \rangle_\rho \right\|_{L^2(\rho)} &\leq \|g\|_{L^2(\rho)} \left\| \sum_{k=K+1}^\infty e^{\lambda_k (t-t_1)} e_k(x) e_k \right\|_{L^2(\rho)} \\
   %  &\leq e^{\lambda_{K+1} (t-t_1)} \|g\|_{L^2(\rho)}  \lp \sum_{k=K+1}^\infty e_k^2(x) \rp^{1/2}
   % \end{align*}}
   % \rwm{Here you can use Hilbert space identities (Parseval identity I think?) because the $e_k$ are orthogonal (and because we are estimating in $L^2$). This was the reason to pay the price for the $e^{s\Delta}$, because otherwise one is trying to do an infinity estimate and Parseval doesn't help you at all.} 
   Hence we define
   \[
\xi := \min_{t_1,s > 0, t_1 + s < t}    |e^{\lambda_K t}-1| \max_{k=1 \dots K} \|e_k\|_{L^\infty(E_k)}^2\|\rho\|_\infty + e^{\lambda_{K+1} (t-t_1-s)}\|e^{t_1 \Delta}\|_{L^1(\rho) \to L^2(\rho)}\| e^{s \Delta}\|_{L^2(\rho) \to L^\infty(\rho)}.
   \]
   In turn, we only need to estimate the first term in Equation \eqref{eqn:spectral-decomp-heat}. This is considered in Proposition \ref{prop:low-mode-estimate}. That Proposition also provides an upper bound upon $\|e_k\|_{L^\infty(E_k)}^2$, which completes the proof.
\end{proof}

The quantity $\xi$ represents all of the ``high Fourier mode'' effects encoded by the heat kernel. In the case of well-separated mixtures, one would expect that $\lambda_K \sim 0$ and $\lambda_{K+1}$ would be much larger, and hence one would anticipate that it whould be possible to choose time parameters so that $\xi$ is small. This ``spectral gap'' is quantified in the following proposition from \citep{trillos2021geometric}.
\begin{proposition}[Propositions 27 and 38, \citet{trillos2021geometric}]
    Consider a mixture model with parameters $(\mcl{S},\mcl{C},\Theta)$ which satisfies Assumption \eqref{eqn:spectral-assumption}, and suppose that 
    % $K \min(\mcl S,\mcl S^{1/2}) < 1$
    $K \mcl S < 1$. Then the following bounds hold:
    \begin{align*}
        &\lp \sqrt{\Theta(1-K\mcl S)} - \frac{\sqrt{\mcl C K \mcl S}}{1 - \mcl S}\rp^2 \le \lambda_{K+1} \\
        &\lambda_K \leq \frac{K \mcl C}{1-K \mcl S^{1/2}}.
    \end{align*}
\end{proposition}
In turn, the main consideration in bounding $\xi$ relates to the choice of $t$ and the operator norm bounds on the heat kernel. The types of bounds that are implicitly assumed in the definition of $\xi$ are called \emph{ultracontractivity} estimates for the heat kernel, and are known to be finite under various classes of assumptions: a standard reference on the subject is \citep{daviesbook}. In particular, these bounds will hold for densities which are compactly supported and do not degenerate, or for mixtures whose densities are smooth and (asymptotically) log-concave. While for empirical data it will be challenging to verify bounds on $\xi$, there are many important models, including those used previously in the active learning literature \citep{karzand_maximin_2020, cloninger_cautious_2021}, where it is possible to bound $\xi$ (e.g., Gaussian mixtures, uniform distributions on sets with very small overlap). Furthermore, if we restrict our ``hypothesis class'' to be densities for which the heat kernel obeys desired bounds (for example by assuming uniform log-concavity on the mixture components), then we can completely control all of the remaining terms in Proposition \ref{prop:heat-class-sep} using the parameters $\mcl S, \mcl C, \Theta$. 

Now, we turn to a key estimate for the heat kernel, regarding the effect of the low Fourier modes after neglecting decay. For convenience, and following the notation in \citep{trillos2021geometric}, we define the spectral embedding $F$ to be the mapping from $\R^d$ to $\R^K$ given by
\[
F(x) := (e_i(x))_{i=1}^K.
\]

\begin{proposition}\label{prop:low-mode-estimate}
    Consider a mixture model with parameters $(\mcl{S},\mcl{C},\Theta)$. We define the parameters 
    \[
 \Xi := K\frac{(\Lambda -\sqrt{\mcl S})^2 }{4} + 4K^{3/2} \left( \frac{1}{\sqrt{1-K\Lambda}}-1\right),
    \]
    with 
    \[
    \Lambda := 4\left(\sqrt{\frac{\Theta(1-K\mcl S)}{\mcl C} - \frac{\sqrt{K\mcl S}}{(1-\mcl S)}} \right)^{-1} + \sqrt{\mcl S}.
    \]
    We assume that

    \begin{align}
        a,b > 0 \\ \label{eqn:param-assum-first}
        \frac{a \sin(b) }{\sqrt{w_{max}}} \geq \sqrt{\Xi} \\
        \mcl S  < \frac{w_{min}(1-\cos^2(\sigma))}{w_{max}\cos^2(\sigma)K^2}\\
        \Lambda - \sqrt{\mcl S} > 0 \\
        \Lambda K < 1 \\
        b + \sigma < \frac{\pi}{4} \label{eqn:param-assum-last}
    \end{align}
    and we let 
    \[
    \tilde \delta := \frac{w_{max} K^2 \mcl S}{w_{min}^2 (1-\cos^2(\sigma))} + \frac{1-\cos^2(\sigma)}{w_{min}} 
    \]

    Then there exists disjoint sets $E_k$ so that $\rho_k(E_k) > 1-(\tilde \delta + a^2)$, which satisfy the inequalities, for $x \in E_j, x^* \in E_k$
    \begin{align}
        \left|\sum_{k=1}^K e_k(x) e_k(x^*)\right| &= |\langle F(x),F(x^*) \rangle| \leq \frac{(1+\sin(b))^2}{w_{min}} \cos\lp \frac{\pi}{2} - 2(\sigma + b)\rp  \qquad \text{ if } j \neq k \\
        \left|\sum_{k=1}^K e_k(x) e_k(x^*)\right| &= |\langle F(x),F(x^*) \rangle| \geq \frac{(1-\sin(b))^2}{w_{max}} \cos(2(\sigma + b)) \qquad \text{ if } j = k \\
        \left|\sum_{k=1}^K e_k(x) e_k(x^*)\right| &= |\langle F(x),F(x^*) \rangle| \leq \frac{(1+\sin(b))^2}{w_{min}}  \qquad \text{ if } j = k
    \end{align}

\end{proposition}

\begin{proof}
    The ingredients for this proof are all present in the proof of Theorem 10 in \citet{trillos2021geometric}. However, their statements are all made in terms of the overall probability of points which are in the support of a set of approximately orthogonal cones: in other words, they estimate probabilities in terms of $\rho$. This was natural in their unsupervised setting, whereas in ours we are closer to the supervised setting and care about the labels associated with each mixture component; that is, we want to estimate the size of sets in terms of the $\rho_k$. We simply sketch how the estimates we use can be obtained by small modifications of their proofs.
    
    In one of the first main steps in their proof, they identify the functions $\tilde q_k := \sqrt{\frac{\rho_k}{\rho}}$ as a system of approximate eigenfunctions, which are almost orthogonal in a way that can be quantified by $\mcl S$. Specifically, in the proof of Proposition 29, given an angle $0 < \sigma< \frac{\pi}{4}$ they define the disjoint sets
    \[
A_k := \left\{ x : \rho_k(x) > \cos^2(\sigma) \sum_{j=1}^K \rho_j(x) \right\}
    \]
    They then prove that $$\rho\left(\bigcup_{k=1}^K A_k\right) \geq 1 - \frac{w_{max} K^2 \mcl S}{w_{min} (1-\cos^2(\sigma))}.$$ We now extend their estimate to apply to the conditional probabilities $\rho_k$, as needed in estimating the class separator parameters in Definition \ref{def:class-separator}. We first notice that the previous inequality immediately implies that $$\rho_k\left( \left(\bigcup_{j=1}^K A_j\right)^c\right) \leq \frac{w_{max} K^2 \mcl S}{w_{min}^2 (1-\cos^2(\sigma))}.$$ On the other hand, by the definition of the $A_k$ we have that for $j \neq k$
    \[
    w_j \rho_k(A_j) \leq (1-\cos^2(\sigma)) w_j \rho_j(A_j) \leq (1-\cos^2(\sigma)) \rho(A_j),
    \]
    which after summing over $j \neq k$ gives
    \[
    w_{min} \rho_k\left(\bigcup_{j \neq k} A_j\right) \leq (1-\cos^2(\sigma)).
    \]
    This then implies that
    \[
    \rho_k(A_k^c) \leq \frac{w_{max} K^2 \mcl S}{w_{min}^2 (1-\cos^2(\sigma))} + \frac{1-\cos^2(\sigma)}{w_{min}} =: \tilde \delta.
    \]

     In terms of embeddings, they define the mapping $F^Q := (\tilde q_k)_{k=1}^K$. The proof of Theorem 10 in their paper shows that there exists an orthonormal matrix $O$ so that
    \[
   W_2^2(O F_\sharp \rho, F_\sharp^Q \rho) \leq K\frac{(\Lambda -\sqrt{\mcl S})^2 }{4} + 4K^{3/2} \left( \frac{1}{\sqrt{1-K\Lambda}}-1\right) =: \Xi,
    \]
    with 
    \[
    \Lambda := 4\left(\sqrt{\frac{\Theta(1-K\mcl S)}{\mcl C} - \frac{\sqrt{K\mcl S}}{(1-\mcl S)}} \right)^{-1} + \sqrt{\mcl S}
    \]
    Here we have used the notation $F_\sharp \rho$ to denote the standard push-forward measure, which is defined to be the measure so that $F_\sharp \rho( B) := \rho(F^{-1}(B))$ for Borel sets $B$.
    
    Using the definition of the Wasserstein distance, we can also infer that $$W_2^2(O F_\sharp \rho_k, F_\sharp^Q \rho_k) \leq  \frac{\Xi}{w_{min}}.$$

    We also notice that
    \[
    F^Q(A_k) \subset \left\{ z : \frac{z_k}{|z|} > \cos(\sigma),  \frac{1}{\sqrt{w_{max}}} \leq |z| \leq \frac{1}{\sqrt{w_{min}}}  \right\}.
    \]
    In words, the embedding $F^Q$ separates the $A_k$ into orthogonal cones with angle $\sigma$ and places the mass in $A_k$ at least distance $\frac{1}{\sqrt{w_{max}}}$ from the origin.

    Finally, in Proposition 22 of their paper they show stability bounds of cones under the Wasserstein distance. In particular, given a vector $v \neq 0$ and an angle parameter $0 < \sigma < \frac{\pi}{4}$, if we let the set $C$ be a cone defined by $$C := \left\{ x : \frac{\langle v,x \rangle}{|v||x|} \geq \cos(\sigma) \right\} $$ then we can define the set $C_r := C \cap (B(0,r))^c$.
    % \km{Notation: ball of radius $r$ centered at a point -- we just need to sync this up between the proofs}\rwm{I switched the notation you had used in the consistency section to match what is here, and I added a comment about notation for the ball in the notation section.}. 
    By following the proof of Proposition 22 by \citet{trillos2021geometric}, as applied to one cone as opposed to an orthogonal system of cones, one has that, for a probability measure $\mu_1$ with $\mu_1(C_r)>1-\delta$, then for any $a,b>0$ satisfying
    \[
    ra \sin(b) \geq W_2(\mu_1,\mu_2)
    \]
    and assuming that $0 < \sigma + b < \pi/4$, 
    we will also have that $\mu_2(\tilde C \cap (B(0,\tilde r))^c) > 1-\hat \delta$, with $\tilde C$ being the cone centered on $v$ with angle smaller than $\sigma + b$, $\tilde r = r(1-\sin(b))$ and $\hat \delta = \delta + a^2$. We furthermore note that their proof could be extended so that if we have $\mu_1(C \cap (B(0,R) \setminus B(0,r)))$ with $R>r$ then we will have $\mu_2(\tilde C \cap (B(0,\tilde R) \setminus B(0,\tilde r)))$ with $\tilde R = R(1+\sin(b))$. 
    % \km{Maybe switch the $\tilde{\delta}$ in this section to avoid collision with the specific instance of $\tilde{\delta}$ defined previously?} \rwm{I agree. I switched to $\hat \delta$ in this statement. I think that below it should be tilde delta though, as we're actually using the estimate we got on the previous page.}

    Putting these facts together, we then have that if our parameters satisfy equations \eqref{eqn:param-assum-first}-\eqref{eqn:param-assum-last} and
  if we define a cone centered on the $k$-th coordinate vector with angle $\sigma +b$ to be $C_k$, then for $\tilde r = \frac{1-\sin(b)}{\sqrt{w_{max}}}$ and $\tilde R = \frac{1+\sin(b)}{\sqrt{w_{min}}}$ and letting $\tilde C_k = C_k \cap (B(0,\tilde R) \setminus B(0,\tilde r))$ we have
    \[
    OF_\sharp \rho_k(\tilde C_k) \geq 1-(\tilde \delta + a^2).
    \]

    We notice that if $z \in \tilde C_k$ and $z' \in \tilde C_j$ with $j \neq k$ then we have that $$\langle z, z' \rangle \leq \tilde R^2 \cos(\pi/2 - 2(\sigma + b)).$$ On the other hand, if $j = k$ we have that $$\tilde r^2 \cos(2(\sigma + b)) \leq \langle z,z' \rangle \leq \tilde R^2.$$ This then concludes the proof.

\end{proof}

The parameters in the previous propositions are somewhat involved, and for convenience we provide the following simplified result.
\begin{corollary}
    Suppose that $\mcl C \leq \kappa,  \mcl S \leq \kappa$ and that $\Theta \geq \sqrt{\kappa}$. Let $\|e^{ \Delta}\|_{L^1(\rho) \to L^\infty(\rho)} \leq \sqrt{M}$. Then for $\kappa$ sufficiently small and for $t = -\log(\kappa) \kappa^{1/2}$ we have that the heat kernel $\Ker_t$ is a $(\delta,\zeta,\e)$ class separator with $\delta \leq C \kappa^{1/4}, \zeta \geq w_{max}^{-1}-C(|\log(\kappa)|\kappa^{1/2} + M\kappa)$ and $\e \leq C(\kappa^{1/4}+M\kappa)$, where $C$ is a constant that depends only upon $K, w$ and $\|\rho\|_\infty$.
\end{corollary}
\begin{proof}
    This simply amounts to choosing the parameters $a,b,\sigma$ appropriately, and then bounding all of the terms. Throughout this proof, we let $C$ be a constant that varies line by line, and may depend upon $K,w_{min}, w_{max}, \|\rho\|_\infty$, but does not depend upon the other parameters. First, we notice that
    \[
\left(\sqrt{\frac{\Theta(1-K\mcl S)}{\mcl C} - \frac{\sqrt{K\mcl S}}{(1-\mcl S)}} \right)^{-1} \leq C \kappa^{1/4}
    \]
    In turn
    \[
    \Xi \leq C\kappa^{1/2}
    \]
    Now let $a = \kappa^{1/4}$ and $b = \kappa^{1/4}$, and $\sigma = \kappa^{1/4}$. Then $\tilde \delta \leq C \kappa^{1/2}$.
    
    We also note that using our assumptions we have that $|\lambda_{K}| \leq \kappa$ and $|\lambda_{K+1}| \geq C\kappa^{1/2}$. Therefore by setting $t = -\log(\kappa)\kappa^{-1/2}$ we have that $|e^{\lambda_K t }- 1| \leq C|\log(\kappa)|\kappa^{1/2}$ and $|e^{\lambda_{K+1}(t-2)}| \leq C\kappa$. In turn this implies that $\xi \leq C (|\log(\kappa)|\kappa^{1/2} + M\kappa)$
    Putting this all together, then gives that our kernel is a $(\delta,\zeta,\e)$ class separator with $\delta \leq C\kappa^{1/2}$, $\zeta \geq w_{max}^{-1}-C(|\log(\kappa)|\kappa^{1/2} + M\kappa)$ and $\e \leq C(\kappa^{1/4}+M\kappa)$.
\end{proof}

\subsection{Derivation of asymptotic exploitation } \label{subsec:exploit-analysis}

In this section, we identify a formal continuum limit associated with Dirichlet Learning, and identify its steady states. For simplicity, we focus our attention on the case where $\mcl X$ is a bounded open set in $\mbb R^d$, and that the underlying features are associated with a smooth and bounded density $\rho(x) = \sum_{k=1}^K w_k \rho_k(x)$ supported on $\mcl X$ with smooth and bounded class-conditional densities $\rho_k(x)$ for $k \in [K]$. Recalling \eqref{eq:dirichlet-variance-af}, we furthermore define $V(\alpha(x)) :=  \mcl A_{var}(x)$ and assume that our acquisition function samples proportional to $e^{\lambda V(\alpha(x))}$ (i.e., the \textit{Proportional Sampling} policy introduced in \ref{sec:query-point-selection}). Furthermore, we will assume that the value of $\lambda$ is allowed to vary during the sampling process.

In this context, we can view the expected change in the pseudo-labels from the $\ell$-th to $(\ell+1)$-th observation by the relation
\[
    \mbb E_{X_{\ell + 1}, Y_{\ell + 1}}[\alpha^{\ell+1}_{k}(x) - \alpha^{\ell}_k(x) | (X_1, Y_1), \ldots, (X_\ell, Y_\ell)] =  \int_{\mcl X}  \frac{w_k\rho_k(z)}{\rho(z)}\Ker(z,x)\frac{e^{\lambda(\ell) V(\alpha^\ell(z))}}{\int_{\mcl X} e^{ \lambda(\ell) V(\alpha^\ell(\tilde{z}))}\,d\tilde{z}} \,dz,
\]
where $\alpha_k^\ell(x) \geq 0$ is the $k$-th entry of the concentration parameter vector $\alpha^\ell(x) \in \mbb R_+^K$ at the $\ell$-th observation. 
% Furthermore, by $\sim$\rwm{Did we use this here?} we mean that the left-hand side is a random variable distributed according to the right-hand side.

By assuming that we take a large number of discrete steps in one unit of ``continuous time'', the law of large numbers then formally leads to the evolution equation
\begin{align}\label{eq:Int-Diff-alpha}
\partial_t \alpha(x, t) &= \int_{\mcl X}  \frac{w_k\rho_k(z)}{\rho(z)}\Ker(z,x) \frac{e^{\lambda(t) V(\alpha(z,t))}}{\int_{\mcl X} e^{\lambda(t) V(\alpha(\tilde{z}, t))}\,d\tilde{z}} \,dz \nonumber \\
    &=:  \int_{\mcl X}  \eta^K(z) \Ker(z,x) q(z,t) \,dz,
\end{align}
where we have introduced the notation $\eta^K(x) \in \mbb R_+^K$ as the vector of weighted class-conditional densities concatenated together and $q(x,t)$ as the sampling distribution according to Dirichlet variance.
This evolution equation provides a convenient means of understanding the effect of newly observed labels on our future acquisitions in the limit of a large number of observations. Furthermore, it is natural to consider kernels that are increasingly localized when analyzing the large-sample limits of kernel methods \citep{devroye1996probtheory}.
With that in mind, it is instructive to consider a limiting case where $\Ker(z,x)$ is given by the Dirac mass $\delta_{z}(x)$. In that case, we can write \eqref{eq:Int-Diff-alpha} as simply
\begin{equation} \label{eq:ode-alpha}
    \partial_t \alpha(x, t) = q(x,t) \eta^K(x),
\end{equation}
from which we see that the trajectories of the $\alpha(x,t) \in \mbb R_+^K$ lie along the corresponding rays $R_x = \{a \eta^K(x) : a \geq 0\}$, whose speed is determined by the sampling $q(x,t)$. 

We can then write
\begin{equation} \label{eq:alpha-upto-t}
    \alpha(x,t) = \lp \int_0^t q(x,s)\,ds\rp \eta^K(x), 
\end{equation}
from which we can define the amount of sampling that has occurred at $x$ up to time $t \geq 0$ as
\begin{align*}
    \beta_x(t) := \sum_{k=1}^K \alpha_k(x,t) = \int_0^t q(x,s)\,ds,
\end{align*}
since the entries of $\eta^K(x)$ sum to $1$. By appealing to \eqref{eq:alpha-upto-t}, we can now write the variance at $x$ as
\begin{align*}
    V(\alpha(x,t)) &=  \frac{\sum_{i \neq j} \alpha_i(x,t) \alpha_j(x,t)}{\beta_x^2(t)(\beta_x(t) + 1)} = \frac{\sum_{i \neq j} \eta_i(x) \eta_j(x)}{\beta_x(t) + 1} = \frac{\sum_{k=1}^K \eta_k(x)(1 - \eta_k(x))}{ \beta_x(t) + 1}\\
    &=: \frac{G(x; \eta)}{\beta_x(t) + 1},
\end{align*}
where this function 
\begin{equation} \label{eqn:G-x-eta}
    G(x;\eta) = \sum_{k=1}^K \eta_k(x)(1 - \eta_k(x)) \in \left[0, 1 -\frac{1}{K}\right]
\end{equation}
is a measure of the population level uncertainty as in in Example \ref{example:2types-unc}. 

We now assume that the distribution of points sampled by the acquisition function, namely $q(x,t) = \frac{e^{\lambda(t)  V(\alpha(x,t))}}{\int_{\mcl X} e^{\lambda(t)  V(\alpha(z,t))} \,dz}$ is asymptotically convergent as $t \to \infty$. The question of whether this actually always occurs for solutions of the evolution equation \eqref{eq:ode-alpha} is not simple, but as $q$ will be weakly compact in $t$ it is natural to guess that it will approach some steady state $\bar q(x)$. If one were to use this time-independent distribution throughout the entire stochastic process, we would obtain the simple sampling relationship
\[
    \bar{\beta}_x(t) = \int_0^t \bar{q}(x) = t \bar{q}(x),
\]
with corresponding pseudo-label densities $\bar{\alpha}(x,t) = t \bar{q}(x) \eta^K(x)$. 

The assumption that $q(x,t) \to \bar{q}(x)$ as $t \to \infty$ suggests then that $\partial_t \alpha(x,t) \approx \partial_t \bar{\alpha}(x,t)$ for large enough $t$. From this, we can approximate $\beta_x(t) \approx \bar \beta_x(t) = t \bar{q}(x)$ for $t$ large. Thus, we expect that
\begin{equation} \label{eqn:barq-exp}
    q(x,t) \propto \exp\lp \frac{\lambda(t) G(x;\eta)}{t \bar{q}(x) + 1} \rp \approx \exp\lp \frac{\lambda(t) G(x;\eta)}{t \bar{q}(x)} \rp 
\end{equation}
as $t \to \infty$. This gives a non-linear algebraic relation that allows us to narrow our search for possible limiting sampling distributions. At this point, we consider two cases for the scaling of $\lambda(t)$: when (i) $\lambda(t) = \lambda_0 t^p = o(t)$ for $p < 1$  and (ii) $\lambda(t) = \lambda_0 t$ for $\lambda_0 > 0$. 

\textbf{Case 1:} When $\lambda(t) = \lambda_0 t^p$ for $p < 1$, then the the exponent in the last expression of \eqref{eqn:barq-exp} decreases like $t^{-(1-p)}$ for every $x \in \mcl X$ and the result is that the limiting distribution is uniform over the domain, $q(x,t) \to \bar q(x) = \lp\int_{\mcl X}\,dz\rp^{-1}$. Roughly, we have that the equivalent expression yields
\[
    q(x,t) = \lp \int_{\mcl X} \exp\left\{\frac{\lambda_0}{t^{1-p}} \lp \frac{G(z;\eta)}{\bar q(z)} - \frac{G(x;\eta)}{\bar q(x)} \rp  \right\} \,dz \rp^{-1}  \to \lp \int_{\mcl X}\,dz\rp^{-1}
\]
as $t \to \infty$ implying 
% \rwm{"implying" would sound stronger here: we've already made it clear that we are taking some formal steps, and I don't know that we need to hedge more.} 
that the limiting distribution $\bar{q}(x)$ is uniform over $\mcl X$ when $\lambda(t) = o(t)$. 
We can interpret this as asymptotic \textit{exploration} that samples throughout the domain irrespective of the corresponding marginal density $\rho(x)$ in contrast to passive sampling via said marginal distribution. 

\textbf{Case 2:} Now consider the case when $\lambda(t) = \lambda_0 t$. The expression \eqref{eqn:barq-exp} the becomes autonomous in the large $t$ limit, and suggests that the limiting distribution in this case satisfies
\begin{equation} \label{eqn:qbar-lambda0}
    \bar q(x) \propto \exp\lp \frac{\lambda_0 G(x;\eta)}{\bar q(x)} \rp. 
\end{equation}
In the next section, we provide numerical simulations that suggest that for moderate values of $\lambda_0$ this $\bar q(x)$ emphasizes regions of greater population-level uncertainty; namely, where multiple conditional probabilities ($\eta_i(x) = \mbb P(Y = i| X = x)$) are simultaneously large. In other words, this limiting distribution asymptotically \textit{exploits} in regions around the true decision boundaries. For $\lambda_0 \ll 1$, however, we expect \eqref{eqn:qbar-lambda0} to imply that $\bar q(x)$ behaves like the uniform distribution on $\mcl X$.  

We notice that the two cases here both give asymptotic sampling densities which are independent of $\rho$. In large sample regimes, we posit that independence on $\rho$ ought to be viewed positively, meaning that we will explore (Case 1), or exploit (Case 2), in a fashion that is unbiased by the density $\rho$.

% Qualitatively, this indicates that much more time will be spent sampling in regions where the class-conditional densities are of comparable size. The expression \eqref{eqn:G-x-eta} is a multivariate generalization of this expression, and can also be interpreted in terms of covariances of the multinomial distribution. In any dimension, it is easy to check that $\bar q$ will be zero in regions where only one $\rho_i$ is positive and will be maximized when all of the conditional densities are equal.

Of course, several steps in this derivation are formal. The passage to ``continuous time'' is likely justifiable using stochastic approximation techniques under appropriate scalings, but the details would be significant. The use of non-Dirac kernels certainly will increase the number of possible steady-state sampling densities, but given a density with compact support it also seems plausible that one could show that sampling density steady states need to be ``close'' to the one given by the Dirac mass kernel. Finally, the existence of steady-state densities is not completely obvious, especially in light of the $\alpha$ large approximations that we made. On the other hand, we do anticipate that it should be possible to demonstrate that sampling densities associated with the evolution equation \eqref{eq:ode-alpha} do converge to the type of steady states we have formally described here.

Although the work here is informal, we think that it gives valuable insight into the behavior observed in Figure \ref{fig:unc-demo}, where DiAL effectively transitions from exploratory to exploitative behavior. We leave more rigorous proof of the types of formulas given in this section for future work. We now present some numerics to illustrate the behavior of \eqref{eqn:barq-exp} in the two cases of $\lambda(t)$ that we've considered here.

\begin{comment}
The assumed convergence of $q(x,t) \to \bar q(x)$ suggests also that $\partial_t q(x,t) \to 0$, for which we compute
\begin{align*}
    \partial_t q(x,t) &\approx \frac{e^{\lambda(t) G(x;\eta) / t \bar q(x)}}{\int_{\mcl X} e^{\lambda(t) G(z;\eta) / t \bar q(z)} \,dz} \lp \frac{\lambda(t)}{t} \rp' \lp \frac{G(x;\eta)}{\bar q(x)} - \int_{\mcl X}  \frac{G(z;\eta)}{\bar q(z)} \frac{e^{\lambda(t) G(x;\eta) / t \bar q(x)}}{\int_{\mcl X} e^{\lambda(t) G(\xi;\eta) / t \bar q(\xi)} \,d\xi} \,dz\rp  \\
    &= q(x,t) \lp \frac{t\lambda'(t) - \lambda(t)}{t^2}\rp \lp \frac{G(x;\eta)}{\bar q(x)} - \int_{\mcl X}  \frac{G(z;\eta)}{\bar q(z)} q(z,t)\,dz\rp \\
    &=\lp \frac{t\lambda'(t) - \lambda(t)}{t^2}\rp \lp G(x;\eta) \frac{q(x,t)}{\bar q(x)} - q(x,t) \int_{\mcl X}  G(z;\eta) \frac{q(z,t)}{\bar q(z)}\,dz\rp.
\end{align*}
Thus, we see that $\partial_t q(x,t) \to 0$ with $q(x,t) \to \bar q(x)$ as $t \to \infty$  suggests that
\begin{equation} \label{eq:qbar}
    \bar q(x) = \frac{G(x;\eta)}{\int_{\mcl X} G(z;\eta)\,dz}.
\end{equation}
We also see that the scaling of $\lambda(t)$ comes into the computation of $\partial_t q(x,t)$; namely, if $\lambda(t)$ does not grow fast enough with $t$, then this analysis would suggest that it is possible for $\partial_t q(x,t) \to 0$ trivially and $q(x,t)$ converges to something other than \eqref{eq:qbar}. We provide some numerical studies in Section \ref{subsubsec:num-demo-uncertainty} in this direction. 
    
\end{comment}

\subsubsection{Numerical Demonstration} \label{subsubsec:num-demo-uncertainty}

We return to the mixture of Gaussians setup of Example \ref{example:2types-unc} and numerically solve the evolution equation \eqref{eq:ode-alpha} for $\alpha(x,t)$ up to $t = 10^{10}$ in order to simulate convergence of $q(x,t) \to \bar q(x)$ for various scalings of $\lambda(t)$. In panel (b) of Figure \ref{fig:unc-demo-lambda}, we plot the corresponding distributions $\bar q(x)$ for algebraic scalings of $\lambda(t) = t^p$ for some values of $p \in [0,1]$; note that the case of $p = 0$ corresponds to $\lambda(t) = \lambda_0 = 1$. In panel (c), we plot $\bar q(x)$ for  $\lambda(t) = \lambda_0 t$ for a range of values of $\lambda_0$.

As suggested by our analysis above, we see in panel (b) that for $\lambda(t)$ constant (and for sufficiently small powers $p < 1$), the limiting distribution $\bar q(x)$ corresponds to uniform sampling over the whole domain. We can interpret this as asymptotic \textit{exploration} that samples throughout the domain irrespective of the corresponding marginal density $\rho(x)$ in contrast to passive sampling via said marginal distribution. As $p \to 1$, however, we observe that $\bar q(x)$ focuses on the regions of the domain that lie near the true decision boundary between the classes; we can interpret this as asymptotic \textit{exploitation}. Similarly, in panel (c) of Figure \ref{fig:unc-demo-lambda}, we observe the influence of $\lambda_0$ on $\bar q(x)$, including the expected behavior as $\lambda_0 \ll 1$ corresponding to the uniform distribution.  
% recalling that the limiting distribution has a dependence on $\lambda_0$ (i.e., $\bar q(x) = \bar q(x; \lambda_0)$). 

These numerical experiments simply demonstrate the intimately coupled nature of the scaling of $\lambda(t)$ and the limiting distribution $q(x,t) \to \bar q(x)$ of this stochastic process that represents Dirichlet Active Learning (DiAL). In addition to further work on analyzing the steady states of the stochastic process associated with DiAL, we suggest that an in-depth numerical exploration into this process is warranted pursuant our findings here.

\begin{figure} 
    \centering
    \begin{subfigure}{0.4\textwidth}
        \centering
        \includegraphics[width=\textwidth]{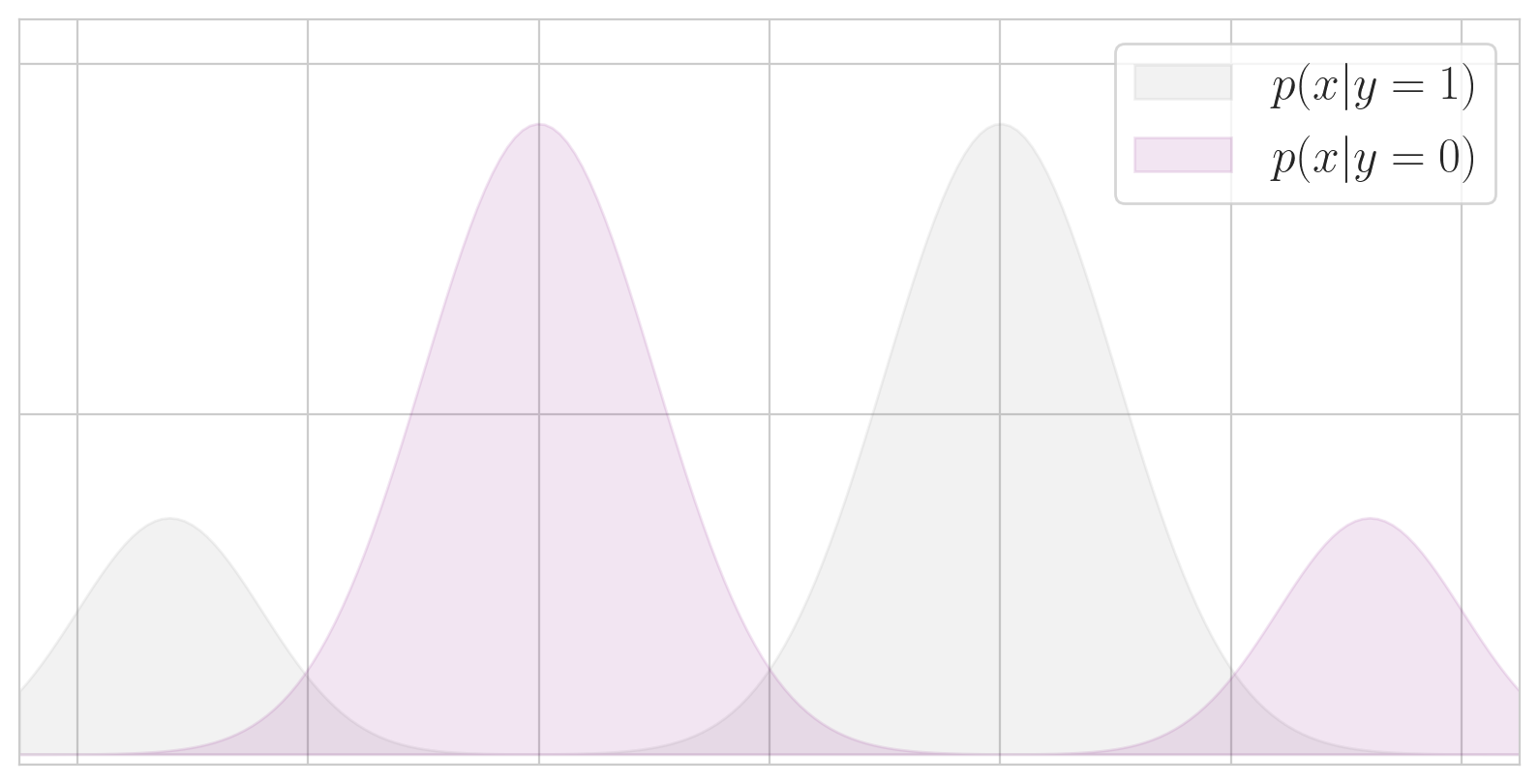}
            \caption{Setup of 2 classes}
    \end{subfigure}
    \\
    \begin{subfigure}{0.49\textwidth}
        \centering
        \includegraphics[width=\textwidth]{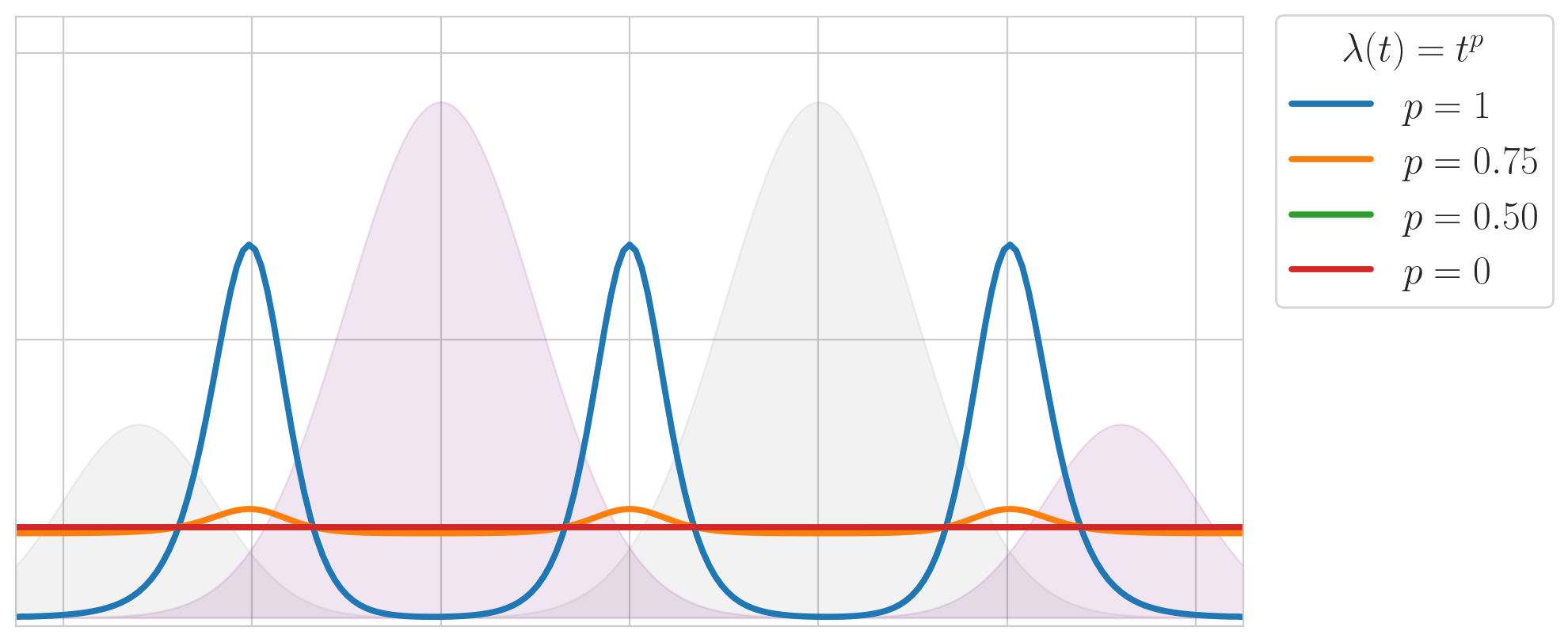}
        \caption{$\bar q(x)$ when $\lambda(t) = t^p$}
    \end{subfigure}
    % \hspace{-3em}
    \begin{subfigure}{0.49\textwidth}
        \centering
        \includegraphics[width=\textwidth]{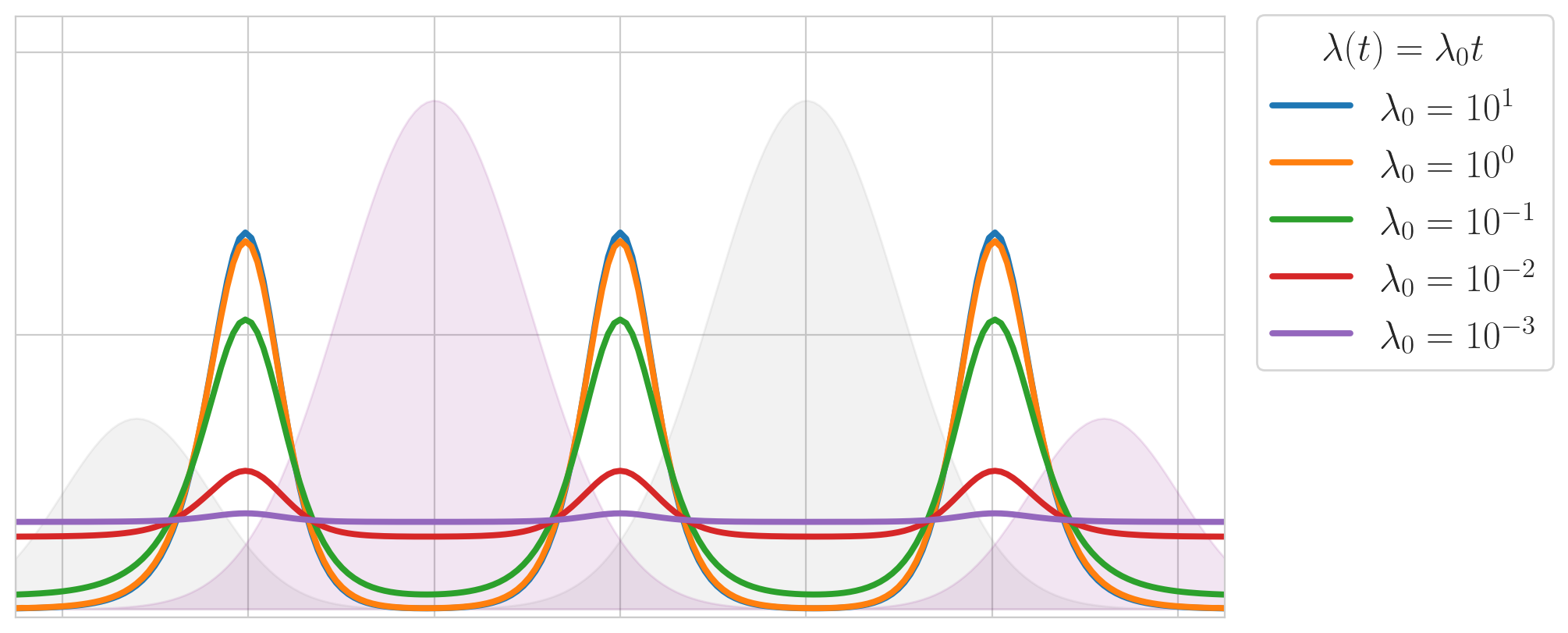}
        \caption{$\bar q(x)$ when $\lambda(t) = \lambda_0 t$}
    \end{subfigure}
    \caption{Numerical demonstration of the effect of $\lambda(t)$ on the steady state distribution $\bar q(x)$ in a 1D setting. With class-conditional distributions $p(x | y= 0), p(x | y=1)$ show respectively in gray and purple panel (a), we numerically solve \eqref{eq:ode-alpha} up to $t = 10^{10}$ for different scalings of $\lambda(t)$ to simulate convergence of $q(x,t) \to \bar q (x)$. Panels (b) and (c) suggest that the limiting behavior when $\lambda(t) = \lambda_0 t$ for $\lambda_0 \gg 0$ corresponds to \textit{asymptotic exploitation}, whereas when $\lambda(t) = \lambda_0 t^p$ for $p < 1$ or $\lambda_0 \ll 1$ corresponds to \textit{asymptotic exploration}.
    % Note that for powers of $t^p$ for $p < 1$, $\bar q(x)$ corresponds to a (nearly) uniform distribution over the domain (panel (b)). Similarly, in the case that $\lambda(t) = \lambda_0 t$ (panel (c)) $\bar q(x)$ our results suggest that $\lambda_0$ chosen too small 
    }
    \label{fig:unc-demo-lambda}
\end{figure}

\subsection{Consistency of Dirichlet Learning classifier}\label{sec:dir-learning-consistency}

We now turn to establishing the asymptotic consistency of the underlying classifier~\eqref{eq:inference} associated with Dirichlet Learning in the binary case ($K=2$)
\[
    \hat{y}(x) = \argmax_{k = 1, 2} \alpha_k(x) = \argmax_{k = 1, 2} \sum_{x^\ell \in \mcl L_k} \Ker(x^\ell, x).  
\]
We consider an open, connected, smooth, and bounded domain $\mcl X \subset \mbb R^d$ with $n$ observed labeled data pairs $(X_i, Y_i) \in \mbb R^d \times \{0, 1\}, i = 1, 2, \ldots, n$ drawn \textit{independently and identically} according to a joint distribution with density $\nu$ with marginal over the inputs given by the density $\rho(x)$. Notice that we have shifted the labels to be $Y_i \in \{0,1\}$ as is common practice in analyzing methods in the binary classification case. Consider the \textit{density-dependent} heat kernel propagation $\Ker(z, x) \equiv\Ker_t(z, x)$ with source $z \in \mcl X$ that solves 
\begin{equation} \label{eq:heat-eqn}
\begin{dcases}
\begin{aligned}
    \partial_t\Ker_t(z, x) - \Delta_\rho\Ker_t(z, x)&= 0  &x \in \mcl X, t > 0 \\
   \Ker_0(z, x) &= \delta_z(x) & \\
    % \km{\text{b.c.'s ?}} & % \nabla\Ker_t(z, x) \cdot \mathbf{n} &= 0  & x\in \partial \mcl X 
\end{aligned}
\end{dcases}
\end{equation}
where $\Delta_\rho = \frac{1}{\rho} \operatorname{div} \lp \rho^2 \nabla \cdot \rp$ is a self-adjoint diffusion operator with respect to the $\rho$-weighted inner product. As mentioned previously, please note that the differential operators in \eqref{eq:heat-eqn} are with respect to the variable $x$, while the variable $z$ is the source. Furthermore, recall that depending on the domain $\mcl X$, one must assume boundary conditions to make \eqref{eq:heat-eqn} well-defined, but we will instead state assumptions on the heat kernel, $\Ker_t(z,x)$, and provide reasonable example situations in which these assumptions hold. 

% \km{
% \begin{remark}[Use of heat-kernel as opposed to Poisson]
%     While we have previously discussed and utilized the Poisson data-dependent propagation operator in this work, we note that the behavior of these propagations is not clear for the continuum setting in the limit of $\tau \rightarrow 0$ \citep{calder_poisson_2020}. As such, we focus on a related \textit{heat-kernel} propagation operator for use as our kernel in our consistency analysis of Dirichlet Learning. It is noteworthy that one can view the Poisson propagation \ref{eq:poisson-prop} as using a first-order Taylor approximation of the heat equation differential operator, and so our analysis is relevant to understanding the behavior of Dirichlet Learning with the Poisson propagation. 
% \end{remark}
% }

\begin{remark}\label{remark:why-iid}
    Asymptotic consistency of classifiers is nearly always studied under the setting of labeled data drawn i.i.d. from the underlying data-generating distribution. However, in the case of active learning it is important to note that the labeled data is \textit{not} drawn i.i.d. from this distribution; rather, the sequence of observed data points is highly dependent on the previously labeled points. Furthermore, when the active learning policy selects query points randomly according to the distribution of acquisition function values (e.g., Prop.\,Sampling from Table~\ref{table:dirvar-defs}), this evolving distribution does not align with the underlying marginal density $\rho$. As such, we emphasize that the result of this section is to establish a classical type of consistency property of the Dirichlet Learning classification rule in this continuum limit setting (i.e., infinite unlabeled data) when the labeled data is drawn i.i.d. from the joint distribution.

    The establishment of consistency of the Dirichlet Learning classifier with the biased sampling of query points according to Dirichlet Variance is left as a future direction of work.  
    We posit, however, that a reasonable simplification can come from assuming that the distribution of proportional sampling has converged to the steady-state ``asymptotic exploitation'' distribution that we derived in the previous section. 
\end{remark}

For a data density $\rho(x)$ that corresponds to the marginal distribution of inputs for the joint distribution of $(X, Y) \in \mcl X \times \{0,1\}$, define the \textit{true} trend function $\eta(x) := \mbb E[Y | X=x]$. Showing consistency of the Dirichlet Learning classifier can be related to the convergence of an estimated trend function related to \eqref{eq:dir-learning-classifier}. We define the Dirichlet Learning estimated trend function to be
\begin{equation} \label{eq:eta-hat}
    \hat{\eta}(x) := \frac{1}{n\mbb E_\rho[\Ker_t(Z, x)]} \sum_{i=1}^n Y_i\Ker_t(X_i, x), % = \frac{1}{n\rho(x)} \sum_{i=1}^n Y_i\Ker_t(X_i, x) ,
\end{equation}
which gives that $\hat{y}(x) = \mathbbm{1}\{ \hat{\eta}(x) \geq \frac{1}{2} \}$.

We make the following assumptions that will simplify our consistency proof:
\begin{assumption} \label{assumption:unif-cts}
    The true trend function, $\eta(x)$, is uniformly continuous on $\mcl X$. 
\end{assumption}

\begin{assumption} \label{assumption:bounded-p-moment}
    Given open and connected domain $\mcl X \subset \mbb R^d$ and density $\rho : \mcl X \rightarrow \mbb R_+$, then there exists $p \geq 2$ such that 
    \[
        \|\Ker_t(\cdot, x)\|_{L^p(\rho)} \leq t^{-\frac{d}{2}}
    \]
    for each $x \in \mcl X$ and $t \leq 1$. 
\end{assumption}

\begin{assumption} \label{assumption:bound-unweighted-int}
    Given open and connected domain $\mcl X \subset \mbb R^d$ and density $\rho : \mcl X \rightarrow \mbb R_+$, then there exists a universal constant $\theta > 0$ such that 
    \[
        \sup_{z \in \mcl X} \lp \int_\mcl X\Ker_t(z,x) dx \rp  \leq \theta < \infty
    \]
    for all $t \leq 1$.
\end{assumption}

% \begin{assumption} \label{assumption:relation-rho-exp-kt}
%     For each $n \geq 2$, there exists $\hat{t}(n) > 0$ sufficiently small such that 
%     \[
%         \frac{\mbb E_\rho[\Ker_t(Z,x)]}{\rho(x)} \geq \frac{n-1}{n}.
%     \]
% \end{assumption}

% \begin{assumption} \label{assumption:bdd-rho}
%     The density $\rho(x)$ is uniformly bounded from above $\rho(x) \leq \rho_{max}$.
% \end{assumption}

\begin{remark}
The question of uniform heat kernel estimates has received significant attention in the mathematical community, and classical references on the topic include \citep{daviesbook,grigoryan2009heat,ouhabaz2009analysis}. These types of estimates are intimately linked with log-Sobolev and isoperimetric inequalities, which are also often linked with quantifiable sampling estimates. One notable class where these types of bounds hold are for log-concave distributions. Another is for densities supported on a bounded, smooth domain, which are bounded uniformly away from zero on the entire domain. We have assumed heat kernel estimates of this form for convenience, and certainly this assumption could be relaxed somewhat at the cost of added complexity in the proofs.
\end{remark}

% \todo{Write example/contexts in which the assumptions hold.}

We now state the following theorem that implies consistency of the Dirichlet Learning classifier decision rule \eqref{eq:dir-learning-classifier}.
\begin{thm} \label{thm:consistency-dir-learning}
    Consider the binary case of the Dirichlet Learning classifier
    \begin{equation} \label{eq:dir-learning-classifier}
        \hat{y}(x) := \begin{cases} 
                            1 & \text{if } \sum_{i=1}^n Y_i\Ker_t(X_i, x) \geq \sum_{i=1}^n (1-Y_i)\Ker_t(X_i,x) \\
                            0 & \text{otherwise}\\
                    \end{cases},
    \end{equation}
    with $\Ker_t(z,x)$ the density-dependent heat kernel defined in \eqref{eq:heat-eqn} with open, bounded, and connected domain $\mcl X \subset \mbb R^d$ and for times $t \in [0,1]$. Let $\{(X_i, Y_i)\}_{i=1}^n$ be samples from the underlying joint distribution $P(X, Y)$. Assume that the true trend function $\eta(x) = \mbb E [Y | X = x]$ is uniformly continuous on $\mcl X$ (Assumption~\ref{assumption:unif-cts}) and that Assumptions~\ref{assumption:bounded-p-moment} and \ref{assumption:bound-unweighted-int} are satisfied. 
    If $t \rightarrow 0$ and $n t^{\frac{d}{2}} \rightarrow \infty$ as $n \rightarrow \infty$, then for $\delta \in (0,1)$, the estimated trend function~\eqref{eq:eta-hat} 
    for the Dirichlet Learning classifier \eqref{eq:dir-learning-classifier} satisfies
    \[
        \int_\mcl X \left| \hat{\eta}(x) - \eta(x) \right| \rho(x) dx \leq \Xi(\theta) \sqrt{\frac{\ln\lp \frac{1}{\delta}\rp}{n}} 
    \]
    with probability greater than $1- \delta$, where the constant $\Xi(\theta)$ only depends on the universal constant $\theta < \infty$ as given in Assumption~\ref{assumption:bound-unweighted-int}. 
\end{thm}
% \todo{remove $\gamma$ in proof and introduce constant $C$ in statement of the theorem.}
\begin{remark}
    By Devroye et al (\citep{devroye1996probtheory}, Thm.\, 2.3), the conclusion of Theorem~\ref{thm:consistency-dir-learning} implies that the Dirichlet Learning classifier \eqref{eq:dir-learning-classifier} satisfies universal consistency, in the sense that
    \[
        \mbb P \lp \mcl E(\hat{y}) - \mcl E^\ast \rp \leq 2 e^{-n\epsilon^2/\Xi(\theta)^2}
    \]
    for $\epsilon > 0$ where $\mcl E(f) = \mbb E_{\nu | \nu^n} [f(X) \neq Y | X_1,Y_1, \ldots, X_n, Y_n ]$ is the error probability of a classification rule, $f$ and $\mcl E^\ast$ is the Bayes optimal error achieved by $y^\ast(x) = \mathbbm{1}\{ \eta(x) \geq \frac{1}{2} \}$.
\end{remark}

\subsubsection{Proof of Theorem~\ref{thm:consistency-dir-learning}}
We now present the details of the proof of Theorem~\ref{thm:consistency-dir-learning}.
\begin{proof}
% The consistency of Dirichlet Learning can then be phrased in terms of considering the $\ell^1$ difference between $\hat{\eta}(x)$ and the true trend function $\eta(x) = \mbb E[Y | X=x]$:
% \[
%     \int_\mcl X |\eta(x) - \hat{\eta}(x)| \rho(x) dx = \mbb E_\rho [ |\eta(x) - \hat{\eta}(x)|].
% \]
First, note that the expected value of the estimated trend function \eqref{eq:eta-hat} satisfies:
\begin{align*}
    \mbb E_{\nu^n}[\hat{\eta}(x)] &= \frac{1}{n \mbb E_\rho[\Ker_t(Z, x)]}\sum_{i=1}^n \mbb E_{\nu}[Y_i\Ker_t(X_i, x)] \\
    &= \frac{1}{n \mbb E_\rho[\Ker_t(Z, x)]} \sum_{i=1}^n\int_\mcl X \eta(x_i)\Ker_t(x_i, x)\rho(x_i) dx_i \\
    &= \frac{\mbb E_\rho[\eta(Z)\Ker_t(Z, x)]}{\mbb E_\rho[\Ker_t(Z, x)]}.
\end{align*} 
This explicitly highlights the ``averaging'' nature of this heat kernel classifier. Further, when $t \rightarrow 0$, we recover $\mbb E_{\nu^n} [ \hat{\eta}(x) ] \rightarrow \eta(x)$ due to the initial condition that $\Ker_0(z,x) = \delta_z(x)$. This also suggests then that as we observe more data ($n \rightarrow \infty$), it is natural to consider $t \rightarrow 0$. 

Consider the decomposition 
\[
    |\eta(x) - \hat{\eta}(x)| = \mbb E_{\nu^n} [ |\eta(x) - \hat{\eta}(x)|] + \lp |\eta(x) - \hat{\eta}(x)| - \mbb E_{\nu^n} [ |\eta(x) - \hat{\eta}(x)|] \rp,
\]
where the expectation is taken over the randomness in the labeled data $(X_i, Y_i) \sim^{iid} \nu$,
and we aim to bound the first term by bounding the terms
\begin{equation}  \label{eq:eta-decomp}
    \mbb E_{\nu^n}[ |\eta(x) - \hat{\eta}(x)|] \leq \underbrace{|\eta(x) - \mbb E_{\nu^n}[\hat{\eta}(x)]|}_{A} + \underbrace{\mbb E_{\nu^n}|\mbb E_{\nu^n}[\hat{\eta}(x)] - \hat{\eta}(x)|}_{B}.
\end{equation}

\paragraph{Part 1:}
Given $\tilde{\epsilon} > 0$, then by the uniform continuity of $\eta$ there exists $\xi > 0$ such that $|\eta(x) - \eta(z) | \leq \tilde{\epsilon}$ for all $z \in B(x,\xi)$. Defining $I(t; x, \xi) = \int_{\mcl B(x,\xi)^c}\Ker_t(x,z) \rho(z)dz < \mbb E_\rho[\Ker_t(Z,x)]$, then note that $I(t; x, \xi)$ is continuous with respect to $t \geq 0$ and 
\[
    \frac{I(t; x, \xi)}{\mbb E_\rho[\Ker_t(Z,x)]} \rightarrow 0, \text{ as } t \rightarrow 0^+.
\]
Thus, with $\xi > 0$, we can choose $t^\ast(\xi) > 0$ such that 
\begin{equation} \label{eq:Itdelta-small}
    \frac{I(t; x, \xi)}{\mbb E_\rho[\Ker_t(Z,x)]} \leq \tilde{\epsilon}.
\end{equation}
Using the fact that $\eta(x) \in [0, 1]$ for all $x$, then for $t \leq t^\ast(\xi)$, we can rewrite the integral of $\mbb E_\rho[ A]$ as
\begin{align*}
    \int_\mcl X |\eta(x) - \mbb E_{\nu^n}[\hat{\eta}(x)]| \rho(x) dx &\leq \int_\mcl X \int_\mcl X |\eta(x) - \eta(z)| \frac{k_t(z,x)}{\mbb E_\rho[\Ker_t(Z,x)]} \rho(z) dz \rho(x) dx \\
    &\leq  \tilde{\epsilon} \int_\mcl X \int_{ B(x, \xi)} \Ker_t(z,x) \rho(z) dz dx \\
    & \qquad + \int_\mcl X \overbrace{\int_{ B(x, \xi)^c}\Ker_t(z,x)  \rho(z) dz}^{I(t; x, \xi)} \frac{\rho(x)}{\mbb E_\rho[\Ker_t(Z,x)]} dx \\
    &\leq \tilde{\epsilon}\int_\mcl X \rho(x) dx + \int_\mcl X \frac{ I(t; x, \xi)}{\mbb E_\rho[\Ker_t(Z,x)]}\rho(x) dx \\
    &\leq 2 \tilde{\epsilon},
\end{align*}
where we have used \eqref{eq:Itdelta-small}.

\begin{comment}
\paragraph{Part 2:} Consider a ball $\mcl X$ of radius $R > 0$ centered at the origin. Then, the integral of the term $B$ from \eqref{eq:eta-decomp} over the complement of this ball is 
\[
    \int_{\mcl X^c} \mbb E_{\nu^n}[ |\mbb E_{\nu^n}[\hat{\eta}(x)] - \hat{\eta}(x)|] \rho(x) dx \leq 2 \int_{\mcl X^c} \mbb E_{\nu^n}[\hat{\eta}(x)] \rho(x) dx\rightarrow 2 \int_{\mcl X^c} \eta(x) \rho(x) dx
\]
as $n \rightarrow \infty$ and $t \rightarrow 0$ by the same method we applied previously to bound the integral of $A$. Clearly, we can choose $R$ so that $2 \int_{\mcl X} \eta(x) \rho(x) dx < \tilde{\epsilon}$.

For $x \in \mcl X$, we then bound 
\end{comment}

\paragraph{Part 2:} For the integral of the term $B$ from \eqref{eq:eta-decomp}, we can bound 

\begin{align} \label{eq:last-line-B}
    \mbb E_{\nu^n}[ |\mbb E_{\nu^n}[\hat{\eta}(x)] - \hat{\eta}(x)|]&\leq \sqrt{\mbb E_{\nu^n}[|\mbb E_{\nu^n}[\hat{\eta}(x)] - \hat{\eta}(x)|^2]} \nonumber \\ 
    &= \sqrt{\frac{\mbb E_{\nu^n}\left\{ \lp \sum_{i=1}^n \mbb E_{\nu}[Y_i\Ker_t(X_i, x)] - \sum_{j=1}^n Y_j\Ker_t(X_j, x) \rp^2  \right\}}{n^2(\mbb E_\rho[\Ker_t(Z,x)])^2}} \nonumber \\ 
    &= \sqrt{\frac{Var_\nu\lp Y\Ker_t(X,x)\rp}{n(\mbb E_\rho[\Ker_t(Z,x)])^2}} \nonumber \\
    &= \frac{1}{\sqrt{n}} \sqrt{\frac{\mbb E_\nu \left[ Y^2 k^2_t(X,x) - 2Yk_t(X,x) \mbb E_\rho[\eta(Z)k_t(Z,x)] \right] + (\mbb E_\rho[\eta(Z)k_t(Z,x)])^2}{(\mbb E_\rho[\Ker_t(Z,x)])^2}} \nonumber   \\
    &= \frac{1}{\sqrt{n}} \sqrt{\frac{\mbb E_\rho[\eta(X) k^2_t(X,x)] - (\mbb E_\rho[\eta(Z)\Ker_t(Z,x)])^2}{(\mbb E_\rho[\Ker_t(Z,x)])^2}} \qquad \qquad (Y^2 = Y) \nonumber \\
    &\leq \frac{1}{\sqrt{n}} \sqrt{\frac{\mbb E_\rho[\eta(X) k^2_t(X,x)]}{(\mbb E_\rho[\Ker_t(Z,x)])^2}} \nonumber \\
    &\leq \frac{1}{\sqrt{n}} \sqrt{\frac{\mbb E_\rho[k^2_t(Z,x)]}{(\mbb E_\rho[\Ker_t(Z,x)])^2}}
    % &\leq \sqrt{\frac{\max_{z}\left\{\Ker_t(z,x) \right\} }{n \mbb E_\rho[\Ker_t(Z,x)]}}
\end{align}
where we have used the fact that the data was sampled i.i.d. according to $\nu$. Assumption \ref{assumption:bounded-p-moment} with $\|\Ker_t(z,\cdot)\|_{L^{p}(\rho)}$ for some $p \geq 2$, then by H\"{o}lder's inequality: 
\begin{align*}
    \int_{\mcl X}\Ker_t^2(z,x) \rho(z) dz &= \int_{\mcl X}\Ker_t^2(z,x) \rho^{\frac{2}{p}}(z) \rho^{\frac{p - 2}{p}}(z)dz \\
    &\leq \lp \int_{\mcl X}\Ker_t^p(z,x) \rho(z)dz \rp^{\frac{2}{p}} \lp \int_{\mcl X} \rho^{\frac{p - 2}{p} \frac{p}{p-2}}(z) \rp^{\frac{p-2}{p}} \\
    &= \|\Ker_t(\cdot, x)\|_{L^p(\rho)}^2 \lp \int_{\mcl X} \rho(z)dz \rp^{\frac{p-2}{p}} \\
    &= \|\Ker_t(\cdot, x)\|_{L^p(\rho)}^2.
\end{align*}
Thus, we can simplify \eqref{eq:last-line-B} to 
\[
    \mbb E_{\nu^n}[ |\mbb E_{\nu^n}[\hat{\eta}(x)] - \hat{\eta}(x)|] \leq \frac{1}{\sqrt{n}} \frac{\|\Ker_t(\cdot, x)\|_{L^p(\rho)}}{\mbb E_\rho[\Ker_t(Z,x)]}.
\]
Noting that as $t \rightarrow 0^+$ 
\[
    \mbb E_\rho[\Ker_t(Z,x)] \rightarrow \rho(x) \text{  as } t \rightarrow 0^+,
\]
then we have that there exists $\hat{t} > 0$ such that for all $t \leq \hat{t}$
\begin{equation} \label{eq:ratio-gamma}
    \frac{\mbb E_\rho[\Ker_t(Z,x)]}{\rho(x)} \geq \frac{1}{2}
\end{equation}
for a sufficiently small $\gamma > 0$.
% with Assumption \ref{assumption:relation-rho-exp-kt} that for $t \leq \hat{t}$ sufficiently small 
% \begin{equation} \label{eq:integ-n-1-n}
%     \frac{\mbb E_\rho[\Ker_t(Z,x)]}{\rho(x)} \geq \frac{n-1}{n}.
% \end{equation}
% By appealing to the uniform bound of $\sup_{z,x \in \mcl X}\Ker_t(z,x) \leq C t^{-\frac{d}{2}}$ (Assumption~\ref{assumption:unif-bound-heat-kernel}(a)) and \eqref{eq:integ-n-1-n}, the integral of $B$ then becomes
The integral of $B$ then becomes
\begin{align} \label{eq:n-and-t-scaling}
    \int_{\mcl X} \mbb E_{\nu^n}[ |\mbb E_{\nu^n}[\hat{\eta}(x)] - \hat{\eta}(x)|] \rho(x) dx  &\leq \frac{1}{\sqrt{n}} \int_{\mcl X} \|\Ker_t(z,\cdot)\|_{L^p(\rho)} \frac{\rho(x)}{\mbb E_\rho[\Ker_t(Z,x)]} dx \nonumber\\
    &\leq \frac{2}{\sqrt{nt^d}} Vol(\mcl X) \nonumber \\
    &\rightarrow 0 \quad \text{as } n \rightarrow \infty.
\end{align}
Thus, as long as $nt^d \rightarrow \infty$ as $n \rightarrow \infty$ and $t \rightarrow 0$, we can take $n$ sufficiently large to bound \eqref{eq:eta-decomp} by combining the bounds on $\mbb E_\rho[A]$ and $\mbb E_\rho[B]$ with \eqref{eq:n-and-t-scaling} to get
\[
    \mbb E_{\nu^n}\left[ \int_\mcl X |\eta(x) - \hat{\eta}(x)| \rho(x) dx \right] \leq 3 \tilde{\epsilon} = \frac{\epsilon}{2},
\]
where we take $\tilde{\epsilon} = \frac{\epsilon}{6}.$

\paragraph{Part 3:} We can then use McDiarmid's inequality to give a high probability bound on the difference  
\[
\int_\mcl X \lp |\eta(x) - \hat{\eta}(x)| - \mbb E_{\nu^n}[ |\eta(x) - \hat{\eta}(x)|] \rp  dx.
\]
Let $\mcl D = \{ (x_i, y_i)\}_{i=1}^n$ be the fixed training (labeled) data, and let $\mcl D^{(i)}$ be a copy of $\mcl D$ with the $i^{th}$ pair replaced with $(\tilde{x}_i, \tilde{y}_i)$. Denote by $\hat{\eta}_i(x)$ the corresponding estimated trend function for $\mcl D^{(i)}$. Then, we can straightforwardly bound the following difference by using Assumption~\ref{assumption:bound-unweighted-int}
\begin{align*}
    \int_\mcl X |\eta(x) - \hat{\eta}(x)| \rho(x) dx &- \int_\mcl X |\eta(x) - \hat{\eta}_i(x)| \rho(x) dx \leq \int_\mcl X |\hat{\eta}(x) - \hat{\eta}_i(x)| \rho(x) dx\\
    &= \frac{1}{n} \int_\mcl X |Y_i\Ker_t(X_i, x) - \tilde{Y_i}\Ker_t(\tilde{X}_i,x)|\frac{\rho(x)}{\mbb E_\rho[\Ker_t(Z,x)]}dx \\
    &\leq  \frac{4}{n}  \sup_z \left\{ \int_\mcl X\Ker_t(z, x)dx\right\} \\
    &\leq \frac{4\theta}{n}
\end{align*}
where the constant $\theta$ was defined in Assumption~\ref{assumption:bound-unweighted-int} and have used \eqref{eq:ratio-gamma} with $n \geq 2$ trivially.
% \todo{need to be able to quantify finite $\theta$. Devroye uses the integral condition which requires the scaling between different times to be something like $\Ker_t(z,x) = \frac{1}{t}k_1\lp \frac{z}{t}, \frac{x}{t}\rp$. It is possible that something like this is true for the heat kernel that we have been considering. }
McDiarmid's inequality yields that 
\begin{align*}
    \mbb P &\lp \int_\mcl X |\eta(x) - \hat{\eta}(x) | \rho(x) dx > \epsilon \rp \\
    &\leq \mbb P \lp \int_\mcl X |\eta(x) - \hat{\eta}(x)| \rho(x) dx  -  \mbb E_{\nu^n}\left[ \int_\mcl X |\eta(x) - \hat{\eta}(x)| \rho(x) dx \right] > \frac{\epsilon}{2} \rp\\
    &\leq e^{-\frac{n\epsilon^2}{32\theta^2}},
\end{align*}
our desired result that, with high probability, the $\ell^1$ difference between $\eta(x)$ and $\hat{\eta}(x)$ is small as $n \rightarrow \infty$ and $t \rightarrow 0^+$ with $n t^d \rightarrow \infty$ with $t$ sufficiently small. 
\end{proof}

\acks{KM acknowledges support from the Peter J.\,O'Donnell Jr.\,Postdoctoral Fellowship and NSF
IFML grant 2019844. RM acknowledges partial support from NSF-DMS 2307971 and the Simons Foundation MP-TSM.}

\vskip 0.2in
\bibliography{references}

\end{document}